\documentclass[11pt]{article}
\usepackage{enumerate}
\usepackage{fullpage}
\usepackage{algorithm}
\usepackage[noend]{algorithmic}
\usepackage{todonotes}
\usepackage{amsmath,amsthm,amsfonts,amssymb}
\usepackage{amsmath}
\usepackage{hyperref}
\usepackage{thm-restate}
\usepackage{color}
\usepackage{mathrsfs}
\usepackage{enumitem}
\usepackage{bm}
\usepackage{multirow}
\usepackage{booktabs}
\usepackage{makecell}
\usepackage{graphicx}
\usepackage{subfigure}
\usepackage{caption}
\usepackage{thmtools}
\usepackage{thm-restate}
\usepackage{setspace}
\usepackage{makecell}
\definecolor{light-gray}{gray}{0.85}
\usepackage{colortbl}
\usepackage{xcolor}
\usepackage{hhline}
\usepackage{xspace} 
\usepackage{mathtools} 
\usepackage{cancel}
\usepackage{dsfont}
\usepackage[numbers]{natbib}

\include{packages}
\usepackage{amsfonts}

\renewcommand{\th}{^\mathrm{th}}

\newcommand{\defeq}{\mathrel{\mathop:}=}
\newcommand{\vect}[1]{\ensuremath{\mathbf{#1}}}
\newcommand{\mat}[1]{\ensuremath{\mathbf{#1}}}

\newcommand{\argmin}{\mathop{\rm argmin}}
\newcommand{\argmax}{\mathop{\rm argmax}}

\renewcommand{\det}{\mathrm{det}}
\newcommand{\rank}{\mathrm{rank}}

\newcommand{\trans}{^{\top}}
\newcommand{\diag}{\mathrm{diag}}
\newcommand{\poly}{\mathrm{poly}}
\newcommand{\polylog}{\mathrm{polylog}}

\newcommand{\abs}[1]{|{#1}|}
\newcommand{\norm}[1]{\|{#1} \|}

\newcommand{\E}{\mathbb{E}}
\newcommand{\D}{\mathbb{D}}
\renewcommand{\P}{\mathbb{P}}

\newcommand{\la}{\langle}
\newcommand{\ra}{\rangle}
\newcommand{\bL}{\mathbb{L}}

\newcommand{\cO}{\mathcal{O}}

\newcommand{\Z}{\mathbb{Z}}
\newcommand{\N}{\mathbb{N}}
\newcommand{\R}{\mathbb{R}}

\newcommand{\A}{\mat{A}}
\newcommand{\B}{\mat{B}}

\newcommand{\I}{\mat{I}}
\newcommand{\U}{\mat{U}}
\newcommand{\V}{\mat{V}}
\newcommand{\W}{\mat{W}}
\newcommand{\M}{\mat{M}}
\newcommand{\X}{\mat{X}}
\newcommand{\Y}{\mat{Y}}

\newcommand{\e}{\vect{e}}
\renewcommand{\b}{\vect{b}}

\renewcommand{\a}{\vect{a}}
\renewcommand{\o}{\vect{o}}
\newcommand{\w}{\vect{w}}
\newcommand{\x}{\vect{x}}
\newcommand{\y}{\vect{y}}
\newcommand{\z}{\vect{z}}
\newcommand{\g}{\vect{g}}

\newcommand{\fS}{\mathfrak{S}}

\newcommand{\cF}{\mathcal{F}}
\newcommand{\cM}{\mathcal{M}}
\newcommand{\cN}{\mathcal{N}}

\newcommand{\cS}{\mathcal{S}}
\newcommand{\cA}{\mathcal{A}}
\newcommand{\cB}{\mathcal{B}}
\newcommand{\cX}{\mathcal{X}}
\newcommand{\cY}{\mathcal{Y}}
\newcommand{\cT}{\mathcal{T}}
\newcommand{\cQ}{\mathcal{Q}}

\newenvironment{proof-sketch}{\noindent{\bf Proof Sketch}
  \hspace*{1em}}{\qed\bigskip\\}
\newenvironment{proof-idea}{\noindent{\bf Proof Idea}
  \hspace*{1em}}{\qed\bigskip\\}
\newenvironment{proof-of-lemma}[1][{}]{\noindent{\bf Proof of Lemma {#1}}
  \hspace*{1em}}{\qed\bigskip\\}
\newenvironment{proof-of-proposition}[1][{}]{\noindent{\bf
    Proof of Proposition {#1}}
  \hspace*{1em}}{\qed\bigskip\\}
\newenvironment{proof-of-theorem}[1][{}]{\noindent{\bf Proof of Theorem {#1}}
  \hspace*{1em}}{\qed\bigskip\\}
\newenvironment{inner-proof}{\noindent{\bf Proof}\hspace{1em}}{
  $\bigtriangledown$\medskip\\}
\newenvironment{proof-attempt}{\noindent{\bf Proof Attempt}
  \hspace*{1em}}{\qed\bigskip\\}

\newcommand{\Fcal}{\mathcal{F}}

\newcommand{\omle}{OMLE\xspace}

\renewcommand{\cT}{\mathcal{T}}

\renewcommand{\fS}{\mathscr{S}}
\newcommand{\fA}{\mathscr{A}}
\newcommand{\fO}{\mathscr{O}}
\newcommand{\fM}{\mathfrak{M}}

\newcommand{\TV}{\mathrm{TV}}

\newcommand{\Pb}{\overline{\mathbb{P}}}
\newcommand{\fD}{{\mathfrak{D}}}

\newcommand{\cC}{\mathcal{C}}

\newcommand{\cYh}{\hat{\mathcal{Y}}}

\renewcommand{\M}{\mathbb{M}}
\newcommand{\bM}{\mathbf{M}}
\newcommand{\m}{\mathbf{m}}
\newcommand{\Wbb}{\mathbb{W}}

\newcommand{\bpsi}{\boldsymbol{\psi}}
\newcommand{\bphi}{\boldsymbol{\phi}}
\newcommand{\Pie}{\Pi_{\rm exp}}

\newcommand{\iprod}[2]{\langle #1, #2 \rangle}

\newcommand{\pa}{{\rm{pa}}}

\newcommand{\SAIL}{\textsc{SAIL}}

\newcommand{\dlin}{{d_{\rm lin}}}

\newcommand{\bmu}{\boldsymbol{\mu}}
\newcommand{\bnu}{\boldsymbol{\nu}}
\newcommand{\bup}{\boldsymbol{\upsilon}}

\newcommand{\Csl}{\underline{C_{\sigma}}}
\newcommand{\Csu}{\overline{C_{\sigma}}}

\renewcommand{\M}{\mathbb{G}}
\newcommand{\K}{\mathbb{K}}

\newcommand{\name}{{well-conditioned }}

\newcommand{\Qa}{\cQ^{\textsc{a}}}

\renewcommand{\D}{\mathbb{D}}

\renewcommand{\cX}{\mathcal{X}}

\renewcommand{\Z}{\mat{Z}}

\newcommand{\G}{\mat{G}}
\usepackage{times}

 \newtheorem{theorem}{Theorem}[section]
 \newtheorem{lemma}[theorem]{Lemma}
 \newtheorem{corollary}[theorem]{Corollary}
 \newtheorem{remark}[theorem]{Remark}
\newtheorem{claim}{Claim}
 \newtheorem{proposition}[theorem]{Proposition}
 
 \theoremstyle{definition}
 \newtheorem{definition}[theorem]{Definition}
 \newtheorem{condition}[theorem]{Condition}
\newtheorem{assumption}[theorem]{Assumption}
\newtheorem{example}{Example}

\renewcommand{\epsilon}{\varepsilon} 

\renewcommand{\O}{\mathbb{O}}
\newcommand{\T}{\mathbb{T}}

\newcommand{\set}[1]{\left\{#1\right\}}

\newcommand{\praneeth}[1]{\noindent{\textcolor{green}{\{\textbf{PN:} \em #1\}}}}

 \hypersetup{
     colorlinks,
     linkcolor={blue!50!black},
     citecolor={blue!50!black},
 }
 \colorlet{linkequation}{blue}

\begin{document}

 \title{\fontsize{16pt}{16pt}
 \textbf{Optimistic MLE---A Generic Model-based Algorithm for Partially Observable Sequential Decision Making}
 }
 
 \date{}
 
 \author{ 
  Qinghua Liu\footnote{Part of this work was done during QL's internship at DeepMind.} \\
  Princeton University \\
  \texttt{qinghual@princeton.edu} \\
  \and
  Praneeth Netrapalli \\
  Google Research India \\
  \texttt{pnetrapalli@google.com} \\
  \and
  Csaba Szepesv{\'a}ri \\
  DeepMind and University of Alberta \\
  \texttt{szepesva@ualberta.ca} \\
  \and
  Chi Jin \\
  Princeton University \\
  \texttt{chij@princeton.edu} \\
}

 \maketitle
\begin{abstract}%

This paper introduces a simple efficient learning algorithms for general sequential decision making. The algorithm combines Optimism for exploration with Maximum Likelihood Estimation for model estimation, which is thus named OMLE. We prove that OMLE learns the near-optimal policies of an enormously rich class of sequential decision making problems in a \emph{polynomial} number of samples. This rich class includes not only a majority of known tractable model-based Reinforcement Learning (RL) problems (such as tabular MDPs, factored MDPs, low witness rank problems, tabular weakly-revealing/observable POMDPs and multi-step decodable POMDPs ), but also many new challenging RL problems especially in the partially observable setting that were not previously known to be tractable.

Notably, the new problems addressed by this paper include (1) \emph{observable} POMDPs with continuous observation and function approximation, where we achieve the first sample complexity that is completely independent of the size of observation space; (2) \emph{well-conditioned} low-rank sequential decision making problems (also known as Predictive State Representations (PSRs)), which include and generalize all known tractable POMDP examples under a more intrinsic representation; (3) general sequential decision making problems under \emph{SAIL} condition, which unifies our existing understandings  of model-based RL in both fully observable and partially observable settings. SAIL condition is identified by this paper, which can be viewed as a natural generalization of Bellman/witness rank to address partial observability. This paper also presents a reward-free variant of OMLE algorithm, which learns approximate dynamic models that enable the computation of near-optimal policies for all reward functions simultaneously.

\end{abstract}

\newpage
\tableofcontents
\newpage


\section{Introduction}
A wide range of modern artificial intelligence applications can be cast as sequential decision making problems, in which an agent interacts with an unknown environment through time, and learns to make a sequence of decisions using intermediate feedback. Sequential decision making covers not only problems like Atari games \citep{mnih2013playing}, Go \citep{silver2017mastering}, Chess \citep{campbell2002deep} and basic control systems \citep{todorov2005generalized}, where states are fully accessible to the learner (the \emph{fully observable} setting), but also applications including StarCraft \citep{vinyals2019grandmaster}, Poker \citep{brown2019superhuman}, robotics with local sensors \citep{akkaya2019solving}, autonomous driving \citep{levinson2011towards} and medical diagnostic systems \citep{hauskrecht2000planning}, where observations only reveal partial information about the underlying states (the \emph{partially observable} setting). While  the fully observable sequential decision making problems have been under intense theoretical investigation over recent years, 
the partially observable problems remain comparatively less understood.

Distinguished from fully observable systems, a  learner in partially observable systems is only able to see the observations that contain partial information about the underlying states. Observations in general are no longer Markovian.  As a result, it is no longer sufficient for the learner to make decision based on the observation or information available at the current step. Instead, the learner is required to additionally infer the latent states using past histories (memories). Such histories of observations have exponentially many possibilities, leading to many well-known hardness results in the worst case in both computation \cite{papadimitriou1987complexity,mundhenk2000complexity,vlassis2012computational,mossel2005learning} and statistics \cite{krishnamurthy2016pac}. To avoid these worst-case barriers, a recent line of results started to investigate rich subclasses of Partially Observable Markov Decision Process (POMDPs) under the basic settings of finite states and observations \cite[see, e.g.,][]{jin2020sample,liu2022partially}, which still only constitute a relatively small subset of all partially observable problems of practical interests.

In this paper, we introduce a simple, generic, model-based algorithm---OMLE, which combines Optimism (O) for exploration with Maximum Likelihood Estimation (MLE) for model estimation. We prove that OMLE learns the near-optimal policies of an enormously rich class of sequential decision making problems in a \emph{polynomial} number of samples. This rich class includes not only a majority of known tractable model-based Reinforcement Learning (RL) problems such as tabular MDPs, factored MDPs, low witness rank problems \cite{sun2019model}, tabular weakly-revealing/observable POMDPs \cite{jin2020sample,liu2022partially} and multi-step decodable POMDPs \cite{efroni2022provable}, but also, more importantly, many new challenging RL problems especially in the partially observable setting \emph{that were not previously known to be tractable} (see Section \ref{sec:contribution}). To achieve these new results, this paper develops new frameworks and techniques which address a set of fundamental challenges that are uniquely presented in the partially observable systems:

\paragraph{Challenge 1: Continuous observation space and function approximation with partial observability.} Modern applications of sequential decision making often involve an enormous (or even infinite) number of  observations, where \emph{function approximation} must be deployed to approximate dynamic models, value functions, or policies. While function approximation greatly expands the potential reach of existing frameworks, particularly via deep architectures, it raises a number of fundamental questions including generalization, model misspecification, and how to address those issues in presence of exploration. Function approximation becomes even more complicated in the partially observable setting when further coupled with the inference of latent states and the use of history dependent policies. As a result, existing results on function approximation in the partially observable setting remain very limited \cite{cai2022reinforcement, uehara2022provably}. They make rather restrictive assumptions, and do not provide efficient guarantees even to a relatively simple continuous-observation extension of the basic tabular weakly-revealing or observable POMDPs \cite{golowich2022planning, liu2022partially}---GM-POMDPs (Section \ref{sec:GM-POMDP}), which only add Gaussian noise to the observations in the original models.


\paragraph{Challenge 2: Learning under intrinsic representation of partially observable systems.} Most existing works on efficient learning of partially observable problems focus on the model of POMDPs. POMDPs are based on latent states that are \emph{unobservable} and subject to non-trivial ambiguity---there can exist multiple different POMDPs that represent the same sequential decision making problem. This ambiguity directly leads to the unidentifiability of latent states even in the benign settings where learning near-optimal policy is possible. This paper considers a more intrinsic modeling of partially observable dynamic system---Predictive State Representations (PSRs) \cite{littman2001predictive,singh2012predictive}, which model a  dynamic system using only \emph{observable} experiments of futures. It is known that PSRs can represent any low-rank sequential decision making problems, which are more expressive than finite-state POMDPs \citep{jaeger1998discrete}. However, it remains unclear how to learn large class of PSRs sample-efficiently.

\paragraph{Challenge 3: A unified understanding of fully observable and partially observable RL.} There has been a long line of important works on generic framework of reinforcement learning \cite{jiang2017contextual,sun2019model,jin2021bellman,du2021bilinear,foster2021statistical}. However, most of them focus on the fully observable problems and are only capable of dealing with very special partially observable problems such as reactive POMDPs. A majority of them critically rely on the complexity measures that are based on Bellman rank \cite{jiang2017contextual} or witness rank \cite{sun2019model} (the model-based version), which assumes the Bellman error or the model estimation error (in the model-based setting) to have a bilinear structure.
These bilinear-based complexity measures completely fail to explain the tractability of many basic partially observable problems \cite{golowich2022planning, liu2022partially, efroni2022provable}. It remains open to develop a unified theoretical framework which explain large classes of both fully observable and partially observable problems.

This work addresses all three challenges above.
For Challenge 1, we prove that OMLE learns \emph{observable} POMDPs with continuous observation and function approximation, where we achieve the first sample complexity that is completely independent of the size of observation space. For Challenge 2, we show that OMLE learns \emph{well-conditioned} PSRs, which include and generalize all known tractable POMDP examples under a more intrinsic representation; For Challenge 3, we identify a new condition---Summation of Absolute values of Independent biLinear
functions (SAIL)---which can be viewed as a natural generalization of Bellman/witness rank to address partial observability. We prove that OMLE learns general sequential decision making problems under \emph{SAIL} condition, which include all problems considered in this paper, and unify our existing understanding for model-based RL in both fully observable and partially observable settings.

\subsection{Overview of our results}
\label{sec:contribution}

This paper introduces a generic algorithm framework of OMLE, and prove it learns a very rich class of sequential decision making problems sample-efficiently. The OMLE algorithm (in its basic form) was first proposed in \cite{liu2022partially} for sample-efficient learning of tabular weakly-revealing POMDPs. Here we introduce some extra flexibility to the algorithm, address new challenges, and provide learning guarantees in a significantly more general setup. Specifically,

\begin{itemize}
    \item We identify a sufficient condition for OMLE---generalized eluder-type condition (Condition \ref{cond:eluder}), under which OMLE is guaranteed to find near-optimal policy in a polynomial number of samples. We will use this generalized eluder-type condition to analyze all problems considered in this paper.
   
    \item We consider sequential decision making with low-rank structure (also known as Predictive State Representations (PSRs)). We first show that learning generic PSRs is intractable. We then identify a rich subclass called \emph{well-conditioned} PSRs, and prove that OMLE learn them sample-efficiently. Our sample complexity depends polynomially on the rank of PSRs and the size of core action sequences, and is independent of the size of core tests and the size of observation space.

    \item We show that a wide range of POMDP models fall in to the class of \emph{well-conditioned PSRs}. They include not only previously known tractable problems such as tabular weakly-revealing/observable POMDPs \cite{jin2020sample,liu2022partially}, multistep decodable POMDPs \cite{efroni2022provable}; but also new problems including observable POMDPs with continuous observation (in particular, GM-POMDPs, see Section \ref{sec:GM-POMDP}), and POMDPs with a few known core action sequence. Our PSR results immediately imply sample efficient guarantees of OMLE to learn these POMDP models.

    \item We identify a new \SAIL\ condition which can be viewed as a natural generalization of Bellman/witness rank, and prove that OMLE sample-efficiently learns any sequential decision making problem with \SAIL\ condition. We show that \SAIL\ condition holds for well-conditioned PSRs and for all problems with low witness rank \cite{sun2019model}. The latter covers a majority of known tractable model-based RL problems in the fully observable setting including factored MDPs, kernel linear MDPs, sparse linear bandits.
    Moreover, our sample complexity guarantees for learning low witness rank problems improve  over the existing results \cite{sun2019model} by a multiplicative factor of witness rank. 

    \item We propose a variant of OMLE for reward-free learning. We show that Reward-free OMLE learns an approximate dynamic model sample-efficiently under a slightly stronger version of the SAIL condition.
    This approximate dynamic model allows us to compute the near-optimal policies for all reward functions simultaneously.
    \end{itemize}
\begin{figure}
\begin{center}
  \includegraphics[width=0.7\textwidth]{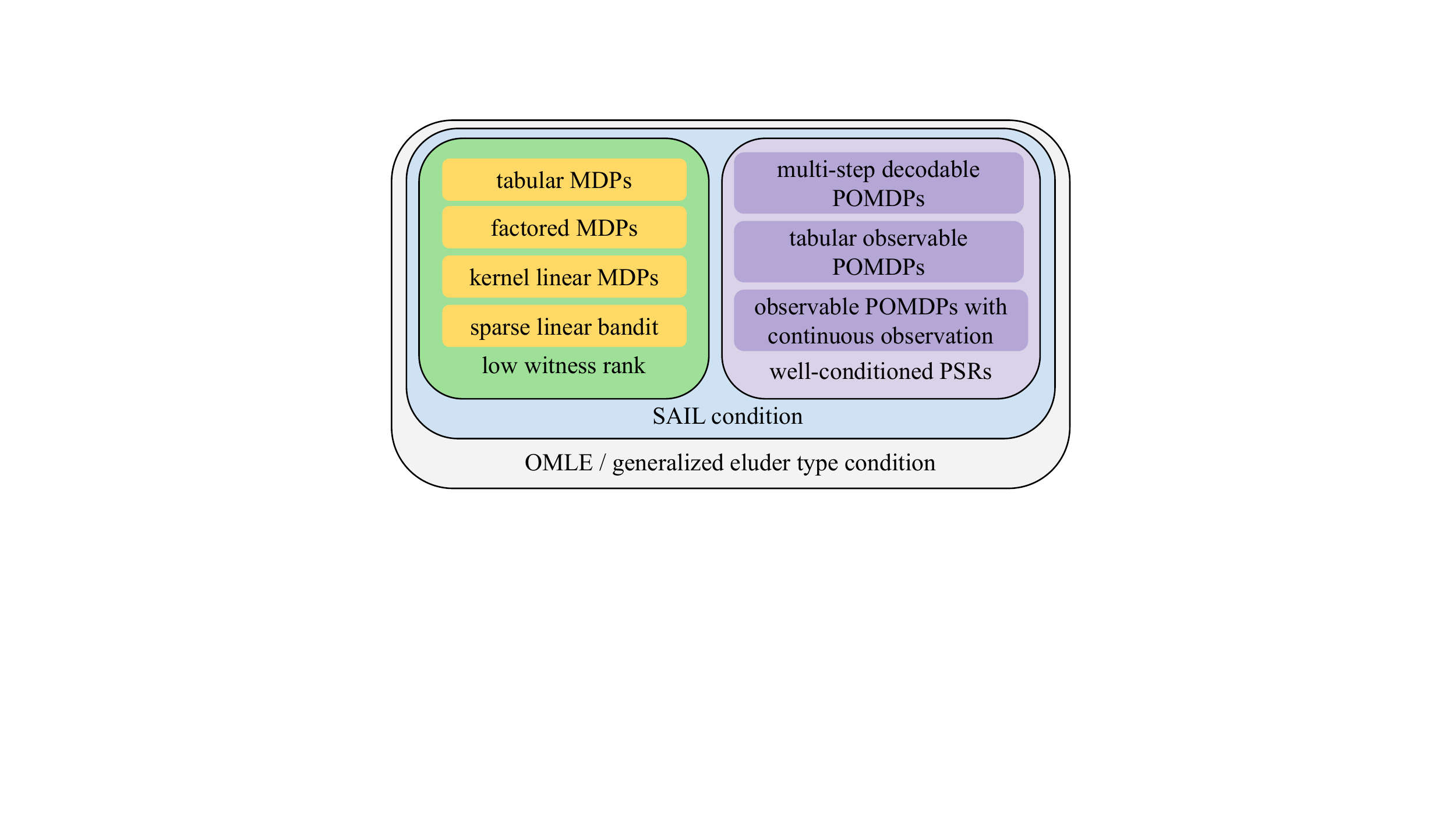}
  \caption{A summary of sequential decision making problems that can be efficiently learned by OMLE.
  }
  \label{fig:overview}
  \vspace{-.8cm}
\end{center}
\end{figure}

Besides above results, this paper also establishes the rigorous formulations for \emph{overparameterized} PSRs, studies their properties, gives rigorous treatment for PSRs with continuous observation, and bounds the bracketing number of tabular PSRs, which might be of independent interests to the community.

\subsection{Technical contribution}
\label{sec:technical_contribution}

Underlying our new results is a set of new techniques for  handling PSRs with infinite  observations.

\begin{itemize}
\item \textbf{New sharp elliptical potential style lemma for \SAIL.} A crucial component for analyzing optimistic algorithms is pigeon-hole's principle \cite{azar2017minimax, jin2018q} or so-called elliptical potential lemma \cite{lattimore2020bandit} which ensures that the size of confidence set is shrinking fast enough to guarantee near-optimality of the learned policy after a small number of rounds. Standard elliptical potential lemma applies to linear bandits whose reward is a linear function of form $\la \theta, x \ra$. To analyze POMDPs or PSRs, we establish a new generalized version of elliptical potential lemma which applies to Summation of Absolute values of Independent biLinear functions (SAIL) of form $\sum_{i=1}^m \sum_{j=1}^n |\la \theta_i, x_j \ra|$. A similar problem has been studied in \cite{liu2022partially} but the bounds derived therein \emph{depend} on $m, n$, which scales with the size of observation space in PSRs/POMDPs. Such result becomes vacuous in the infinite-observation setting. We address this issue by developing a significantly sharper argument, which gives bounds completely \emph{independent} of $m, n$ (thus the size of observation space). Please see Appendix \ref{app:l1-pigeon-hole} for details. 

\item \textbf{Projection that approximately preserve the $\ell_1$-norm.} To apply the new sharp elliptical potential lemma discussed above, we need a projection operator which maps a function (or high-dimensional vector) defined on the observation space  into a low-dimensional Euclidean space whose dimension is equal to the intrinsic complexity of POMDPs or PSRs. Our analysis further requires the resulting vector after projection to have a small $\ell_1$-norm. In POMDPs, we can directly construct such a projection by taking the pseudo-inverse of emission matrices (as in \cite{liu2022partially}). However, such choice does not apply to PSRs as it has less structure than POMDPs. To address this issue, we consider the general problem of projecting high-dimensional vectors (that lie in a low-dimensional subspace) to a low-dimensional Euclidean space without significantly increasing their $\ell_1$-norm. We achieve so by constructing a projection using the Barycentric spanner technique. Please see Lemma \ref{lem:barycentric} and Step 3 in Appendix \ref{app:psr-main} for details.

\item \textbf{Matrix pseudo-inverse with small $\ell_1$-norm.} To establish efficient guarantees for learning observable POMDPs, we need to construct operator $\bM$ as in the framework of PSR, and bound the $\ell_1$-norm of the operator. All previous works \cite[e.g.,][etc]{azizzadenesheli2016reinforcement,jin2020sample,liu2022partially,zhan2022pac} construct such operators using the pseudo-inverse of emission matrices $\O^\dagger$, whose $\ell_1$-norm scales with the size of observation space even under the \emph{observable} condition (Condition \ref{cond:pomdp}). Such dependency prevents their analysis from generalizing to the infinite observation setting. We address this issue by adding a matrix $\Y$ that lies in the subspace complementary to $\O^\dagger$. We show that with an optimal choice of $\Y$,  $\O^\dagger + \Y$ has a small $\ell_1$-norm which is independent of the size of observation space. To our best knowledge, this operator design is completely new and has not been considered in  the previous POMDP literature. 
 \end{itemize}

\subsection{Related works}

Reinforcement learning theory has been intensively studied in recent years. For the purpose of this work, we will mainly focus on the works that provide finite-sample complexity guarantees for RL problems related to either partial observability or function approximation. 

\paragraph{Learning POMDPs.}
There is a line of well-known worst-case hardness results for planning and learning in POMDPs on both computation and sample complexity:
\cite{papadimitriou1987complexity} shows that planning (finding the optimal policy given the POMDP model) is PSPACE-complete. \citep{vlassis2012computational} further shows that even finding the optimal memoryless policy remains NP-hard. On the other hand, even with infinite computational resource, \citep{krishnamurthy2016pac} proves that learning near-optimal policy of a POMDP in the worst case requires a number of samples that is exponential in the episode length.

Due to these worst-case hardness barrier, recent line of works have started to develop positive results for learning POMDPs under structural assumptions. One such structure is \emph{observability} \cite{golowich2022planning} (or \emph{weakly-revealing condition} \cite{jin2020sample,liu2022partially}; they two are equivalent in the tabular setting up to a polynomial factor). Observability requires the emission matrix to be rank-$S$ (the number of latent states) so that different belief states would induce distinct distributions over observations. A sequence of works  \citep[e.g.,][]{guo2016pac,azizzadenesheli2016reinforcement,jin2020sample,xiong2021sublinear,cai2022reinforcement,wang2022embed} have applied the method of moments \citep{hsu2012spectral,anandkumar2014tensor} to learn  POMDPs and provide polynomial sample guarantees. All of them except for  \citep{jin2020sample,cai2022reinforcement,wang2022embed} require further strong assumptions to learn tabular  observable POMDPs in the undercomplete setting (where the number of states is less than the number of observations), and none of them provably learn tabular observable POMDPs in the overcomplete setting.
Recently,  \cite{liu2022partially}  prove that observability alone is sufficient to guarantee sample-efficient learning of POMDPs in both undercomplete and overcomplete settings. In particular, \cite{liu2022partially} achieve so by designing OMLE algorithm, which is mostly related to our works.  The main algorithm \omle\ in this paper inherits the key algorithmic designs from \cite{liu2022partially} but the results we derived are much stronger and more general than \cite{liu2022partially}---while the results of \cite{liu2022partially} only applies to tabular observable POMDPs, our theory further applies to observable POMDPs with continuous observation,  multi-step docodable POMDPs, well-conditioned PSRs, and RL problems with low witness rank, all of which are beyond the capability of the techniques in \cite{liu2022partially}. Finally, two recent works \citep{golowich2022planning, golowich2022learning} provide new algorithms for planning and learning observable POMDPs with quasi-polynomial computation complexity and sample complexity respectively, which significantly improve over prior works in terms of computation. However, their sample complexity is not polynomial, and is thus worse than ours.

Another structure commonly studied by the community is \emph{decodability}. Block MDPs  \citep[see, e.g.,][]{krishnamurthy2016pac,jiang2017contextual,du2019provably, misra2020kinematic} assume that the current latent state can be deterministically decoded from the current observation. Technically speaking, block MDPs are not partially observable problems, as all information about the state is contained in the current observation.
Multi-step decodable POMDPs \cite{efroni2022provable} are partially observable problems, which generalize block MDPs by assuming the latent state can be only decoded from the most recent $m$-step action-observations. \cite{efroni2022provable} proves polynomial sample complexity guarantees for learning multi-step decodable POMDPs with infinite observations under value-based function approximation. While this paper also covers the setting of multi-step decodable POMDPs and derives similar guarantees, we consider model-based function approximation due to the algorithm style of OMLE, which is slightly different from \cite{efroni2022provable}.

\begin{table}[t]
  \renewcommand{\arraystretch}{1.2}
  \begin{center}
  \begin{tabular}{|c|c|c|c|c|}
       \hline
  & \textbf{RL problems} & \textbf{\makecell{Witness\\rank \cite{sun2019model}}} & \textbf{\makecell{\cite{zhan2022pac}\\(concurrent)}} & \textbf{\makecell{This\\work}} \\ \hline
  \multirow{5}{*}{\makecell{Partially\\observable}} 
  & tabular observable POMDPs \cite{jin2020sample,liu2022partially} &  & $*$  & $\checkmark$  \\\hhline{|~----|}
  & \cellcolor{light-gray} observable POMDPs with cts observation  &  &  & $\checkmark$ \\ \hhline{|~----|}
  & multistep decodable POMDPs \cite{efroni2022provable} &  & $*$ & $\checkmark$ \\\hhline{|~----|}
  & \cellcolor{light-gray} tabular well-conditioned PSRs &  & $*$ & $\checkmark$\\\hhline{|~----|}
  & \cellcolor{light-gray} well-conditioned PSRs with cts observation &  &  & $\checkmark$ \\\hline
  \multirow{3}{*}{\makecell{Fully\\observable}} 
  &  factored MDPs & $\checkmark$ &  & $\checkmark$\\\hhline{|~----|}
  &  kernel linear MDPs \citep{yang2020reinforcement} & $\checkmark$ &  & $\checkmark$\\\hhline{|~----|}
  &  sparse linear bandits & $\checkmark$ &  & $\checkmark$\\\hline
  \end{tabular}
  \end{center}
  \vspace{-2ex}
    \caption{A table of whether theoretical frameworks provide sample-efficient guarantees for learning various classes of RL problems. Grey color denotes the classes of RL problems that are first identified and proven tractable by this paper. $*$ means that the framework only applies to a limited subset of the problems.\protect\footnotemark}
  \vspace{-1ex}  
\end{table}

\footnotetext{For all three classes of problems, \cite{zhan2022pac} requires extra regularity assumptions (Assumption 1 in \cite{zhan2022pac}).}

\paragraph{Learning Predictive State Representations (PSRs).}
Predictive state representations (PSRs) \citep{littman2001predictive,singh2003learning,singh2012predictive,james2004planning,wolfe2005learning,mccracken2005online,bowling2006learning} is a generic approach for modeling low-rank sequential decision making problems, which assumes there exist a small number of future experiments (called core tests) so that conditioning on any history,  the future dynamics of the system is a deterministic function of  the conditional probability of the core tests. 
\cite{jaeger1998discrete} proved that any POMDPs can be modeled by PSRs while there exist PSRs that cannot be represented by any finite-state POMDPs in the infinite horizon setting.   
Similar to the case of POMDPs, spectral learning-style methods have been applied to learning PSRs 
and polynomial sample-efficiency guarantees are derived  under certain full-rank assumption, exploratory assumptions, and more \citep{boots2011closing,hefny2015supervised,zhang2021reinforcement}. Such exploratory assumptions significantly reduce the difficulty of exploration problem. Compared to the above works, this paper directly address the challenge of online exploration and provides the first polynomial sample complexity guarantees for learning any \emph{well-conditioned} PSR without any further assumptions.

During the preparation of this draft, we noticed a concurrent work \cite{zhan2022pac} that similarly extends the idea of OMLE \cite{liu2022partially} to provide sample-efficient guarantees for subclasses of PSRs. We remark the key difference comparing to our results on PSRs: (1) the tractable subclass of PSRs identified by \cite{zhan2022pac} is more restrictive than ours. In particular, their class do not include standard observable POMDPs or multistep decodable POMDPs unless extra assumptions are made; (2) their sample complexity depends on the size of observation space, which does not apply to the setting of continuous observation.

A work in parallel \citep{chen2022partially} (released on the same day as this work) also applies OMLE \cite{liu2022partially} to learn a large subclass of PSRs, which is similar to the well-conditioned PSRs identified in this paper. Their framework can address problems with \emph{countably} infinitely many observations, and they use Hellinger distance-based analysis (instead of TV distance used in this paper), which allows them to use $\ell_2$-based elliptical potential lemma, and obtain sharper sample complexity guarantees. However, comparing to our $\ell_1$-based elliptical potential lemma (Lemma \ref{prop:pigeon-hole}), their $\ell_2$-based version requires a stronger precondition, and is thus applicable to a smaller set of problems. For instance, their techniques do not directly extend to the problems under SAIL condition (see Section \ref{sec:beyond-low-rank}). 
Finally, \citep{chen2022partially} did not provide efficient guarantees for learning observable POMDPs with continuous observation (e.g., GM-POMDPs). Despite that this class of problems belongs to well-conditioned PSRs (and the counterparts in \cite{chen2022partially}),  proving this inclusion relation further requires our new techniques on matrix pseudo-inverse with small $\ell_1$-norm  (Section \ref{sec:technical_contribution} and Appendix \ref{app:l1-inverse}), a rigorous formulation for addressing PSRs with \emph{uncountable} many or continuous observations (Appendix \ref{sec:psr_cts}), and a new core tests design technique (Appendix \ref{app:observable-continuous}), which are not present in \cite{chen2022partially}. In sum, \citep{chen2022partially} does not overlap with our  contribution (1) and (3) highlighted in the abstract.

\paragraph{RL under general function approximation.} Many works \cite[see, e.g.,][]{jiang2017contextual,sun2019model,jin2021bellman,du2021bilinear,foster2021statistical} have tried to develop unified theoretical frameworks for RL under general function approximation. This line of works typically introduce new complexity measures, and then design sample-efficient algorithms for learning RL problems of low complexity under those complexity measures. Most applications of their frameworks are fully observable problems and they have only been proved to address very special partially observable problems such as reactive POMDPs or reactive PSRs \cite{jiang2017contextual}, where the optimal value only depends on the current-step observation-action pair. The complexity measures in \cite{jiang2017contextual,sun2019model,jin2021bellman,du2021bilinear} are Bellman rank \cite{jiang2017contextual}, witness rank \cite{sun2019model} (the model-based version) and their extensions and variants. These complexities are not small even in the basic partially observable problems such as tabular observable POMDPs. 
\cite{foster2021statistical} proposes a complexity measure named DEC, which can be potentially small for certain partially observable problems considered in this paper. However, showing these classes of problems to have a low DEC is a highly non-trivial task by itself, and is very likely to require the results and techniques developed in this paper. Distinguished from these works, this paper develops a general framework that provably address a wide range of partially observable problems. Importantly, many tractable partially observable problems are newly identified by this paper, which are not known to prior works on general frameworks. This greatly expand our understanding for RL under partial observability. Finally, we remark that since the focus of this paper is OMLE algorithm which is a model-based algorithm. In rigorous comparison, we restrict ourselves to comparing against prior frameworks on model-based RL \cite{sun2019model} instead of model-free RL \cite{jiang2017contextual,jin2021bellman,du2021bilinear}.

A recent work \citep{uehara2022provably} developed a general framework for partially observable RL under bilinear actor-critic framework. Different from our work, their results do not provide polynomial sample guarantee for learning the near-optimal policy for the basic tabular observable POMDPs. We remark that a concurrent work \cite{chen2022partially} proves that well-conditioned PSRs also have low DEC.
Plugging this result into their recent strengthened framework for learning general problems of low DEC \cite{chen2022unified}, they can also provide a polynomial sample complexity guarantee  for learning well-conditioned PSRs.


\section{Preliminaries}

\paragraph{Notation.} For a positive integer $n$, we let $[n] = \{1,\dots,n\}$. We use the notation $x_{1:n}$ to denote the sequence $(x_1, \ldots, x_n)$. We use bold upper-case letters  $\B$ to denote matrices and bold lower-case letters $\b$ to denote vectors.  Given a matrix $\B\in\R^{m\times n}$, we use $\B_{ij}$ to denote its $(i,j)^{\rm th}$ entry, $\|\B\|_p = \max_{\|\z\|\neq 0} \|\B \z\|_p/\|\z\|_p$ to denote its matrix $p$-norm, and $\B^\dagger$ to denote its Moore-Penrose inverse. For a vector $\b\in\R^m$, we use $\b_i$ to denote its $i\th$ entry, $\norm{\b}_p$ to denote its vector $p$-norm, and $\diag(\b)$ to denote a diagonal  matrix with $[\diag(\b)]_{ii}=\b_i$. Given a set $\cX$, we use $2^\cX$ to denote the collections of all subsets of $\cX$.

\subsection{Sequential decision making}
 We consider the general episodic sequential decision making problems, which can be specified by a tuple $(\fO,\fA, H, \P, R)$. Here $\fO$ and $\fA$ denote the space of observation and  action respectively.  $H$ denotes the length of each episode. $\P = \{\P_{h}\}_{h=1}^H$ specifies the joint distribution over observations $o_{1:H}$ conditioned on action sequence $a_{1:H}$, which can be factorized as:
\begin{equation*}
    \P(o_{1:H}|a_{1:H}) = \prod_{h=1}^H \P_h(o_h|o_{1:h-1}, a_{1:h-1})
\end{equation*}
$\P$ is also known as the \emph{system dynamics}. $R=\{R_h\}_{h\in[H]}$ are the known reward functions from $\fO$ to $[0,1]$  such that the agent will receive reward $R_h(o)$ when she observes $o\in\fO$ at step $h$.\footnote{This is equivalent to assuming that reward information is contained in the observation. We consider this setup to avoid the leakage of information about the dynamic system through rewards beyond observations. We remark that all  results in this paper immediately extend to the more general setting where reward $R(\tau_H)$ can be a function of the entire observation-action trajectory $\tau_H = (o_{1:H}, a_{1:H})$, and is only received at the end of each episode.} To simplify the presentation, we also use the notation $\Pb(o_{1:h}, a_{1:h}):=\P(o_{1:h}|a_{1:h})$  for any trajectory $(o_{1:h}, a_{1:h})$ to represent the conditional probability over observations conditioned on actions. Throughout this paper we assume the finite action space $\fA$ with $|\fA| = A$, but allow infinitely large observation space $\fO$.

At each step $h \in [H]$ of each episode, the environment first samples an observation $o_{h}$ according to $\P_h(\cdot|o_{1:h-1}, a_{1:h-1})$ based on the observation-action sequence in the past, and then the agent takes an action $a_h$. The current episode terminates immediately after $a_{H}$ is taken.

\paragraph{Policy and value.} 
A policy $\pi=\{\pi_h\}_{h=1}^H$ is a collection of $H$ functions where $\pi_h: (\fO\times\fA)^{h-1}\times\fO \rightarrow \Delta_A$ maps a length-$h$ observation-action sequence to a distribution over actions.  
Given a policy $\pi$, we use $V^\pi$ to denote its value, which is defined as the expected total reward received under policy $\pi$:  
\begin{equation*}
    V^\pi := \E_\pi \left[ \sum_{h=1}^H R_h(o_h)\right],
\end{equation*}
where the expectation is with respect to the randomness within the system dynamics $\P$ and the policy $\pi$. 

Since the action space and the episode length are both finite, the maximal value over all policies $\max_\pi V^\pi$ always exists. We call $\max_\pi V^\pi$  the optimal value denoted by  $V^\star$, and call the policy that achieves this optimal value  the optimal policy denoted by  $\pi^\star$.

\paragraph{Learning objective.} Our goal is to learn an $\epsilon$-optimal policy $\pi$ in the sense that $V^\pi \ge V^\star-\epsilon$, using a number of samples polynomial in all relevant parameters. We also consider the problem of learning with low regret. Suppose the agent interacts with the sequential decision making problem for $K$ episodes, and plays  policy $\pi_k$ in the $k^\text{th}$ episode for all $k \in [K]$. The total (expected) regret is then defined as:
\begin{equation*}
 \text{Regret}(K) = \sum_{k=1}^K [V^\star - V^{\pi_k}] .
\end{equation*}
The question then is whether a learner can keep the regret small. 

Below we describe several widely studied reinforcement learning models that can be cast into the framework of sequential interactive decision making. 
\begin{example}[Contextual bandit] 
In a contextual bandit, the observation is the context of the problem. The episode length $H$ is equal to $1$ and there exists a distribution $\mu\in\Delta_{\fO}$ so that the first-step observation $o_1$ of each episode is independently sampled from $\mu$, i.e., $\P(o_1=\cdot)=\mu$.
\end{example}

\begin{example}[MDP] In Markov decision process (MDP), the observation is the state of MDP. The observation-action pair satisfies the Markovian property. That is, there exist a collection of transition kernels $\T=\{\T_{h}\}_{h=1}^H$ so that $\P_h(o_h|o_{1:h-1}, a_{1:h-1}) = \T_{h,a_{h-1}}(o_h\mid o_{h-1})$ for all $h\in[H]$.
\end{example}

\begin{example}[POMDP]\label{defn:pomdp} In partially observable Markov decision process (POMDP), there is an additional latent  state space $\fS$, a collection of transition kernels $\T=\{\T_{h}\}_{h=1}^H$, an initial distribution over the latent state space $\mu_1$, and a collection of emission kernels $\O=\{\O_{h}\}_{h=1}^H$. 
In a POMDP, the latent states are \emph{hidden} from the agent. At the beginning of each episode, the environment samples an initial state $s_1$ from $\mu_1$. 
At each step $h\in[H]$, the agent first observes $o_h$ that is sampled from $\O_h(\cdot\mid s_h)$, the emission distribution of hidden state $s_h$ at step $h$. Then the agent takes action $a_h$ and  receives reward $r_h(o_h)$.
After this, the environment transitions to $s_{h+1}$, whose distribution follows $\T_{h,a_h}(\cdot\mid s_h)$.
\end{example}

We note that MDPs are fully observable models while POMDPs are partially observable models. Distinguished from MDPs where the optimal policies only depend  on the current observation, the near-optimal policies of POMDPs in general depend on the entire history.
This makes both learning and planning in POMDPs significantly more challenging than  in MDPs.

\subsection{Model-based function approximation}
We consider the interactive decision making problems where the observation space $\fO$, the action space $\fA$, the horizon $H$, and the reward function $R$ are known, while the system dynamics $\P$ is unknown. To address infinitely large observation space, we consider the setting where we are given a model class $\Theta$, which specifies a class of system dynamics $\{\P_\theta\}_{\theta \in \Theta}$.
We denote the system dynamics of the real model as $\P_{\theta^\star}$. Throughout this paper, we make the following realizability assumption.

\begin{assumption}[Realizability]
\label{asp:realizability}
$\theta^\star \in \Theta$.
\end{assumption}

Realizability states that the true model resides in the given model class, so there is no misspecification error. Realizability is a standard assumption which appears in a majority of theoretical works in RL.

Following the convention in analyzing MLE  \cite[e.g.,][]{geer2000empirical}, we use the bracketing number to control the complexity of the model class $\Theta$.

\begin{definition}[Bracketing number]
Given two functions $l$ and $u$, the \emph{bracket} $[l, u]$ is the set of all functions $f$ satisfying
$l \le f \le u$. An $\epsilon$-bracket is a bracket $[l, u]$ with $\norm{u-l} < \epsilon$ . The bracketing number
$\cN_{[\cdot]}(\epsilon, \cF, \norm{\cdot})$ is the minimum number of $\epsilon$-brackets needed to cover $\cF$.
\end{definition}

The bracketing number is required in the existing MLE analysis \cite{geer2000empirical}, which is in general equal or greater than the standard covering number.
Across this paper, we use $\cN_{\Theta}(\epsilon)$ to denote the $\epsilon$-bracketing number of  function class $\{\P_\theta\}_{\theta \in \Theta}$
with respect to the policy-weighted $\ell_1$-distance, where the policy-weighted $\ell_1$-distance between two functions $l$ and $u$ defined on  $(\fO\times\fA)^H$ is equal to
$\max_\pi \sum_{\tau_H} |l(\tau_H) - u(\tau_H)|\times\pi(\tau_H)$ \footnote{In the settings with infinite observations, we replace the summation with integral, i.e.,  $\max_\pi \int_{\tau_H} |l(\tau_H) - u(\tau_H)|\times\pi(\tau_H)d \tau_H$.}, where the maximum is taken over all policy $\pi$. Intuitively, we need this maximization, because $\P_{\theta}$ is a conditional probability of observations given actions.

\section{Optimistic MLE} 
\label{sec:OMLE_algorithm}

In this section, we present the generic \emph{Optimistic Maximum Likelihood Estimation} (\omle) algorithm. 
Moreover, we  provide a general sufficient condition---a generalized eluder-type condition (Condition \ref{cond:eluder}), and prove that for any RL problems satisfying this condition, OMLE learns them within a polynomial number of samples.

\subsection{Algorithm}

\begin{algorithm}[t]
    \caption{\textsc{\textbf{O}ptimistic \textbf{M}aximum \textbf{L}ikelihood \textbf{E}stimation} $(\Theta,\beta)$}
 \begin{algorithmic}[1]\label{alg:omle}
 \STATE \textbf{initialize:} $\cB^1 = \Theta$, $\mathfrak{D}=\{\}$ 
    \FOR{$k=1,\ldots,K$}
    \STATE compute $(\theta^k,\pi^k) \leftarrow \argmax_{\theta\in\cB^k, \pi} V^\pi(\theta)$ \label{line:optimisim}
    \STATE compute exploration policies $\Pi_{\exp}^k \leftarrow \Pi_{\exp}(\pi^k)$
    \label{line:exploration-start}
    \FOR{each $\pi \in \Pi_{\exp}^k$}
    \STATE execute policy $\pi$ and collect a trajectory $\tau = (o_1, a_1, \ldots, o_H, a_H)$
    \STATE add $(\pi, \tau)$ into dataset $\fD$\label{line:exploration-end}
    \ENDFOR
    \STATE update confidence set
    \vspace{-3mm}
    \begin{equation*}
    \cB^{k+1} = \bigg\{\hat\theta \in \Theta: \sum_{(\pi,\tau)\in\fD} \log \P_{{\hat\theta}}^{\pi} (\tau)
    \ge \max_{ \theta' \in\Theta} \sum_{(\pi,\tau)\in\fD} \log \P^{\pi}_{{\theta'}}(\tau) -\beta  \bigg\} \bigcap \cB^k
    \vspace{-4mm}
    \end{equation*}\label{line:confiSet}
    \ENDFOR
    \STATE \textbf{output} $\pi^{\rm out}$ that is a uniform mixture of  $\{\pi^k\}_{k=1}^K$
 \end{algorithmic}
 \end{algorithm}

The pseudocode of \omle~is provided in Algorithm \ref{alg:omle}.
We remark that the \omle~algorithm was first proposed in \cite{liu2022partially} for sample-efficient learning of weakly-revealing POMDPs and here we introduce some extra flexibility in the data collection steps to handle more general learning problems.

Formally, \omle~is a model-based algorithm which takes  as input a model class $\Theta$, and executes the following three key steps in each iteration $k\in[K]$:
\begin{itemize}
    \item \textbf{Optimistic planning} (Line \ref{line:optimisim}):  \omle~computes the most optimistic model $\theta^k$ in the model confidence set $\cB^k$ and its corresponding optimal policy $\pi^k$.
    \item \textbf{Data collection} (Line \ref{line:exploration-start}-\ref{line:exploration-end}): Based on the optimistic policy $\pi^k$, \omle\ constructs a set of exploration policies $\Pi_{\rm exp}(\pi^k)$ and then the learner executes each of them to collect a trajectory. As will be explained in later sections, these exploration policies could simply be $\pi^k$ or some composite policies that combine $\pi^k$ with random or certain action sequences, depending on the structure of the problems to solve. Intuitively, by actively trying exploratory action sequences after $\pi^k$, the learner could gather more information about the system dynamics under $\pi^{k}$. As an example, when applying \omle~to learning PSRs, the exploration policies will execute the core action sequences after $\pi^k$, which we will explain in details in Section \ref{sec:psr}. 
    \item  \textbf{Confidence set update}  (Line \ref{line:confiSet}):  Finally, \omle~updates the model confidence set using the newly collected data. Specifically, it constructs  $\cB^{k+1}$ to include all the models  $\theta\in\Theta$ whose log likelihood on all the historical data collected so far is close to the maximal log likelihood up to an  additive factor $\beta$. This can be viewed as a relaxation of the classic maximal likelihood estimation (MLE) approach which chooses the model estimate to be the one exactly maximizing the log likelihood. In particular, when $\beta=0$, $\cB^{k+1}$ reduces to the solution set of MLE. One important reason behind this construction is that by choosing the relaxation parameter $\beta$ properly, we can guarantee the true model $\theta^\star$ lies in the confidence set for all $k\in[K]$ with high probability, under the realizability assumption.
\end{itemize}

\subsection{Theoretical guarantees}

In this section, we present the theoretical guarantees for \omle. 
To present our results in the most general form, we first introduce a sufficient condition, called \emph{generalized  eluder-type condition}. We then provide the sample-efficiency guarantees for \omle~ in learning any RL problems that satisfy this condition. Let $\P_\theta^\pi$ denote the distribution over $(o,a,r)_{1:H}$ induced by executing policy $\pi$ in model $\theta$.

\begin{condition}[Generalized eluder-type condition] \label{cond:eluder}
There exists a real number $d_{\Theta} \in \R_+$  and a function $\xi$ such that: for any $(K,\Delta)\in\N\times\R^+$,  and  for the models $\{\theta^k\}_{k\in[K]}$ and the policies $\{\pi^k\}_{k\in[K]}$, $\{\Pi^k_{\exp}\}_{k\in[K]}$ in Algorithm \ref{alg:omle}, we have
\begin{equation} \label{eq:generalized_eluder}
    \forall k\in[K], ~\sum_{t=1}^{k-1} \sum_{\pi \in \Pi^t_{\exp}} d_\TV^2(\P_{\theta^k}^{\pi}, \P_{\theta^\star}^{\pi}) \le  \Delta   
    ~~\Rightarrow~~ \sum_{k=1}^{K} d_\TV( \P_{\theta^k}^{\pi^k} , \P_{\theta^\star}^{\pi^k}) \le  \xi(d_{\Theta}, K, \Delta,|\Pie|)
\end{equation}
where $|\Pie| := \max_\pi |\Pie(\pi)|$ is the largest possible number of exploration policies in each iteration.
\end{condition}
At a high level, Condition \ref{cond:eluder} resembles the pigeonhole principle   and the elliptical potential lemma widely used in tabular MDPs \citep[e.g.,][]{azar2017minimax,jin2018q} and linear bandits/MDPs \citep[e.g., ][]{lattimore2020bandit,jin2020provably} respectively. Such type of condition is widely used as a sufficient condition for algorithms using optimistic exploration \cite{russo2013eluder}.
Importantly, \emph{we will prove that Condition \ref{cond:eluder}  holds for all the problems studied in this paper}, with moderate $d_\Theta$ and function $\xi$ whose leading term scales as  $\tilde{\cO}(\sqrt{d_\Theta \Delta |\Pi_{\rm exp}| K})$.

For an intuitive understanding of this generalized eluder-type condition, imagine that in each $k^{\rm th}$ iteration, the learner chooses a model $\theta^k$ such that $\theta^k$ can accurately predict the behavior of the historical exploration  policies in  $\Pi^1_{\rm exp},\ldots,\Pi^{k-1}_{\rm exp}$ up to cumulative error $\Delta$ (i.e., the left inequality of \eqref{eq:generalized_eluder}). Since $\theta^k$ could be different from $\theta^\star$, the learner will still suffer an instantaneous error in predicting the behavior of policy $\pi_k$ using model $\theta_k$. 
And $\xi(d_{\Theta}, K, \Delta,|\Pie|)$ essentially measures the worst-case growth rate of the cumulative instantaneous error with respect to $K$.

The key motivation behind  Condition \ref{cond:eluder} is that because of the way \omle~constructs the confidence set $\cB^k$, we can use the classical analysis of MLE \cite{geer2000empirical} to guarantee that any model inside $\cB^k$
is close to the true model $\theta^\star$ in TV-distance under the  historical policies in $\Pi^1_{\rm exp},\ldots,\Pi^{k-1}_{\rm exp}$
with high probability. 
As a result, if the problem further satisfies the generalized eluder-type condition, then \omle~immediately enjoys low-suboptimality guarantee by the optimism of $\{\pi^k\}_{k=1}^K$ and Condition \ref{cond:eluder}. Formally, we have the following theoretical guarantee for \omle.

\begin{restatable}{theorem}{thmomle} \label{thm:omle}
There exists absolute constant $c_1,c_2>0$ such that for any $\delta\in(0,1]$ and $K\in\N$,  if we choose $\beta = c_1\log(T\cN_{\Theta}(T^{-1})/\delta)$ with $T=K|\Pi_{\rm exp}|$ in OMLE (Algorithm \ref{alg:omle}) and assume Condition \ref{cond:eluder} holds, then with probability at least $1-\delta$, we have
      $\sum_{k=1}^K [V^\star - V^{\pi_k}] \le H \xi(d_{\Theta}, K, 
      c_2\beta,|\Pie|)$.
\end{restatable}

As mentioned before, for all problems studied in this paper, the leading term (in terms of $K$ dependency) of function $\xi$   scales as $\tilde{\cO}(\sqrt{d_\Theta \beta |\Pie| K})$. 
Then, Theorem \ref{thm:omle} immediately leads to a guarantee $\sum_{k=1}^K [V^\star - V^{\pi_k}] \le \tilde{\cO}(H\sqrt{d_\Theta \beta |\Pie| K}) + o(\sqrt{K})$, which gives the optimal $\sqrt{K}$ dependency up to a polylogarithmic factor. We remark that Theorem \ref{thm:omle} is not a regret guarantee unless $\Pi^k_{\exp}=\{\pi^k\}$, because it is the policies in  $\{\Pi^k_{\exp}\}_{k=1}^K$  that are executed by \omle, not $\{\pi^k\}_{k=1}^K$.

\paragraph{Sample complexity.}  
Since the output policy $\pi^{\rm out}$ is a uniform mixture of $\{\pi^k\}_{k=1}^K$, we have $V^{\pi^{\rm out}} = (\sum_{k=1}^K V^{\pi_k})/K$.  
 As a result, Theorem \ref{thm:omle} immediately implies that with probability at least $1-\delta$,  $\pi^{\rm out}$ of \omle  is $\epsilon$-optimal as long as $H \xi(d_{\Theta}, K, \beta)/K \le \epsilon/2$. In particular, when $\xi(d_{\Theta}, K, \beta)$ scales as  $\tilde{\cO}(\sqrt{K})$ with respect to $K$, it suffices to run OMLE for $K\ge \tilde{\cO}(\epsilon^{-2})$ episodes, where the dependency on $\epsilon$ is again optimal up to a polylogarithmic factor.


\section{Low-rank Sequential Decision Making}\label{sec:psr}

In this section, we consider an important large class of sequential decision making problems which has a low-rank structure. Note that the entire dynamics of the sequential decision making problem is fully specified by the joint probability $\P(o_{1:H}|a_{1:H})$. We can equivalently view this joint probability as system-dynamic matrices $\{\D_h\}_{h\in[H]}$: for each fixed step $h$, we call an observation-action sequence in previous steps up to $h$, i.e., $\tau_h=(o_{1:h}, a_{1:h})$  a \textbf{history}, and call an observation-action sequence in future steps, i.e., $\omega_h = (o_{h+1:m}, a_{h+1:m})$ for any $m\in [h+1, H]$  a \textbf{future} (or test).  Denote the set of all possible histories at step $h$ as $\cT_h$ and the set of all possible futures as $\Omega_h$. Then we can define the system-dynamic matrix $\D_h \in \R^{|\cT_h| \times |\Omega_h| }$ as a matrix with histories as rows and futures as columns\footnote{For clean presentation, here we write $\D_h$ as a matrix, which requires $|\Omega_h|$ or $|\fO|$ to be finite. We remark that our framework immediately extends to the infinite observation setting. See Appendix \ref{sec:psr_cts} for more details.}
whose entry is specified as
\begin{equation} \label{eq:system_dynamic_matrix}
    [\D_h]_{\tau_h, \omega_h} = \Pb(\tau_h, \omega_h) := \P(o_{1:H}|a_{1:H})
\end{equation}
The \textbf{rank} of the sequential decision making problem is simply defined as $\max_{h\in[H]} \rank (\D_h)$, which is the maximal rank of the system-dynamic matrices $\{\D_h\}_{h\in[H]}$.

\subsection{Predicative state representations}
\label{subsec:psr-defn}

Predicative State Representations (PSRs) are proposed by \cite{littman2001predictive,singh2012predictive} as a generic approach to model low-rank sequential decision making problems. Consider a fixed step $h \in [H-1]$, and denote $r = \rank(\D_h)$. For any integer $d \ge r$, there always exist $d$ columns (denoted as $\cQ_h$) of matrix $\D_h$, such that the submatrix restricted to these columns $\D_h[\cQ_h]$ satisfies $\rank(\D_h[\cQ_h]) = r$. These $d$ columns correspond to $d$ futures $\cQ_h = \{q_1, \ldots, q_d\}$, which are called \emph{core tests}. Throughout this section, we assume all models in our model class $\Theta$ share the same sets of core tests, which are known to the learner. While  most  literature in PSRs often choose $d = r$, in many applications (as shown in the next section), learner only knows a set of core tests with a larger size. Therefore, we also consider the setting when $d > r$, to which we refer as \emph{overparameterized} PSR.

Core tests allow the system-dynamic matrix $\D_h$ to be factorized as follows for certain matrix $\W_h$:
\begin{equation} \label{eq:psr_factorization}
    \D_h = \D_h[\cQ_h] \cdot \W_h\trans,  \quad \D_h[\cQ_h] \in \R^{|\cT_h|\times d}, \W_h \in \R^{|\Omega_h|\times d}
\end{equation}
This implies an important property: for any history $\tau_h$, the  $\tau_h^{\text{th}}$ row of $\D_h[\cQ_h]$, which we denote as $\bpsi(\tau_h) := (\Pb(\tau_h, q_1), \ldots, \Pb(\tau_h, q_d))$, serves as a sufficient statistics for the history $\tau_h$ in predicting the the probabilities of all futures conditioned on $\tau_h$. In sum, PSR captures the state of a dynamic system using $\bpsi(\tau_h)$---a vector of predictions for future tests.

Formally, PSR models the dynamic system using a tuple $(\bphi, \bM, \bpsi_0)$, where $\bphi=\{\bphi_H(o, a)\}_{(o, a) \in \fO \times \fA}$ is a set of vectors where $\bphi_H(o, a) \in \R^{|\cQ_{H-1}|}$; $\bM = \{\bM_h(o, a)\}_{(h, o, a) \in [H-1] \times \fO \times \fA}$ is a set of matrices where $\bM_h(o, a) \in \R^{|\cQ_h| \times |\cQ_{h-1}|}$, and $\bpsi_0$ is a vector in $\R^{|\cQ_0|}$. The tuple satisfies following two equations:
\begin{align}
    \P(o_{1:H}|a_{1:H}) =& \bphi_H(o_H, a_H)\trans \bM_{H-1}(o_{H-1}, a_{H-1}) \cdots \bM_1(o_1, a_1)\bpsi_0,   \label{eq:PSR_1}\\
    \bpsi(o_{1:h}, a_{1:h}) =& \bM_{h}(o_{h}, a_{h}) \cdots \bM_1(o_1, a_1) \bpsi_0,  \label{eq:PSR_2}
\end{align}
for any $h \in [0, H-1]$ and any observation-action sequence $(o_{1:h}, a_{1:h})$.
That is, in PSR, the joint probability $\P(o_{1:H}|a_{1:H})$ can be factorized as a product of matrices and vectors where each matrix only depends on the observation and action at the corresponding step. The second condition \eqref{eq:PSR_2} further requires the product of the first $h$ matrices to have a probabilistic interpretation---the sufficient statistics $\bpsi(o_{1:h}, a_{1:h})$ for the history $(o_{1:h}, a_{1:h})$.
In condition \eqref{eq:PSR_2}, we include the special case $h=0$, where the history $\tau_h$ is empty $\emptyset$, and the condition becomes $\bpsi(\emptyset) := (\Pb(q_1), \ldots, \Pb(q_d)) = \bpsi_0$ for core tests $\{q_1, \ldots, q_d\}$ in $\cQ_0$. 
We call the sets of core tests $\{\cQ_h\}_{h \in [H-1]}$ along with the tuple $(\bphi, \bM, \bpsi_0)$ the PSR representation of the dynamic system. 
Finally, we define the rank of a PSR to be the rank of the underlying sequential decision making problem that the PSR describes (according to \eqref{eq:PSR_1}).

\paragraph{Representation power of PSRs.} The following theorem \cite[see e.g.,][]{littman2001predictive,singh2012predictive} guarantees the existence of such PSR representation $(\bphi, \bM, \bpsi_0)$ for any low-rank sequential decision making problem.

\begin{theorem}\label{thm:sdm-psr}
    Any rank-$r$ sequential decision making problem can be represented by a PSR with sets of core tests whose sizes are no larger than $r$. That is, there always exist sets of core tests $\{\cQ_h\}_{h \in [H-1]}$ with size $\max_{h \in [H-1]}|\cQ_h| \le r$, and a corresponding tuple $(\bphi, \bM, \bpsi_0)$ which jointly satisfy Equation  \eqref{eq:PSR_1} \eqref{eq:PSR_2}.
\end{theorem}

Theorem \ref{thm:sdm-psr} demonstrates the superior expressive power of PSR, in the sense  that any low-rank sequential decision making problem admits an equivalent and compact  PSR representation.
This is in sharp contrast to other models of dynamical systems such as POMDPs which not only implicitly require the system dynamics being low-rank but also explicitly assume the existence of latent nominal states so that the current state of the system can be represented as a probability distribution over these unobservable nominal states. 
As a result, PSRs can model strictly more complex dynamical systems than POMDPs with finite states, e.g., the probability clock introduced in  \cite{jaeger1998discrete}.

\paragraph{Linear weight vectors.} According to low rank factorization \eqref{eq:psr_factorization}, we know there exist linear weight vectors $\{\m(\omega_h)\}_{\omega_h \in \Omega_h}$ only depending on the futures (where $\m(\omega_h)$ can be the $\omega_h^{\text{th}}$ row of $\W_h$ matrix) such that for any future $\omega_h$ and history $\tau_h$, the joint probability can be written in the bilinear form 
\begin{equation}\label{eq:psr-m-psi}
    \Pb(\tau_h, \omega_h) = \m(\omega_h)\trans \bpsi(\tau_h).
\end{equation} 
Equation \eqref{eq:PSR_1} and \eqref{eq:PSR_2} give two natural constructions for weight vectors. First, consider futures of full length $\Omega^{(1)}_h:=(\fO\times\fA)^{H-h}$. Equation \eqref{eq:PSR_1} gives the weight vector of any future $\omega_h = (o_{h+1:H}, a_{h+1:H}) \in \Omega^{(1)}_h$ as:
\begin{equation} \label{eq:weight_vector_1}
    \m_1(\omega_h)\trans =  \bphi_H(o_H, a_H)\trans \bM_{H-1}(o_{H-1}, a_{H-1}) \cdots \bM_{h+1}(o_{h+1}, a_{h+1})
\end{equation}
Next, consider the future set of $\Omega^{(2)}_h:=\fO \times \fA \times \cQ_{h+1}$. Equation \eqref{eq:PSR_2} gives the weight vector of any future $\omega_h = (o_{h+1}, a_{h+1}, q_i) \in \Omega^{(2)}_h$ (where $q_i \in \cQ_{h+1} $ is the $i\th$ core test of $\cQ_{h+1}$)as:
\begin{equation} \label{eq:weight_vector_2}
    \m_2(\omega_h)\trans =  \bm{e}_i \trans \bM_{h+1}(o_{h+1}, a_{h+1})
\end{equation}
We note that in the overparameterized setting ($|\cQ_h| > \rank(\D_h)$), the choice of linear weights $\m(\cdot)$ in \eqref{eq:psr-m-psi} may not be unique. As a result, the constructions in \eqref{eq:weight_vector_1} and \eqref{eq:weight_vector_2} are not necessarily related in general, unless a further \emph{self-consistent} condition is satisfied (see discussion in Appendix \ref{app:self-consistent} for more details).

\paragraph{Core action sequences.} We note that multiple core tests might use the same action sequence $a_{h+1:m}$ for $m \in [h+1, H]$. Therefore, in many occasions, it is convenient to consider the set of core action sequences $\Qa_h$, which is the set of unique actions sequences within the set of core tests $\cQ_h$. We know immediately that $|\Qa_h| \le |\cQ_h|$ and any rank-$r$ system-dynamic matrix $\D_h$ admits at least one set of core action sequences with size $|\Qa_h| \le r$. 
The size of core action sequences $|\Qa_h|$ determines the number of experiments we need to conduct in the dynamic system in order to estimate $\bpsi(\tau_h)$. As we will see later, all our sample complexity results only depend on $|\Qa_h|$ instead of $|\cQ_h|$. WLOG, we assume that no core action sequence is a prefix of another core action sequence.

\paragraph{Continuous observation.} For clean presentation, we write the results in this section using the formulation with finite observations. As we will see, our sample complexity results are completely independent of the number of observations, which allows our results to readily extend to the setting of continuous observation. 
For rigorous treatment, we note that in our current definition each core test is a single observation-action sequence, which has probability $0$ to be observed if the observation is continuous. In Appendix \ref{sec:psr_cts}, we provide two approaches to modify the PSR formulation to resolve this issue. One approach is to consider a dense set of core tests with infinitely many futures, and generalize $\bpsi(\tau)$ and $\bM(o, a)$ from vectors and matrices to functions and linear operators in Hilbert space. Our results remain meaningful even with infinitely many core tests as long as the number of core action sequences is small. The second approach is to generalize the definition of core test to be an event of whether the future lands in a measurable subset of future space. We defer the details of rigorous treatment of continuous observation to Appendix \ref{sec:psr_cts}.

\subsection{Well-conditioned PSRs}

Since PSR includes POMDP as a special case, it naturally inherits all the hardness results of learning POMDPs. In particular, even when the observation space, the action space and the sets of core tests are all small, finding a near-optimal policy still requires an exponential number of samples in the worst case.
\begin{proposition} \label{prop:psr-hard}
There exists a family of PSRs with  $|\fO|,|\fA|,\max_{h}|\cQ_h|=\cO(1)$  so that any algorithm requires at least $\Omega(2^{H})$ samples to learn a $(1/4)$-optimal policy with probability $1/6$ or higher.
\end{proposition}

The proof of Proposition \ref{prop:psr-hard} essentially follows from Theorem 6 in \cite{liu2022partially} which shows the hardness for learning POMDPs when the weakly-revealing coefficient is bad. See Appendix \ref{app:psr-hard} for details.

Intuitively, the hard instances in Proposition \ref{prop:psr-hard} is due to the following reason: in the definition of PSR, we require that for each step $h$, the core tests $\cQ_h$ satisfies $ \rank(\D_h[\cQ_h]) = \rank(\D_h) := r$. However, this requirement alone does not prohibit the submatrix $\D_h[\cQ_h]$ to be extremely close to some matrix whose rank is strictly less than $r$. That is, matrix $\D_h[\cQ_h]$ can be highly ill-conditioned. This will lead to high non-robustness in predicting the probability of $\Pb(\tau_h, \omega_h) = \m(\omega_h)\trans \bpsi(\tau_h)$ when the vector $\bpsi(\tau_h)$ needs to be estimated---the corresponding linear weight $\m(\omega_h)$ can be extremely large such that we need to estimate $\bpsi(\tau_h)$ up to an extremely high accuracy. Indeed, in the hard instances of Proposition \ref{prop:psr-hard}, there exists some future $\omega_h$ such that $\|\m(\omega_h)\|_1\ge \Omega(2^H)$.

To rule out such hard instances, core tests are required to not only guarantee $\rank(\D_h[\cQ_h]) := r$, but also ensure $\D_h[\cQ_h]$ to be ``well-conditioned'' in certain sense. In this paper, we enforce such condition by assuming an upper bound on the magnitude of linear weight vectors.

\begin{condition}[$\gamma$-\name  PSR]\label{asp:psr}
We say a PSR is \emph{$\gamma$-\name} if for any $h\in[H-1]$ and any policy $\pi$ independent of the history before step $h+1$, the weight vectors $\m_1(\cdot), \m_2(\cdot)$ and the corresponding future sets $\Omega^{(1)}_h, \Omega^{(2)}_h$ in \eqref{eq:weight_vector_1} \eqref{eq:weight_vector_2} satisfy: 
\begin{equation} \label{eq:well_condition_psr}
\max_{i \in \{1, 2\}} ~ \max_{\substack{~~~~\x \in \R^{|\cQ_h|} \\ \norm{\x}_1 \le 1}
} ~ \sum_{\omega_{h}\in{\Omega^{(i)}_h}} \pi( \omega_h) \cdot |\m_i(\omega_h)\trans\x|  \le \frac{1}{\gamma}~.
\end{equation}
\end{condition}

\begin{remark}
Condition \ref{asp:psr} requires $\max_{i \in \{1, 2\}}$ because for overparameterized PSR, linear weight vectors are not unique. Thus, $\m_1, \m_2$ in general are not related. However, if they are related by  self-consistency (see Appendix \ref{app:self-consistent} for details), then it is sufficient to assume the inequality \eqref{eq:well_condition_psr} only for $\m_1, \Omega^{(1)}_h$.
\end{remark}

Intuitively, the parameter $\gamma^{-1}$ above  measures how much the future weight vectors $\{\m(\omega_h)\}_{\omega_h \in {\Omega}_h}$ can amplify the  error $\x$ arising from estimating  the probability of core tests, in an averaged sense that the future $\omega_h$ is  sampled from policy $\pi$. 
Being $\gamma$-\name naturally requires this error amplification to be not   extremely large since otherwise the hard instances mentioned before will come into play.  
In Section \ref{sec:psr-example}, we will prove  many common partially observable RL problems are naturally $\gamma$-\name PSRs with moderate $\gamma$, e.g., observable POMDPs and multistep decodable  POMDPs.

\subsection{Theoretical results}
In this subsection, we present the theoretical guarantees for learning well-conditioned PSRs with \omle.
To analyze \omle, we first need to specify the exploration policy function $\Pie$. 
Denote by $\nu(\pi,h,\a)$ a composite policy that first executes policy $\pi$ for step $1$ to step $h-1$, then takes  random action at step $h$, and after that  executes action sequence $\a = (a_{h+1}, \ldots, a_m)$ till certain step $m$, and finally finishes the remaining steps of the current episode by taking random actions. We construct the following exploration policy function:
\begin{equation}\label{eq:explore-func}
    \Pi_{\rm exp}(\pi):=\bigcup_{h\in[H-1]}\{\nu(\pi,h,\a): \a\in \Qa_h \}.
\end{equation}
By using the above exploration policy function in \omle, we have the following polynomial sample-efficiency guarantee for learning well-conditioned PSRs. 
\begin{theorem}
\label{thm:psr}
Let $c>0$ be an absolute constant large enough and $\Theta$ be a rank-$r$ $\gamma$-well-conditioned PSR class. 
For any $\delta\in(0,1]$ and $K\in\N$,  if we choose $\beta = c\log(T\cN_{\Theta}(T^{-1})\delta^{-1})$ with $T=KH\max_h|\Qa_h|$ and $\Pie$ specified by Equation \eqref{eq:explore-func} in \omle (Algorithm \ref{alg:omle}), then with probability at least $1-\delta$, we have
      $$
      V^\star - V^{\pi^{\rm out}} \le \poly(r,\gamma^{-1},\max_h|\Qa_h|,A,H,\log K)\times \sqrt{\frac{\beta}{K}}~.
      $$
\end{theorem}
The result in Theorem \ref{thm:psr} scales polynomially with respect to the rank of the PSR $r$, the inverse well-conditioned  parameter $\gamma^{-1}$,  the number of core action sequences  $\max_h|\Qa_h|$, the log-bracketing number of the model class $\log\cN_\Theta$, the number of actions $A$, and the episode length $H$. 
In particular, (1) it does \emph{not} depend on the size of core tests, but instead only depend on the size of core action sequence; (2) it is completely independent of the size of the observation space. Both empower our results to handle problems with continuous observations.
Moreover, when the bracketing number satisfies $\log\cN_{\Theta}(T^{-1})\le \cO({\rm polylog}(T))$ (e.g., in tabular PSRs and POMDPs with mixture of Gaussian observations),  Theorem \ref{thm:psr} guarantees that $K=\tilde{\cO}(\epsilon^{-2})$ episodes suffices for finding an $\epsilon$-optimal policy, which is optimal up to a polylogarithmic factor.

The proof of Theorem \ref{thm:psr} relies on the following key lemma, which states that any class of well-conditioned PSRs satisfy the generalized eluder-type condition (Condition \ref{cond:eluder}) with  favorable $d_\Theta$ and $\zeta$.

\begin{lemma}\label{prop:psr-eluder}
    Let  $\Theta$ be a family of rank-$r$ $\gamma$-well-conditioned PSRs. Then  Condition \ref{cond:eluder} holds with $\Pie$ defined in Equation \eqref{eq:explore-func},  $d_\Theta=(r\gamma^{-2} A^2\max_h|\Qa_h|)^2\poly(H)$, and $\xi(d_\Theta,\Delta, |\Pi_{\rm exp}|,K)=\tilde{\cO}(\sqrt{d_\Theta \Delta |\Pi_{\rm exp}| K})$.
\end{lemma}

Once Lemma \ref{prop:psr-eluder} is established, 
Theorem \ref{thm:psr} follows immediately from combining it with the guarantee of \omle (Theorem \ref{thm:omle}). 

\paragraph{Technical challenge.} One of the key steps in proving 
Lemma \ref{prop:psr-eluder} is to establish a generalized  version of elliptical potential lemma for Summation of Absolute values of Independent biLinear (SAIL) functions of form $\sum_{i=1}^m\sum_{j=1}^n |\langle \theta_i,x_j\rangle|$. Despite similar problems have been investigated in the previous analysis of OMLE \citep{liu2022partially}, the bound derived therein scales with $m,n$, which  depend on  the number of observations. 
As a result, that bound is incapable to handle the settings with infinite observations. To address this issue, we develop a much tighter elliptical potential lemma which completely get rids of the $m,n$ dependence. 
With the help of this strengthened elliptical potential lemma and other newly developed techniques, we are able to prove Lemma \ref{prop:psr-eluder} without suffering any dependence on the size of the observation space. We refer an interesting reader to Appendix \ref{app:l1-pigeon-hole} for more technical details.

\subsubsection{Special cases: tabular PSRs}
To apply Theorem \ref{thm:psr}, we still need to upper bound the bracketing number of model class $\Theta$. 
The following proposition states that in tabular PSRs (i.e., PSRs with finite observations and actions) the log-bracketing number of $\Theta$ is always upper bounded.
\begin{theorem}[bracketing number of tabular PSRs]\label{thm:tabularPSR-bracket}
Let $\Theta$ be the collections of all  rank-$r$ PSRs  with $O$ observations, $A$ actions and episode length $H$. Then
$\log\cN_\Theta(\epsilon) \le \cO(r^2OAH^2 \log (rOAH/\epsilon))$.
\end{theorem}
We remark that the bracketing number in Theorem \ref{thm:tabularPSR-bracket} is independent of the size of core tests or core action sequences. This is because the representation power of rank-$r$ PSRs is limited to rank-$r$ sequential decision making problems regardless the choices of core tests.

The key intermediate step in proving Theorem \ref{thm:tabularPSR-bracket} is to show every low rank sequential decision making problem admits an observable operator model (OOM) representation wherein the norm of the operators are well controlled. Once this argument is established,  we can upper bound the bracketing number  by  discretizing those operators. 
In comparison, recent works on PSRs \citep{zhan2022pac} simply \emph{assume} every PSR representation has bounded  operator norm  without proving it. 
To our knowledge, Theorem \ref{thm:tabularPSR-bracket} provides the first polynomial upper bound for the bracketing number of tabular PSRs \emph{without any additional assumptions}.

Finally, by plugging the above upper bound back into Theorem \ref{thm:psr}, we immediately obtain the following sample complexity bound for learning tabular PSRs: 
$$
\poly(r,\gamma^{-1},\max_h|\Qa_h|,O,A,H,\log(\epsilon^{-1}\delta^{-1}))\cdot \epsilon^{-2}.
$$


\section{Important PSR Subclasses}\label{sec:psr-example}

In this section, we introduce several  partially observable RL problems of interests  and prove that they are all special subclasses of $\gamma$-\name PSRs with moderate $\gamma$.
All the proofs for this section 
are deferred to Appendix \ref{app:psr-examples}.

\subsection{Observable  POMDPs}\label{subsec:weakly-revealing-pomdps}

We first consider observable POMDPs \cite{golowich2022planning}
\footnote{ \cite{liu2022partially} considers a similar subclass called \emph{weakly-revealing} POMDPs, which assumes the $S^{\text{th}}$ singular value of matrix $\M_h$ to be lower bounded. Here $\M_h$ is a matrix of size $O^{m}A^{m-1} \times S$ whose entry is defined in \eqref{defn:M}. \cite{liu2022partially} proved that observable POMDPs and weakly-revealing POMDPs are equivalent up to a polynomial factor that depends on the number of states and observations. Therefore, there is essentially no difference in proving polynomial sample complexity for two classes in the tabular setting. However, observable POMDPs extend more naturally to the setting of continuous observation, and natural examples in continuous observation such as GM-POMDPs in Section \ref{sec:GM-POMDP} do not satisfy weakly-revealing condition.
}
---an important, natural and rich subclass of POMDPs wherein there exists an integer $m\in [H]$ so that any two different distributions over latent states induce different $m$-step observation-action distributions. We will prove a new result that OMLE can sample-efficiently learn any observable POMDP even with \emph{infinite} or \emph{continuous} observation. We remark that while such a result have been proved in the setting of finite observations \cite{liu2022partially}, the sample complexity in \cite{liu2022partially} has a polynomial dependency on the number of observations, thus does not extend to the setting of continuous observation. Our new result is highly non-trivial: in addition to the sample-complexity guarantees of well-conditioned PSRs with continuous observation (Theorem \ref{thm:psr}), our result further requires new techniques on matrix pseudo-inverse with small $\ell_1$-norm  (Appendix \ref{app:l1-inverse}) and a new core tests design technique (Appendix \ref{app:observable-continuous}),


To formally state the observability condition, we first define the $m$-step observation-action  probability kernels as follows: 
\[
\{\M_h\in\{\fO^m\times\fA^{m-1}\rightarrow\R\}^S\}_{h\in[H-m+1]}
\]
For an observation sequence $\o$ of length $m$, a latent state $s$ and an  action sequence $\a$ of length $m-1$,
the value of the $s^{\rm th}$ probability function in $\M_h$ at point $(\o,\a)\in\fO^m\times\fA^{m-1}$, denoted as $\M_{h,s}(\o,\a)$,  is equal to the probability density of observing  $\o$ provided that the action sequence $\a$ is used from state $s$ and step $h$: 
 \begin{equation}\label{defn:M}
     \M_{h,s}(\o,\a):= \P(o_{h:h+m-1}=\o \mid 
     s_h=s,a_{h:h+m-2} =\a ).
 \end{equation}
And we say a POMDP is $m$-step $\alpha$-observable ($m\in[H]$ and $\alpha>0$),  if its $m$-step observation-action   
 probability kernels satisfy the following condition.
\begin{condition}[$m$-step $\alpha$-observable condition]\label{cond:pomdp}
For any $\bnu_1,\bnu_2\in\Delta_{\fS}$ and $h\in[H-m+1]$,
\begin{equation}\textstyle 
    \|\E_{s\sim \bnu_1}[\M_{h,s}] - \E_{s\sim \bnu_2}[\M_{h,s}]\|_1 \ge \alpha \|\bnu_1-\bnu_2\|_1.
\end{equation}
\end{condition}
In the above condition, we use  $\|f-g\|_1=\int_{x\in\cX} |f(x)-g(x)| dx$ to denote the $\ell_1$-distance between two functions from $\cX$ to $\R$. 
Intuitively, Condition \ref{cond:pomdp} can be viewed as a robust version of assuming that the $S$ probability functions in each $\M_h$ are linearly independent, which guarantees that for any two different latent state mixtures $\bnu_1,\bnu_2\in\Delta_S$, there exists an action sequence $\a$ of length $m-1$ so that these two mixtures can be distinguished from the distributions over the next $m$-step observations provided that action sequence $\a$ is executed.  

The following theorem states that any $m$-step $\alpha$-observable POMDP admits an $\alpha/(S+A^{m-1})$-\name PSR representation with core action sets equal to $\fA^{m-1}$. 
\begin{theorem} \label{thm:pomdp}
Let $\Theta$ be a model class of $m$-step $\alpha$-observable POMDPs. Then $\Theta$ satisfies Condition \ref{asp:psr} with $\gamma=\cO(\alpha/S)$ and $\cQ_h^A=\fA^{\min\{m-1,H-h\}}$. 
\end{theorem}
\paragraph{New PSR operators for observable POMDPs.} The key challenge in proving  Theorem \ref{thm:pomdp} is to construct a set of PSR operators that satisfy  Condition \ref{asp:psr} with parameter $\gamma$ independent of the number of observations  $O$. For simplicity of illustration, let us consider $1$-step $\alpha$-observable tabular POMDPs as examples in this paragraph. Previous work \citep{liu2022partially} and concurrent  works \citep{zhan2022pac,chen2022partially} all adopt  the following operator construction:
$$
\bM_h(o,a)=\O_{h+1}\T_{h,a}\diag(\O_h(o\mid \cdot))\O_h^\dagger \in \R^{O\times O},
$$
where $\T_{h,a}\in\R^{S\times S}$ is the transition matrix of action  $a$,  $\O_h\in\R^{O\times S}$ is the observation matrix and $\O_h(o\mid \cdot)$ is the $o\th$ row of $\O_h$, all for step $h$.
However, the above operators have $\gamma$ scaling as  $\cO(\alpha/\sqrt{O})$ in the worst case, which hinders  generalization to the infinite-observation settings. To address this issue, we propose a different  operator construction  based on a novel $\ell_1$-norm matrix inverse technique (Lemma \ref{lem:inverse-O}):
$$
\bM_h(o,a)=\O_{h+1}\T_{h,a}\diag(\O_h(o\mid \cdot))(\Y_h +\O_h^\dagger)\in \R^{O\times O}
\text { where } \Y_h\in\argmin_{\tilde\Y\in\R^{S\times O}}\|\tilde\Y+\O_h^\dagger\|_1,
$$
which, importantly,  satisfies Condition \ref{asp:psr} with   $\gamma=\cO(\alpha/S)$ completely independent of $O$. 
When moving from the single-step observable tabular  setting to the more challenging multi-step observable infinite-observation setting, 
the same idea still plays an important role  in constructing well-conditioned PSR operators,  where we  first use a novel  partition technique to group different observations to obtain an $(\alpha/2)$-observable meta-POMDP with finite but exponentially many meta-observations and then apply the above operator construction on top of the meta-POMDP.  For more technical details, please refer to Appendix \ref{app:observable-finite} and \ref{app:observable-continuous}.

\paragraph{Sample complexity.} By combining Theorem \ref{thm:pomdp}  with Theorem \ref{thm:psr}, we immediately obtain the following sample-efficiency guarantee for learning observable POMDPs with OMLE.

\begin{corollary}\label{cor:pomdp-sample}
Let $\Theta$ be a model class of $m$-step $\alpha$-observable POMDPs.
There exists an absolute constant $c>0$ such that for any $\delta\in(0,1]$ and $K\in\N$,  if we choose $\beta = c\log(T\cN_{\Theta}(T^{-1})\delta^{-1})$ with $T=KHA^m$ in \omle (Algorithm \ref{alg:omle}), then with probability at least $1-\delta$, 
      $$
      V^\star - V^{\pi^{\rm out}} \le \poly(\alpha^{-1},S,A^m,H,\log K)\times \sqrt{\frac{\beta}{K}}~.
      $$
\end{corollary}
Different from previous works on tabular POMDPs \citep{jin2020provably,liu2022partially,golowich2022learning} where the sample complexity scales with the number of observations, The result in Corollary \ref{cor:pomdp-sample} completely gets rid of the dependence on $O$ thanks to our novel  PSR operator design  as is discussed above. As a result, it also applies to learning observable POMDPs with continuous observations as long as the log-bracketing number of model class $\Theta$ is well controlled, whereas previous works cannot.

\subsubsection{Observable tabular POMDPs}

We first consider tabular observable POMDPs where the number of observations is finite. In this case, the $m$-step observation-action probability kernel $\M_h$ is equivalent to an $O^{m}A^{m-1}$ by $S$ matrix wherein the entry at the intersection of the $(\o,\a)^{\rm th}$ row and the $s^{\rm th}$ column is equal to $\P(o_{h:h+m-1}=\o \mid s_h=s,a_{h:h+m-2} =\a )$. And the observable condition (Condition \ref{cond:pomdp}) can be equivalently written as:
\begin{equation}\label{eq:tabular-weakly-revealing}
\max_{h\in[H-m+1]} \| \M_h (\bnu_1 -\bnu_2)\|_1 \ge  \alpha \|\bnu_1 - \bnu_2\|_1.    
\end{equation}
To apply \omle to tabular POMDPs with $S$ states, $O$ observations and $A$ actions, we choose the model class $\Theta$ to consist of all the legitimate POMDP parameterizations $\theta=(\T,\O,\bmu_1)$  whose corresponding $m$-step observation-action probability matrices satisfy Equation \eqref{eq:tabular-weakly-revealing}.
By simple discretization argument \citep[e.g., see Appendix B in ][]{liu2022partially}, we can bound the $\epsilon$-bracketing number of $\Theta$ by 
\begin{equation}\label{eq:pomdp-bracketing}
\log\cN_\Theta(\epsilon) \le \cO(H(S^2A+SO)\log(SAOH\epsilon^{-1})).
\end{equation}
Plugging the above upper bound back into Theorem \ref{thm:pomdp}, we immediately recover the sample efficiency guarantee for learning tabular observable POMDPs in \citep{liu2022partially}.

\subsubsection{Observable POMDPs with Gaussian emission} \label{sec:GM-POMDP}
To showcase the power of Theorem \ref{thm:pomdp}  in handling POMDPs with continuous observations, we consider the model of POMDPs with  Gaussian mixture emissions (abbreviated as GM-POMDP hereafter), which can be intuitively viewed as tabular observable or weakly revealing POMDPs with observations corrupted by Gaussian noise. The Gaussian emissions further allow us to directly control the bracketing number. We start with the formal definition of GM-POMDPs.
\begin{definition}
[GM-POMDPs]\label{example:gm-pomdp}
A $d$-dimensional $n$-components GM-POMDP is a POMDP where the observation distributions are $d$-dimensional Gaussian mixtures of size $n$, i.e., $$\O_h(\cdot\mid s)= \sum_{i=1}^n \Wbb_{h}(i\mid s) \times {\rm Gauss}(\x_{h,i},\sigma_h\cdot\I_{d\times d})$$ where $\Wbb_{h}(\cdot\mid s)\in\Delta_n$, $\x_{h,i}\in\R^d$ and $\sigma_{h}>0$.
\end{definition}
Without further assumptions on GM-POMDPs, the observable condition  can be arbitrarily violated and sample-efficient learning is in general impossible.  Therefore, we introduce the following natural separation condition on the Gaussian mixtures in GM-POMDPs, which, once being satisfied, immediately implies the observable condition holds.  To condense notations, denote $\Wbb:=[\Wbb_{h}(\cdot\mid s)]_{s\in\fS}\in\R^{n\times S}$.
\begin{condition}[$\eta$-separable condition]\label{cond:GM-pomdp}
For all $h\in[H]$, $i\neq j\in[n]$ and $\bnu_1,\bnu_2\in\Delta_\fS$, we have 
$$\begin{cases}
\|\x_{hi} - \x_{hj}\|_2 \ge 4\sqrt{\log (d+1)}\times \sigma_h, \\
 \| \Wbb_h(\bnu_1-\bnu_2)\|_1 \ge \eta \|\bnu_1-\bnu_2\|_1.
\end{cases}$$
\end{condition}
Condition \ref{cond:GM-pomdp} requires that (a) different base Gaussian components are well separated, which is standard in learning Gaussian mixtures in classic theory of statistics, and (b) different latent state distributions induce different weights over the base Gaussian components, which resembles the one-step observable condition for tabular POMDPs. 
Importantly, in Lemma \ref{lem:gmpomdps-alpha} in Appendix \ref{app:gm-pomdps}, we show  that any  GM-POMDPs satisfying the $\eta$-separable condition are $\Omega(\eta)$-observable POMDPs. We remark that GM-POMDPs belongs to the infinite observation extension of tabular observable POMDP but not tabular weakly-revealing POMDPs, this is also the major reason we choose to present observable POMDPs in section \ref{subsec:weakly-revealing-pomdps}.

To apply \omle to learning $\eta$-separable  GM-POMDPs with $S$ states, $A$ actions and $n$  base Gaussian components in $\R^d$, we construct the model class $\Theta$ to include all the valid POMDP models wherein (a) the observation distributions are $\eta$-separable (Condition \ref{cond:GM-pomdp}) and (b) the norm of the mean and variance  of the base Gaussian components are well behaved. Formally,  define 
$$
\Theta:=\left\{ \left(\T,\Wbb,\{(\x_{h,i},\sigma_{h}\cdot \I_{d\times d})\}_{h,i},\bmu_1\right):~~\eta\text{-separable, }\|\x_{h,i}\|_2 \le C_x \text{ and }  \Csl \le \sigma_h \le \Csu\right\}.
$$
By carefully discretizing the parameter space and constructing the envelope functions,  we can derive the following upper bound for the bracketing number of model class $\Theta$ 
(Lemma \ref{lem:gmpomdps-bracket} in Appendix \ref{app:gm-pomdps}): 
\begin{equation}\label{eq:gm-pomdps-bracket}
    \cN_\Theta(\epsilon) 
    \le 
\exp\left(\Theta\bigg( 
H(S^2 A+ Sn + nd) \log(HSAnd(\Csu/\Csl)(C_x/\Csu)\cdot\epsilon^{-1})
\bigg)\right):=\overline{\cN}_\Theta(\epsilon). 
\end{equation}
Now that we know  $\eta$-separable GM-POMDPs are $\Omega(\eta)$-observable and have bounded bracketing number, we can invoke Theorem \ref{thm:pomdp}, which gives the following sample complexity guarantee for learning $\eta$-separable GM-POMDPs with \omle.
\begin{proposition}\label{prop:gauss-pomdp}
Suppose Condition \ref{cond:GM-pomdp} holds. 
There exists an absolute constant $c>0$ such that for any $\delta\in(0,1]$ and $K\in\N$,  if we choose $\beta = c\log(K\overline{\cN}_{\Theta}(K^{-1})\delta^{-1})$ in \omle (Algorithm \ref{alg:omle}) with $\overline{\cN}_\Theta$ specified by Equation \eqref{eq:gm-pomdps-bracket}, then with probability at least $1-\delta$, 
      $$
      V^\star - V^{\pi^{\rm out}} \le \poly(\eta^{-1},S,A,H,n,d,\log K,\log(\Csu/\Csl),\log(C_x/\Csu))\times K^{-1/2}~.
      $$\end{proposition}
Despite the observation space being infinitely large and unbounded, the above sample complexity only scales polynomially with respect to the dimension of the observation space and other relevant finite parameters. Finally,  we emphasize that although we only focus on POMDPs with Gaussian mixture observations in this subsection, our main result (Theorem \ref{thm:pomdp}) also applies to learning other types of continuous  observation distributions as long as the observable condition (Condition \ref{cond:pomdp}) holds and the model class has bounded bracketing number.

\subsection{Multi-step decodable POMDPs}

  Multi-step decodable POMDPs \cite{efroni2022provable} is subclass of POMDPs  in which a suffix of length $m$ of the most recent history contains sufficient  information to decode the latent state. To simplify notations, denote $m(h)=\min\{h-m+1,1\}$.  Formally, 

\begin{condition}\citep[$m$-step decodable POMDPs,][]{efroni2022provable}\label{cond:decodable-pomdp}
There exists an unknown decoder $\zeta=\{\zeta_h\}_{h=1}^H$ such that for every $(o,a)_{1:H}$ we have $s_h=\zeta(z_h)$ for all $h\in[H]$, where $z_h=[(o,a)_{m(h):h-1},o_h]$.
\end{condition}
We remark that neither of multi-step decodable POMDPs or multi-step observable POMDPs is more general than the other. That is, each of them contains statistically tractable POMDP instances that are not included by the other class (see Lemma \ref{lem:decodable-no-revealing} in Appendix \ref{appsub:m-decodable} for the concrete constructions). 
Nonetheless, the following theorem  states that multi-step decodable POMDPs also falls into the family of $\gamma$-\name PSRs with $\gamma=1$ and the sets of core test actions equal to $\fA^{m}$. As a result, \omle\ also enjoys polynomial sample efficiency guarantee for learning multi-step decodable POMDPs.

\begin{theorem}\label{thm:decodable-pomdp}
Let $\Theta$ be a model class of $m$-step decodable POMDPs. Then $\Theta$ admits  rank-$r$ PSR representations  with $\cQ_h^A=\fA^{\min\{m,H-h\}}$  and satisfies Condition \ref{asp:psr} with $\gamma=1$. Moreover, 
there exists an absolute constant $c>0$ such that for any $\delta\in(0,1]$ and $K\in\N$,  if we choose $\beta = c\log(T\cN_{\Theta}(T^{-1})\delta^{-1})$ with $T=KHA^m$ in \omle (Algorithm \ref{alg:omle}), then with probability at least $1-\delta$, 
      $$
      V^\star - V^{\pi^{\rm out}} \le \poly(r,A^m,H,\log K)\times \sqrt{\frac{\beta}{K}}~,
      $$
    where we always have $r\le S$ in any POMDP and $r\le \dlin$ when the underlying MDP can be represented as a $\dlin$-dimensional kernel linear MDP. 
\end{theorem}
Similar to the results in previous sections, the sample complexity in Theorem \ref{thm:decodable-pomdp} is independent of the number of observations, which means it also applies to the cases with infinite observations as long as the log-bracketing number of $\Theta$ is finite. 
Moreover, the above result  scales with the rank of the PSR representations $r$ instead of the number of latent states $S$. Although it is well-known $r\le S$ in any POMDPs, the rank can be much smaller than the number of latent states in certain settings of interest.  
For example, when the underlying MDP can be represented as a $\dlin$-dimensional  linear kernel MDP \citep{yang2020reinforcement}, we have $r\le \dlin$  while $S$ can be arbitrarily large.

\paragraph{Finite observations.} When the number of observations is finite, we can easily upper bound the bracketing number of $\Theta$ by the standard discretization arguments as in Equation \eqref{eq:pomdp-bracketing}.   And by plugging the bound back into Theorem \ref{thm:decodable-pomdp}, we immediately obtain a $\poly(S,A^m,O,H,\log \epsilon^{-1})\times \epsilon^{-2} $ sample complexity upper bound for finding an $\epsilon$-optimal policy with \omle\ in tabular $m$-step decodable POMDPs.

\subsection{POMDPs with a few known core action sequences}

In Section \ref{subsec:weakly-revealing-pomdps}, we prove that  if a POMDP satisfies that any two state mixtures can be distinguished from the observation distributions induced by taking  $m$-step random actions, then it can be represented as an \name PSR and \omle\ can learn it sample efficiently. 
 However, the sample complexity there scales exponentially with respect to $m$ due to $m$-step random exploration, which could be prohibitively large even for moderate $m$. 
 In this subsection, we show that it is possible to get rid of this exponential dependence when there exist a small set of \emph{known} exploratory action sequences so that any two state mixtures can be distinguished from the observation distributions induced by at least one exploratory action sequence.
 
 To simplify  notations, we first define the observation-action probability kernel $\K_h$ at step $h\in[H]$: 
For a latent state $s$ and an  action sequence $\a$ of length $l\le H-h$, 
$\K_{h}(s,\a)$ is equal to the probability density function over $o_{h:h+l}$ provided that action sequence $\a$ is used from state $s$ and step $h$.
 Formally, we consider the following observable-style condition.
 \begin{condition}\label{cond:pomdp-small-core-action-set}
 For any $h\in[H]$, there exists  \emph{known} $\cA_h$ so that for any $\theta\in\Theta$ and $\bnu_1,\bnu_2\in\Delta_S$: 
 \begin{equation} \max_{\a\in\cA_h}
   \big\|\E_{s\sim \bnu_1}\left[\K_{h}(s,\a)\right] - \E_{s\sim \bnu_2}\left[\K_{h}(s,\a)\right]\big\|_1 \ge \alpha \|\bnu_1-\bnu_2\|_1.
\end{equation}
 \end{condition}
 Notice that in Condition \ref{cond:pomdp-small-core-action-set}, the exploratory action sequences in $\cA_h$ can be length-$\Omega(H)$, which means a POMDP class $\Theta$ that satisfies Condition  \ref{cond:pomdp-small-core-action-set} could  satisfy the $m$-step observable condition only for $m=\Omega(H)$. 
 Nonetheless, the following theorem states that as long as $\Theta$ satisfies Condition \ref{cond:pomdp-small-core-action-set} with  $\cA_h$ of small cardinality, then \omle\ is guaranteed to learn a near-optimal policy for any $\theta\in\Theta$ within a number of samples that scales only polynomially with respect to $\max_h |\cA_h|$.

\begin{theorem}\label{thm:pomdp-small-core-action-set}
Let $\Theta$ be a model class of POMDPs that satisfy Condition \ref{cond:pomdp-small-core-action-set} with $\alpha$ and $\{\cA_h\}_{h=1}^H$. 
Then $\Theta$ satisfies Condition \ref{asp:psr} with $\gamma=\alpha/(S+|\cA_h|)$ and $\cQ_h^A=\cA_h$. Moreover, 
there exists an absolute constant $c>0$ such that for any $\delta\in(0,1]$ and $K\in\N$,  if we choose $\beta = c\log(T\cN_{\Theta}(T^{-1})\delta^{-1})$ with $T=KH\max_h|\cA_h|$ in \omle (Algorithm \ref{alg:omle}), then with probability at least $1-\delta$, 
      $$
      V^\star - V^{\pi^{\rm out}} \le \poly(\alpha^{-1},S,\max_h|\cA_h|,H,\log K)\times \sqrt{\frac{\beta}{K}}~.
      $$
\end{theorem}
When  the number of exploratory action sequences ($\max_h|\cA_h|$) is small  but their length ($\max_h\max_{\a\in\cA_h}|\a|$) is large, Theorem \ref{thm:pomdp-small-core-action-set} offers  exponentially sharper sample complexity guarantee than Theorem \ref{thm:pomdp}. 
As an extreme case, when each $\cA_h$ contains a single action sequence of length $H-h$, Theorem \ref{thm:pomdp-small-core-action-set} improves over Theorem \ref{thm:pomdp} by a factor of $A^{\Omega{(H)}}$. 


\section{Beyond Low-rank Sequential Decision Making}\label{sec:beyond-low-rank}

In this section, we extend the sample efficiency guarantees of OMLE to any sequential decision making problems under a new structural condition---SAIL condition. We will show that SAIL condition holds not only in all well-conditioned low-rank sequential decision making problems studied in Section \ref{sec:psr}, but also in problems beyond low-rank sequential decision making, such as factored MDPs, low witness rank problems.

\subsection{SAIL condition}

In the fully observable setting, RL with general function approximation has been intensively studied in the theory community, and various complexity measures have been proposed, including Bellman rank \cite{jiang2017contextual}, witness rank \cite{sun2019model}, and more \cite{jin2021bellman, du2021bilinear}. Most of them critical relies on the Bellman error (model-free setting) or the error in model estimation (model-based setting) to have a bilinear structure. Unfortunately, partially observability significantly complicates the learning problem, and neither  structure mentioned above hold for even the basic tabular weakly-revealing POMDPs.

Here, we introduce a new general structural condition that is also capable of addressing partially observable setting. Our new condition can be viewed as a generalizations of the bilinear structures mentioned above. Since our focus is OMLE which is a model-based algorithm, our new condition requires the model estimation error to be upper and lower bounded by \textbf{S}ummation of \textbf{A}bsolute values of \textbf{I}ndependent bi\textbf{L}inear functions (SAIL). Formally, let $\Pi$ denote the universal policy space.

\begin{condition}[\SAIL\ condition]\label{cond:salp}
We say  model class $\Theta$ satisfies $(d, \kappa, B)$-\SAIL\ condition with exploration policy function $\Pie:\Pi\rightarrow2^{\Pi}$, if  there exist two sets of mappings $\{f_{h,i}\}_{(h,i)\in[H]\times[m]}$,  $\{g_{h,i}\}_{(h,i)\in[H]\times[n]}$ from $\Theta$ to $\R^{d}$  such that  
 for any $\theta,\theta'\in\Theta$, and the optimal policy $\pi_\theta$ of model $\theta$: 
 \begin{align*}
    \textstyle  \sum_{\tilde{\pi}\in\Pie (\pi_\theta)} d_{\rm TV}(\P^{\tilde\pi}_{\theta^\star},\P^{\tilde\pi}_{\theta'} )
        \ge&  \textstyle\kappa^{-1}\sum_{h=1}^H \sum_{i=1}^m \sum_{j=1}^n |\langle f_{h,i}(\theta),g_{h,j}(\theta')\rangle|, \\
    \textstyle d_{\rm TV}(\P^{\pi_\theta}_{\theta^\star},\P^{\pi_\theta}_{\theta} )\le& \textstyle \sum_{h=1}^H \sum_{i=1}^m \sum_{j=1}^n |\langle f_{h,i}(\theta),g_{h,j}(\theta)\rangle|,\\
    \textstyle \left(\sum_{i=1}^{m} \| f_{h,i}(\theta)\|_1 \right)\cdot & \textstyle \left( \sum_{j=1}^{n} \| g_{h,j}(\theta')\|_\infty\right) \le B.
 \end{align*}

\end{condition}
The first inequality requires the model estimation error of $\theta'$ (measured by TV distance) on the exploration policies computed using $\theta$ to be lower bounded by a coefficient $\kappa^{-1}$ times SAIL. In particular, the summand $\langle f_{h,i}(\theta),g_{h,j}(\theta')\rangle$ is a bilinear function, because it is a linear function of $f_{h,i}(\theta)$ (features of $\theta$) when $\theta'$ is fixed, and it is also a linear function of $g_{h,j}(\theta')$ (features of $\theta'$) when $\theta$ is fixed. The second inequality requires the model estimation error of $\theta$ on its optimal policy $\pi_\theta$ to be upper bounded by SAIL. The third inequality is a normalization condition.

At a high-level, standard Bellman rank or witness rank can be viewed as conditions similar to SAIL, with the LHS of the first two inequalities replaced by appropriate error measure and the RHS of the first two inequalities replaced by a bilinear function $\langle f(\theta),g(\theta')\rangle$. SAIL condition generalize them by allowing  multiple feature functions $\{f_i\}_{i \in [m]}, \{g_j\}_{j \in [n]}$ which are indexed by $i,j$, and taking summation of them. One key structure here is that the indexes are decoupled between two features $f, g$, and summation is taken over two indexes \emph{independently}. This is crucial in many partially observable applications where $m, n$ are extremely large and we do not want to suffer any dependency on $m, n$ in the sample complexity.

We will prove in Section \ref{sec:example_sail} that SAIL condition is very general, which holds not only in all well-conditioned low-rank sequential decision making problems studied in Section \ref{sec:psr}, but also in problems beyond low-rank sequential decision making, such as factored MDPs, low witness rank problems.

\subsection{Theoretical guarantees for \SAIL}

 Now we present the theoretical guarantees for \omle in learning sequential decision problems that satisfy the  \SAIL\ condition. 

\begin{theorem}
\label{thm:salp}
There exists an absolute constant $c>0$ such that for any $\delta\in(0,1]$ and $K\in\N$,  if we choose $\beta = c\log(T\cN_{\Theta}(T^{-1})\delta^{-1})$ with $T=K|\Pie|$ in \omle (Algorithm \ref{alg:omle}) and assume $(d, \kappa, B)$-\SAIL\ condition holds,
then with probability at least $1-\delta$, we have
      $$
       \sum_{k=1}^{K} \left(V^\star -V^{\pi^k}\right)  \le   \poly(H) d \left( B  + \kappa \sqrt{\beta|\Pie| K} \right)\log^2(K).
      $$
\end{theorem}

The result in Theorem \ref{thm:salp} scales polynomially with respect to the parameters $(d, \kappa, B)$ and the number of exploration policies $|\Pie|$  in the \SAIL\ condition.  
Moreover, the result is completely independent of the number of the feature mappings $m$ and $n$, which is key in addressing the case of \name PSRs where the  \SAIL\ condition requires exponentially many feature mappings.  
When the log bracketing number has a reasonable growth rate $\log \cN_{\Theta}(T^{-1})\le {\rm polylog}(T)$, Theorem \ref{thm:salp} guarantees that $K=\tilde{\cO}(\kappa^2 d^2 \log\cN_\Theta(\epsilon^{-1}) \cdot \epsilon^{-2})$ episodes suffices for finding an $\epsilon$-optimal policy. The $\epsilon$-dependency is optimal up to polylogarithmic factors.

The critical step in proving Theorem \ref{thm:salp} is our new elliptical potential style lemma for \SAIL, which significantly generalizes the standard elliptical potential lemma that only applies to bilinear functions. Our new lemma immediately implies the following result.
\begin{lemma}\label{lem:sail-eluder}
$(d, \kappa, B)$-\SAIL\ condition implies the  generalized eluder-type condition (Condition \ref{cond:eluder}) with 
    $d_\Theta =  {  \kappa^2 d^2 |\Pie| \poly(H)}$ and $\xi(d_\Theta,\Delta,|\Pie|,K) = \tilde\cO\left( \sqrt{d_\Theta \Delta |\Pie| K} +d B    \poly(H)\right)$.
\end{lemma}

With this lemma, we can directly invoke the guarantee for \omle  (Theorem \ref{thm:omle}), which gives the bound in Theorem \ref{thm:salp}.

\paragraph{Sharper guarantee for single feature mapping.}
For sequential decision making problems that satisfy the \SAIL\ condition with a single pair of feature mappings $(f_h,g_h)$ for each $h\in[H]$, e.g., sparse linear bandits, factored MDPs, and linear MDPs, we can further  derive the following sharper sample complexity guarantee.
\begin{theorem}
\label{thm:salp-sharper}
Suppose  $(d, \kappa, B)$-\SAIL condition holds with  $m=n=1$. Then under the same choice of parameters as in Theorem \ref{thm:salp}, \omle satisfies that with probability at least $1-\delta$,
    $$
       \sum_{k=1}^{K} \left(V^\star -V^{\pi^k}\right)  = \tilde{\cO}\left(\poly(H)\left(d B  + \kappa \sqrt{d \beta |\Pie| K} \right) \right).
    $$ 
\end{theorem}
Theorem \ref{thm:salp-sharper} directly implies a regret bound with leading-order term $\tilde{\cO}(\kappa\sqrt{d \log\cN_\Theta(K^{-1}) K})$  when the exploration policy function $\Pie$ is equal to identity.
This improves a $\sqrt{d}$ factor over Theorem \ref{thm:salp}.
Theorem \ref{thm:salp-sharper} also implies a $\tilde{\cO}(\kappa^2 d \log\cN_\Theta(\epsilon^{-1}) \cdot \epsilon^{-2})$ sample complexity upper bound for finding an $\epsilon$-optimal policy when the log-bracketing number of $\Theta$ grows polylogarithmically with respect to the covering precision, which improves a $d$ factor over the sample complexity implied by Theorem \ref{thm:salp}.

\subsection{Important examples of \SAIL} \label{sec:example_sail}

In this section, we present several widely studied sequential decision making problems that satisfy the \SAIL\ condition. 
We remark that all problems considered in this section are MDPs so we will use $\{s_h\}_{h=1}^H$ to denote states.

\subsubsection{Low-rank sequential decision making}

 To demonstrate the generality of the \SAIL\  condition, we prove 
the following proposition which states that (a) any \name PSR satisfies the \SAIL\  condition with moderate $(d, \kappa, B)$,  and (b) there exist sequential decision making problems,  whose system dynamics matrices have  exponentially large rank though, which still satisfy the \SAIL\  condition with mild $(d, \kappa, B)$.
 
\begin{proposition}[well-conditioned PSR $\subseteq$ \SAIL $\not\subseteq$ low-rank sequential decision making]\label{prop:salp-relation} ~
\begin{enumerate}[label=(\alph*)]
    \item Any rank-$r$ $\gamma$-well-conditioned PSR  class $\Theta$ satisfies  the \SAIL\  condition  with $d=r$  and $\kappa,B=\poly(r,\gamma^{-1}, \max_h|\cQ^A_h|,A,H)$ and the same choice of   $\Pie$ as in Theorem \ref{thm:psr}.
    \item For any $n\in\N$, there exists  $\Theta$  satisfying  the \SAIL\  condition  with $d,\kappa,B=\cO(n)$ and $\Pie(\pi)=\pi$, but for some $\theta\in\Theta$ the system dynamics matrices have rank $\Omega(2^{n})$.
\end{enumerate}
\end{proposition}

\subsubsection{Fully observable problems with low witness rank}

\paragraph{Witness rank.} Witness rank \citep{sun2019model} was introduced as a structural parameter for measuring the difficulty of model-based RL.  
\cite{sun2019model} proved that  the witness rank of a model class being small suffices to guarantee sample-efficient learning, and several RL settings of interest (e.g., factored MDPs) possess rather moderate witness rank. 
To simplify notations, let $\D_\theta(s_h,a_h):=\P_\theta((r_h,s_{h+1})=\cdot  \mid s_h,a_h)$. And the witness rank is defined as following: 
\begin{definition}[Q/V-type witness conditions (slightly modified version\protect\footnotemark ~of \cite{sun2019model})] \label{def:witness}
We say  model class $\Theta$ satisfies $(d,\kappa,B)$-witness condition, if  there exist two sets of mappings $\{f_{h}\}_{h\in[H]}$,  $\{g_{h}\}_{h\in[H]}$ from $\Theta$ to $\R^d$, so that  
 for any $\theta,\theta'\in\Theta$ and $h\in[H]$: 
    $$
    \begin{cases}
     \E_{s_h\sim \P^{\pi_\theta}_{\theta^\star},~a_h\sim \nu(s_h) }\left[\|\D_{\theta'}(s_h,a_h) - \D_{\theta^\star}( s_h,a_h)     \|_1 \right] \ge \kappa^{-1} \left|\langle f_h(\theta),g_h(\theta')\rangle\right|,\\
      \E_{s_h\sim \P^{\pi_\theta}_{\theta^\star},~a_h\sim \nu(s_h) }\left[\|\D_{\theta'}(s_h,a_h) - \D_{\theta^\star}( s_h,a_h)     \|_1 \right] \le \left|\langle f_h(\theta),g_h(\theta')\rangle\right|,\\
      \|f_h(\theta)\|_1 \times  \|g_h(\theta')\|_\infty \le B,
    \end{cases}
    $$
    where $\nu$ is typically chosen as $\pi_\theta$ (Q-type) or $\pi_{\theta'}$ (V-type).
\end{definition}

\footnotetext{The witness condition presented here is slightly different from the original version in \cite{sun2019model}. The original definition replaces the model-based model discrepancy in the second inequality by the value-based Bellman error. This difference is rather minor because (1) all the low witness rank examples considered here and in \cite{sun2019model} satisfy Definition \ref{def:witness}; (2) OMLE can directly handle low witness rank problems in its original definition---in analysis, we only need to slightly generalize Condition \ref{cond:eluder} and the corresponding proofs.}

The Q-type witness condition  requires that at each single step the  expected model discrepancy between  the true model $\theta^\star$ and model candidate $\theta'$  under the state-action distribution induced by the optimal policy of $\theta$  is roughly proportional to the inner product of the features of $\theta$ and $\theta'$.
And the V-type version is defined similarly except that the last action $a_h$ is sampled from $\pi_{\theta'}$ instead of $\pi_{\theta}$.
By basic algebra, we can easily relate the above per-step model discrepancy in witness condition to the whole-trajectory model discrepancy in \SAIL\ condition, which leads to the following conclusion  that the \SAIL\ condition is satisfied with almost the same $(d,\kappa,B)$ whenever either Q-type or V-type witness condition holds.
\begin{proposition}\label{prop:witness-salp}
For any model class $\Theta$, we always have
\begin{itemize}
    \item Q-type $(d,\kappa,B)$-witness  condition implies 
    $(d,2\kappa,B)$-\SAIL\ condition with  $\Pie(\pi)=\{\pi\}$, and $m=n=1$.
    \item V-type $(d,\kappa,B)$-witness  condition implies 
    $(d,2A\kappa,B)$-\SAIL\ condition with  $\Pie(\pi)=\{\pi_{1:h}\circ{\rm Uniform}(\fA):~h\in[0,H-1]\}$, and $m=n=1$.
\end{itemize}
\end{proposition}

In the case when $\log\cN_\Theta$ grows polylogarithmically with respect to the covering precision,  by plugging Proposition \ref{prop:witness-salp} back into Theorem \ref{thm:salp-sharper}, we immediately obtain a $\tilde{\cO}(A^2\kappa^2 d\log\cN_\Theta(\epsilon^{-1}) \epsilon^{-2})$ sample complexity upper bound for \omle\ in the V-type witness rank setting, which improves over the quadratic dependence on $d$ in \cite{sun2019model}. Moreover, \omle\  further enjoys a $\tilde{\cO}(\kappa\sqrt{d \log\cN_\Theta(K^{-1}) K})$ regret guarantee in the Q-type witness rank setting, which is new to our knowledge.

\paragraph{Factored MDPs.}

In factored MDPs, the state admits a factored structure. Concretely, each state $s$ consists of $m$ factors denoted as $(s[1],\ldots,s[m])\in\cX^m$.
Moreover, each factor $i\in[m]$ has a parent set denoted by  $\pa_i\subseteq[m]$, with respect to which the transition admits the following factorized form:
\begin{equation}
    \P(s_{h+1} \mid s_h,a_h) = \prod_{i=1}^m \P^i(s_{h+1}[i] \mid s_h[\pa_i],a_h).
\end{equation}
In other words, the transition of the $i^{\rm th}$ factor of states are only  determined by a subset of all factors, that is $\pa_i$, instead of the whole state. 
 In factored MDPs, it is standard to assume the factorization structure and the reward function are \emph{known} \citep{kearns1999efficient,sun2019model}.
 Therefore, our model class $\Theta$ only needs to parameterize  the transitions under the given  factorization structure. 
 
The following proposition states that when the factorization structure is known, factored MDPs admit low witness rank structure. 
\begin{proposition}\label{prop:factored_MDPs} Let $\Theta$ consist of all the factored MDPs with the same  factorization structure $\{\pa_i\}_{i=1}^m$. Then $\Theta$ satisfies Q-type  witness condition with 
    $d = A \sum_{i=1}^m |\cX|^{|\pa_i|}$,  $\kappa=m$, and $B=\sum_{i=1}^m |\cX|^{|\pa_i|}$.
\end{proposition}

\paragraph{Kernel linear MDPs.} In kernel linear MDPs \citep{yang2020reinforcement}, the transition functions can be represented as a linear functions of the tensor product of two \emph{known} feature mappings. 
Formally, the learner is provided with  features  $\phi:\fS\times\fA\rightarrow \R^\dlin$ and $\psi:\fS\rightarrow \R^\dlin$ so that for any $h\in[H]$, there exists $\W_h\in \R^{\dlin \times \dlin}$ satisfying $\P_h(s_{h+1}\mid s_h,a_h)= \phi(s_h,a_h)\trans \W_h \psi(s_{h+1})$ for all $(s_h,a_h,s_{h+1})\in\fS\times\fA\times\fS$. Besides, kernel linear MDPs satisfy the normalization condition: 
(a) $\|\phi(s_h,a_h)\|_2 \le C_\phi$ for all $(s_h,a_h)$, 
(b) $\|\sum_{s_{h+1}} \psi(s_{h+1}) f(s_{h+1}) \|_1 \le C_\psi$ for all $\|f\|_\infty \le 1$, and (c) $\|\W_h\|_{2} \le C_W $.

For simplicity, we assume the reward function is known. 
Previous works \citep[e.g.,][]{yang2020reinforcement} have shown that  kernel linear MDPs are capable of  representing various examples with moderate dimension $\dlin$, e.g., tabular MDPs with $\dlin=SA$. The following proposition states that kernel linear MDPs also fall into the low witness rank framework with the same ambient dimension.

\begin{proposition}\label{prop:linearMDPs} Let $\Theta$ be the family of $\dlin$-dimensional  kernel linear  MDPs. Then $\Theta$ satisfies V-type witness condition with $d = \dlin$,  $\kappa=1$, and  $B=2(\sqrt{\dlin}C_\phi C_W C_\psi +1)$. 
\end{proposition}

\paragraph{Sparse linear bandits.} 
In sparse linear bandits, the mean reward function can be represented as a sparse linear function of the arm feature. Formally, we have $R_\theta(a)= \langle a,  \theta\rangle$ where (i)  $a\in\fA\subseteq B^\dlin_{C_\fA}(0)$, (ii) $\Theta:=\{\theta\in B^\dlin_{C_\Theta}(0):~\|\theta\|_0\le m  \text{ and } \langle \theta,a\rangle \in[0,1] \text{ for any } a\in\fA\}$. Without loss of generality, assume the stochastic reward feedback is  binary\footnote{If the reward feedback $\hat{r}$ is a real number in $[0,1]$, we can binarize it by sampling $x$ from ${\rm Bernoulli}(\hat r)$ and then using $x$ as the reward feedback instead. Such modification will not change the mean reward.}.
The following  proposition states that the witness rank of sparse linear bandits is no larger than the   ambient dimension $\dlin$.
\begin{proposition}\label{prop:linear-bandits-witness} Let $\Theta$ be the family of $\dlin$-dimensional  $m$-sparse linear bandit. Then $\Theta$ satisfies  Q-type witness condition with   $d = \dlin$,  $\kappa=1$, and  $B=4\sqrt{\dlin} C_\Theta C_\fA$. 
\end{proposition}
By combining Proposition \ref{prop:linear-bandits-witness} with Proposition \ref{prop:witness-salp} and \ref{thm:salp-sharper}, we recover the optimal regret for sparse linear bandits $\tilde{\cO}(\sqrt{m\dlin K})$ up to a polylogarithmic factor.

\section{Reward-free Learning with OMLE}


In the previous sections, we show OMLE can learn near-optimal policies under the SAIL condition provided with reward information. 
In this section, we further extend the guarantees of OMLE by showing that OMLE with simple modification is capable to learn the dynamic models without any reward guidance under a slightly stronger version of the SAIL condition.
A direct implication is that we can use this new variant of algorithm to perform reward-free learning \cite{jin2020reward}, 
i.e., to learn the near-optimal policies for all reward functions simultaneously.



\subsection{Algorithm}
We present the modified algorithm --- Reward-free OMLE in Algorithm \ref{alg:omle-rf}.
Compared to OMLE (Algorithm \ref{alg:omle}), Reward-free OMLE  essentially only changes the optimistic planning  step (Line \ref{line:uncertainty}). Specifically, OMLE follows the principle of optimism in face of  uncertainty to balance exploration and exploitation while Reward-free OMLE  purely focuses on exploration and follows the policy with the highest uncertainty. Formally, for each policy $\pi$, we examine the discrepancy $d_{\rm TV}(\P_{\theta}^\pi,\P_{\tilde\theta}^\pi)$ between the dynamics of every pair of models $\theta$ and $\tilde{\theta}$ in the current model confidence  set $\cB^k$ under policy $\pi$.
In each iteration $k$, Reward-free OMLE simply picks $\pi^k$  to be the policy that achieves the highest discrepancy $\max_{{\theta},\tilde{\theta}\in\cB^k} d_{\rm TV}(\P_{\theta}^\pi,\P_{\tilde\theta}^\pi)$ within the current confidence set $\cB^k$.
By deliberately chasing the uncertainty, we can effectively explore the regimes where the  previously collected data are insufficient for accurately estimating the model dynamics.
Finally,  the output model estimate of Reward-free OMLE is simply  an arbitrary model that remains in the confidence after $K$ iterations.



\begin{algorithm}[t]
    \caption{\textsc{Reward-Free OMLE} $(\Theta,\beta)$}
 \begin{algorithmic}[1]\label{alg:omle-rf}
 \STATE \textbf{initialize:} $\cB^1 = \Theta$, $\mathfrak{D}=\{\}$ 
    \FOR{$k=1,\ldots,K$}
    \STATE compute $\pi^k \leftarrow \argmax_\pi \max_{{\theta},\tilde{\theta}\in\cB^k} d_{\rm TV}(\P_{\theta}^\pi,\P_{\tilde\theta}^\pi)$ \label{line:uncertainty}
    \STATE compute exploration policies $\Pi_{\exp}^k \leftarrow \Pi_{\exp}(\pi^k)$
    \FOR{each $\pi \in \Pi_{\exp}^k$}
    \STATE execute policy $\pi$ and collect a trajectory $\tau = (o_1, a_1, \ldots, o_H, a_H)$
    \STATE add $(\pi, \tau)$ into dataset $\fD$
    \ENDFOR
    \STATE update confidence set
    \vspace{-3mm}
    \begin{equation*}
    \cB^{k+1} = \bigg\{\hat\theta \in \Theta: \sum_{(\pi,\tau)\in\fD} \log \P_{{\hat\theta}}^{\pi} (\tau)
    \ge \max_{ \theta' \in\Theta} \sum_{(\pi,\tau)\in\fD} \log \P^{\pi}_{{\theta'}}(\tau) -\beta  \bigg\}\bigcap \cB^k
    \vspace{-3mm}
    \end{equation*}
    \ENDFOR
    \STATE \textbf{output} an arbitrary model  $\theta^{\rm out}\in\cB^K$
 \end{algorithmic}
 \end{algorithm}

\subsection{Theoretical guarantees}

We first present a slightly stronger version of \SAIL~condition (Condition \ref{cond:salp}) which enables Algorithm \ref{alg:omle-rf} to have new reward-free guarantees.






\begin{condition}[Strong \SAIL\ condition]\label{cond:salp-rf}
We say  model class $\Theta$ satisfies strong $(d, \kappa, B)$-\SAIL\ condition with exploration policy function $\Pie:\Pi\rightarrow2^{\Pi}$, if  there exist two sets of mappings $\{f_{h,i}\}_{(h,i)\in[H]\times[m]}$ from $\Pi$ to $\R^{d}$ and  $\{g_{h,i}\}_{(h,i)\in[H]\times[n]}$ from $\Theta$ to $\R^{d}$ such that  
 for any $(\pi,\theta)\in\Pi\times\Theta$: 
 \begin{align*}
    \textstyle  \sum_{\tilde{\pi}\in\Pie (\pi)} d_{\rm TV}(\P^{\tilde\pi}_{\theta^\star},\P^{\tilde\pi}_{\theta} )
        \ge&  \textstyle\kappa^{-1}\sum_{h=1}^H \sum_{i=1}^m \sum_{j=1}^n |\langle f_{h,i}(\pi),g_{h,j}(\theta)\rangle|, \\
    \textstyle d_{\rm TV}(\P^{\pi}_{\theta^\star},\P^{\pi}_{\theta} )\le& \textstyle \sum_{h=1}^H \sum_{i=1}^m \sum_{j=1}^n |\langle f_{h,i}(\pi),g_{h,j}(\theta)\rangle|,\\
    \textstyle \left(\sum_{i=1}^{m} \| f_{h,i}(\pi)\|_1 \right)\cdot & \textstyle \left( \sum_{j=1}^{n} \| g_{h,j}(\theta)\|_\infty\right) \le B.
 \end{align*}
\end{condition}

We see that Condition \ref{cond:salp-rf} differs from the original one (Condition \ref{cond:salp}) only in the LHS of the two inequalities, which replace greedy policy $\pi_\theta$  by any arbitrary policy $\pi$. We require this stronger version of condition because reward-free learning aims to learn an model estimate that is sufficiently accurate to perform planning for any later specified reward function. This objective is more demanding than the standard objective of finding near-optimal policy for a single reward function.
Nonetheless, 
this stronger version of SAIL condition is still satisfied by all examples in Section \ref{sec:beyond-low-rank}. Please refer to Appendix \ref{app:reward-free} for the proofs. 


Now, we present the theoretical guarantees for Reward-free OMLE under the strong SAIL condition.

\begin{theorem}
\label{thm:salp-rf}
There exists an absolute constant $c>0$ such that for any $\delta\in(0,1]$ and $K\in\N$,  if we choose $\beta = c\log(T\cN_{\Theta}(T^{-1})\delta^{-1})$ with $T=K|\Pie|$ in Reward-free \omle (Algorithm \ref{alg:omle-rf}) and assume strong $(d, \kappa, B)$- \SAIL\ condition holds,
then with probability at least $1-\delta$, we have
      $$
       \max_{\pi} d_{\rm TV}(\P_{\theta^{\rm out}}^\pi,\P_{\theta^\star}^\pi) \le    \poly(H) d \left( \frac{B}{K}  + \kappa \sqrt{\frac{\beta|\Pie|}{ K}} \right) \log^2(K).
      $$
\end{theorem}
Similar to  Theorem \ref{thm:salp}, the  above model estimation error  scales polynomially with respect to the parameters $(d, \kappa, B)$ and the number of exploration policies $|\Pie|$  in the strong \SAIL\ condition, and are independent of the number of the feature mappings $m$ and $n$. The latter characteristic is of vital importance for handling well-conditioned PSRs that  satisfy Condition \ref{cond:salp-rf} with exponentially large $m$ and $n$. 
Moreover, Theorem \ref{thm:salp-rf} implies that $K=\tilde{\cO}(\kappa^2 d^2 \log\cN_\Theta(\epsilon^{-1}) \cdot \epsilon^{-2})$ episodes suffices for pinning down the model dynamics for all polices \emph{simultaneously} up to precision $\epsilon$ when the log-bracketing number $\log\cN_\theta(T^{-1})$ grows no faster than $\polylog (T)$.
We remark that this error measurement in Theorem \ref{thm:salp-rf} is much stronger than  the one in Theorem \ref{thm:salp} which only guarantees the suboptimality in value for a particular reward. In particular,  
the error bound in  Theorem \ref{thm:salp-rf} implies that we can  learn the near-optimal policies for all reward functions simultaneously: because for any reward function $R:~(\fO\times\fA)^H\rightarrow [0,1]$, any optimal policy under the $\epsilon$-accurate model $\theta^{\rm out}$ is also $\epsilon$-optimal under the ground truth model $\theta^\star$.


\section{Conclusion}

In this paper, we study a simple algorithm OMLE for generic model-based sequential decision making, and  prove that OMLE can learn a huge family of both fully observable and partially observable  sequential decision making problems sample-efficiently.
Such problems include but not limited to  tabular observable POMDPs, multi-step decodable POMDPs, low witness rank problems, observable POMDPs with continuous observations, and  well-conditioned PSRs, where the last two are new  and have not been solved by any previous works. 
We further identify a new structural condition \SAIL, which unifies our existing understandings of model model-based RL in both fully observable and partially observable  sequential decision making. We prove that the \SAIL\ condition is satisfied by all sequential decision making problems studied in this paper and OMLE always have   polynomial sample complexity guarantee under the \SAIL\  condition.
Finally, we propose a reward-free variant of OMLE, which is capable of providing a sufficient accurate estimation of the model dynamics such that we can use it to find the near-optimal policies of all reward functions simultaneously.


\section*{Acknowledgement}

We would like to thank  Ahmed Khaled for helpful discussions.


\bibliographystyle{plainnat}
\bibliography{ref}

\newpage
\appendix

\newcommand{\bQ}{{\bar{\cQ}}}
\newcommand{\bq}{{\bar{q}}}

\section{PSRs with continuous  observations}
\label{sec:psr_cts}

In this section, we show how to extend the definitions and conditions for PSRs to handle continuous observations.  

\subsection{Definition of continuous PSRs}

Recall in the finite-observation setting (Section \ref{sec:psr}), we define a PSR representation by using  a set of finite-dimensional  operators $(\bphi, \bM, \bpsi_0)$, which jointly satisfy the  following two equations:
\begin{align}
    \P(o_{1:H}|a_{1:H}) =& \bphi_H(o_H, a_H)\trans \bM_{H-1}(o_{H-1}, a_{H-1}) \cdots \bM_1(o_1, a_1)\bpsi_0,   \label{eq:PSR_3}\\
    \bpsi(o_{1:h}, a_{1:h}) =& \bM_{h}(o_{h}, a_{h}) \cdots \bM_1(o_1, a_1) \bpsi_0.  \label{eq:PSR_4}
\end{align}
However, in the continuous observation setting, the above probabilistic  quantities on the LHS could always be zero for all $(o_{1:H},a_{1:H})$,  since the probability of  a single point is usually zero for general distributions, e.g., Gaussian distribution. 
For Equation \eqref{eq:PSR_3}, it is natural to address this issue by replacing discrete probability with  density. Formally, let $f(o_{1:H}|a_{1:H})$ denote the density at observation sequence $o_{1:H}$ provided that action sequence $a_{1:H}$ is taken. We require the operators to satisfy 
\begin{equation}
f(o_{1:H}|a_{1:H}) = \bphi_H(o_H, a_H)\trans \bM_{H-1}(o_{H-1}, a_{H-1}) \cdots \bM_1(o_1, a_1)\bpsi_0.
\end{equation}
And for Equation \eqref{eq:PSR_4}, we introduce the following two approaches to generalize it.

\paragraph{General core tests as sets of trajectories.} 
Although each single point has probability zero in a general distribution, the probability of sets of points could still be strictly positive. 
Therefore, it is natural to generalize the concept of core tests by allowing them to be sets of infinite/dense trajectories. 
Formally, we consider  $\cQ_h=\{q_{h,1},\ldots,q_{h,d}\}$ where each $q_{h,i}:=(\cY_{h+1},a_{h+1},\ldots,\cY_{h+m},a_{h+m})\in (2^{\fO}\times\fA)^m$
for some $m\le H-h$, where $2^\fO$ denotes the collections of all measurable subsets of $\fO$. 
WLOG, we assume all the  core tests are disjoint and each core action sequence is not a prefix of another one.\footnote{This is WLOG because our sample complexity bounds  only depend on the total number of  different core action sequences.}
Now we can  generalize Equation \eqref{eq:PSR_4} by requiring that for any $q_h=(\cY_{h+1},a_{h+1},\ldots,\cY_{h+m},a_{h+m})\in\cQ_h$:
\begin{equation}\label{eq:PSR_5}
\begin{aligned}
    \MoveEqLeft \e_{q_h}\trans\left[\bM_{h}(o_{h}, a_{h}) \cdots \bM_1(o_1, a_1) \bpsi_0\right] \\
    = &f(o_{1:h} \mid a_{1:h})\times 
\P(o_{h+1:h+m}\in \cY_{h+1}\times \cdots\times \cY_{h+m} \mid o_{1:h},a_{1:h+m}).
    \end{aligned}
\end{equation}
By adopting this generalization, all the operators remain   finite-dimensional matrices or vectors. 

\paragraph{General operators as  linear functional.} 
Another  natural alternative is to simply replace discrete  probability with density in Equation \eqref{eq:PSR_4}. Formally, for each $(o_{1:h},a_{1:h})$ and  $q_h=(o_{h+1},a_{h+1},\ldots,o_{h+m},a_{h+m})\in\cQ_h$,  we require  $\bM_{h}(o_{h}, a_{h}) \cdots \bM_1(o_1, a_1) \bpsi_0$ to be a  density function defined on $\cQ_h$ so that 
\begin{equation}\label{eq:PSR_6}
\begin{aligned}
     [\bM_{h}(o_{h}, a_{h}) \cdots \bM_1(o_1, a_1) \bpsi_0](q_h) 
    = f(o_{1:h+m} \mid a_{1:h+m}).
    \end{aligned}
\end{equation}
If we adopt this approach, operator $\bM_h(o_h,a_h)$ is now a  functional that maps a function defined over $\cQ_{h-1}$ to a function defined on $\cQ_h$.

\paragraph{A unified generalization.}
Now we can combine the above two approaches by allowing $\cQ_h$'s for different $h$'s to be generalized in different ways. 
In other words, given $h\in[H]$, if $\cQ_h$ consists of general core tests, we replace Equation \eqref{eq:PSR_4} with 
Equation \eqref{eq:PSR_5},  and if $\cQ_h$ consists of single trajectories, we require Equation  \eqref{eq:PSR_6} instead. 

\paragraph{Rank of continuous PSR.} In the continuous setting, the system-dynamic  matrices are no longer well-defined, so we define the rank of PSR to be 
\begin{equation*}
    r:=\max_{h\in[H-1]} {\rm dim}\left( \left\{\bpsi(\tau_h):~\tau_h\in(\fO\times\fA)^h \right\}\right),
\end{equation*}
which is equivalent to the definition of rank in Section \ref{sec:psr} for  the finite-observation case. 

\subsection{Well-conditioned continuous PSRs}

In this subsection,  we show how to generalize Condition \ref{asp:psr} into   the continuous setting.
With slight abuse of notation, let $\R^{|\cQ_h|}$ denote the $|\cQ_h|$-dimensional Euclidean space if  $\cQ_h$ is finite,  and let it denote the collections of all real-valued functions defined on $\cQ_h$ if $\cQ_h$ is infinite.

\paragraph{Linear weight functional.}
First, we need to generalize the definition of weight ``vectors'' $\m_1,\m_2$. Given any $h\in[H-1]$ and $\omega_h\in\Omega^{(1)}_h:=(\fO\times\fA)^{H-h}$, $\m_1(\omega_h)$ defined in Equation \eqref{eq:weight_vector_1} is now a linear functional from $\R^{|\cQ_h|}$ to $\R$. And the continuous analogy of Condition \ref{asp:psr} for $\m_1$ is simply replacing summation with integral: 
\begin{equation}
\label{eq:Sep28-1}
    \max_{\substack{~~~~\x \in \R^{|\cQ_h|} \\ \norm{\x}_1 \le 1}
} ~ \int_{\omega_{h}\in{\Omega^{(1)}_h}} \pi( \omega_h) \cdot |\m_1(\omega_h)\trans\x| ~d\omega_h \le \frac{1}{\gamma}~.
\end{equation}
As for $\m_2$,  we can still use Equation \eqref{eq:weight_vector_2} to 
define  $\m_2$ so that each $\m_2(\omega_h)$ with $\omega_h\in \Omega^{(2)}_h:=\fO \times \fA \times \cQ_{h+1}$ is a linear functional  from $\R^{|\cQ_h|}$ to $\R$. Furthermore, if $\cQ_{h+1}$ is made  up of core tests that are trajectories, then the natural  continuous analogy of Condition \ref{asp:psr} for $\m_2$ is also  simply replacing summation with integral: 
\begin{equation}
\label{eq:Sep28-2}
    \max_{\substack{~~~~\x \in \R^{|\cQ_h|} \\ \norm{\x}_1 \le 1}
} ~ \int_{\omega_{h}\in{\Omega^{(2)}_h}} \pi( \omega_h) \cdot |\m_2(\omega_h)\trans\x|  ~d\omega_h\le \frac{1}{\gamma}~.
\end{equation}
However, if we are using general core tests at step $h+1$ (that is, $\cQ_{h+1}$ consists of measurable sets of trajectories),  then
$\pi(\omega_h)$ with $\omega_h\in \Omega^{(2)}_h:=\fO \times \fA \times \cQ_{h+1}$
is no longer well defined. Instead, we adopt the following normalization-style  condition for $\m_2$: let $\omega_h=(o_{h+1},a_{h+1},\omega_{h+1})$
\begin{equation}
\label{eq:Sep28-3}
 \max_{\substack{~~~~\x \in \R^{|\cQ_h|} \\ \norm{\x}_1 \le 1}
} ~ \int_{\omega_{h}\in{\Omega^{(2)}_h}} \pi( o_{h+1},a_{h+1}) \cdot |\m_2(\omega_h)\trans\x|  ~d\omega_h \le \frac{|\Qa_{h+1}|}{\gamma}~.
\end{equation}
wherein,  compared to Equation \eqref{eq:Sep28-2},  we remove the reweighting effect of $\pi$ on the core action sequences in  $\cQ_{h+1}$ from  the LHS and then compensate it  by multiplying the RHS with the number of core action sequences in $\cQ_{h+1}$, i.e., $|\Qa_{h+1}|$.

\begin{condition}[$\gamma$-\name  continuous PSR]\label{asp:psr-continuous}
We say a continuous  PSR is \emph{$\gamma$-\name} if for any $h\in[H-1]$ and any policy $\pi$ independent of the history before step $h+1$, the weight functional $\m_1(\cdot), \m_2(\cdot)$ and the corresponding future sets $\Omega^{(1)}_h, \Omega^{(2)}_h$ in \eqref{eq:weight_vector_1} \eqref{eq:weight_vector_2} satisfy (a) Equation \eqref{eq:Sep28-1} and \eqref{eq:Sep28-2} if $\cQ_{h+1}$ contains standard   core tests that are trajectories, or (b) Equation \eqref{eq:Sep28-1} and \eqref{eq:Sep28-3} if $\cQ_{h+1}$ consists of general core tests that are sets of trajectories. 
\end{condition}

\subsection{Proof for well-conditioned continuous PSRs}\label{app:psr-main-continuous}

The proof of Theorem \ref{thm:psr} for  continuous PSRs follows basically the same steps as in the finite-observation case (Appendix \ref{app:psr-main}). 
To avoid repeating noninformative arguments, here we only describe the key modifications needed for  extending the proofs in Appendix \ref{app:psr-main} to the continuous setting. 
\begin{itemize}
    \item (STEP 1. TV-distance $\le$ operator-difference) easily  goes through by replacing Condition \ref{asp:psr} with \ref{asp:psr-continuous} and the fact that all operators are linear functionals.
    \item (STEP 2. operator-difference $\le$ TV-distance) also holds by using the linearity of operators and the fact that 
    $$
    \int_{\tau_h} \left\| \widehat\bpsi(\tau_{h})-\bpsi(\tau_{h})\right\|_1 \times \pi(\tau_{h-1}) d\tau_h \le 2A\sum_{\a\in\cQ_h^A}  d_{\rm TV} (\P^{\nu(\pi,h,\a)}_{{\widehat\theta}}, \P^{\nu(\pi,h,\a)}_{\theta^\star})
    $$
    holds for any $\pi$ and any type of disjoint core tests.
    \item (STEP 3. $\ell_1$-norm pigeon-hole argument) requires an integral-version of Proposition \ref{prop:step-3} which follow immediately from replacing 
    Proposition \ref{prop:pigeon-hole} with Corollary \ref{cor:pigeon-hole} in the proof of Proposition \ref{prop:step-3}.
     \item (STEP 4. construct SAIL feature mappings) also holds by replacing $\A_{h+1}$ with the Barycentric spanner of 
     $$\text{closure}\left( \left\{\frac{\bpsi(\tau_h)}{\|\bpsi(\tau_h)\|_1}:~\tau_h\in(\fO\times\fA)^h \text{ and } \|\bpsi(\tau_h)\|_1\neq 0 \right\}\right).$$
\end{itemize}


\section{Proofs for Optimistic MLE }

\subsection{Proof of Theorem \ref{thm:omle}}

For the reader's convenience, we restate Theorem \ref{thm:omle} below. 

\thmomle*

The proof consists of three main steps:
\begin{enumerate}
    \item First, we show with high probability, the true model $\theta^\star$ is contained in the model confidence set $\cB^k$ for all $k\in[K]$. 
    \item Second, we prove the optimistic model estimate $\theta^k$ has low average prediction error under historical policies $\pi^1,\ldots,\pi^{k-1}$ with high probability. 
    \item Finally, we conclude the proof by combining the above two results with the generalized eluder-type condition.
\end{enumerate}

\paragraph{Step 1.} 
By Proposition \ref{prop:mle-optimisim}, the definition of $\cB^k$ and the choice of $\beta$, we have with probability at least $1-\delta$, $\theta^\star\in\cB^k$ for all $k\in[K]$. 
Combining it with the fact that $\theta^k$ and $\pi^k$ are picked optimistically, we have 
\begin{equation}\label{eq:opt}
\begin{aligned}
    \MoveEqLeft \text{Regret}(k) = \sum_{t=1}^k \left(V^{\star}_{\theta^\star} - V^{\pi^t}_{\theta^\star} \right)\\
    & \le 
    \sum_{t=1}^k \left(\max_{\theta\in\cB^t} \max_{\pi} V^{\pi}_{\theta} - V^{\pi^t}_{\theta^\star}\right)\\
	&= \sum_{t=1}^k \left( V^{\pi^t}_{\theta^t} - V^{\pi^t}_{\theta^\star}\right) \le H\sum_{t=1}^k d_{\rm TV}(\P^{\pi^t}_{\theta^\star},\P^{\pi^t}_{\theta^t}).
\end{aligned}
   \end{equation}

\paragraph{Step 2.} 
To simplify notations, we denote by $\tau^{i,\pi}$ the trajectory we collected in the $i^{\rm th}$ iteration of \omle by executing policy $\pi\in\Pi_{\rm exp}^i$. 
By Proposition \ref{prop:mle-valid}, the definition of $\cB^k$ and the choice of $\beta$, we have with probability at least $1-\delta$, 
\begin{equation}\label{eq:valid}
    \begin{aligned}
        \MoveEqLeft  \sum_{i=1}^{t-1}\sum_{\pi \in \Pi^i_{\exp}} d_\TV^2(\P_{\theta^t}^{\pi}, \P_{\theta^\star}^{\pi})  \\
        \le  &  c \left( \sum_{i=1}^{t-1} \sum_{\pi \in \Pi^i_{\exp}} \log\left[\frac{\P^{\pi}_{{\theta}^\star}(\tau ^{i,\pi})}{\P^{\pi}_{{\theta^t}}(\tau^{i,\pi})}\right] +\log(T\cN_{\Theta}(T^{-1})/\delta) \right) \\
        \le  &  c \left( \sum_{i=1}^{t-1}\sum_{\pi \in \Pi^i_{\exp}} \log\P^{\pi}_{{\theta}^\star}(\tau^{i,\pi}) - \max_{\theta\in\Theta}\sum_{i=1}^{t-1} \sum_{\pi \in \Pi^i_{\exp}} \P^{\pi}_{\theta}(\tau^{i,\pi}) +\beta +\log(T\cN_{\Theta}(T^{-1})/\delta)\right) \\
        \le  &  \mathcal{\cO}(\beta).
    \end{aligned}
\end{equation}

\paragraph{Step 3.} By combining Equation \eqref{eq:valid} with the generalized eluder-type condition (Condition \ref{cond:eluder}), we have 
that with probability at least $1-\delta$: for all $k\in[K]$, 
$$
\sum_{t=1}^k d_{\rm TV}(\P^{\pi^t}_{\theta^\star},\P^{\pi^t}_{\theta^t})\le H \xi(d_\Theta,k,\mathcal{\cO}(\beta),|\Pie|),
$$
which together with Equation \eqref{eq:opt} completes the proof.

\subsection{Properties of the \omle confidence set}


In this section, we present two important properties satisfied by the model confidence set $\cB^k$ in \omle. 
The results and proofs in this section largely follow \cite{liu2022partially} with minor modification, wherein the analysis builds on  classic techniques for analyzing MLE \cite[e.g.,][]{geer2000empirical}.

Similar to \cite{liu2022partially}, we will present our results for the following general meta-algorithm, of which \omle~is a special case.  

\begin{algorithm}[H]
    \caption{Meta-algorithm}\label{alg:meta}
 \begin{algorithmic}
    \FOR{$t=1,\ldots,T$}
    \STATE choose policy $\pi^t$ as a deterministic function of  $\{(\pi^i,\tau^i)\}_{i=1}^{t-1}$
    \STATE execute policy $\pi^t$ and collect a trajectory $\tau^t$
    \ENDFOR
 \end{algorithmic}
 \end{algorithm}

 Our first result states that the log-likelihood of the true model on the historical data is always close to the empirical maximal log-likelihood up to an error of logarithmic scale.

 \begin{proposition}\label{prop:mle-optimisim}
    There exists an absolute constant $c$ such that for any $\delta\in(0,1]$, with probability at least $1-\delta$: the following inequality holds for all $t\in[T]$ and \textbf{all} $\theta\in\Theta$
        \begin{equation*}
            \sum_{i=1}^{t} \log\left[\frac{\P^{\pi^i}_{{\theta}}(\tau ^i)}{\P^{\pi^i}_{{\theta}^\star}(\tau ^i)}\right] \le c\log\left[\cN_\Theta\left({1}/{T}\right)T/\delta\right] .
        \end{equation*}
    \end{proposition}

Our second result states that any model in $\Theta$ with log-likelihood close to the groundtruth model on the collected data will produce similar distributions over trajectories to  the groundtruth model under the historical policies.

\begin{proposition}\label{prop:mle-valid}
    There exists a universal constant $c$ such that for any  $\delta\in(0,1]$, with probability at least $1-\delta$  for all $t\in[T]$ and \textbf{all} $\theta\in\Theta$, it holds that
        \begin{equation*}
            \begin{aligned}
     \sum_{i=1}^{t}\left( \sum_{\tau\in\cT_H}\left| \P^{\pi^i}_{\theta}(\tau) - \P^{\pi^i}_{\theta^\star}(\tau) \right|\right)^2 
                \le   c \left( \sum_{i=1}^{t} \log\left[\frac{\P^{\pi^i}_{{\theta}^\star}(\tau ^i)}{\P^{\pi^i}_{{\theta}}(\tau ^i)}\right] +\log\left[\cN_\Theta\left({1}/{T}\right)T/\delta\right]\right).
            \end{aligned}
        \end{equation*}
    \end{proposition}

\begin{proof}[Proof of Proposition \ref{prop:mle-optimisim} and Proposition \ref{prop:mle-valid}]

By the definition of bracketing number, given any model $\theta\in\Theta$, there exists $f_\theta \in\Theta_{1/T}$ (a $T^{-1}$ bracketing-cover of $\Theta$) so that for any policy $\pi$: 
\begin{itemize}
    \item $f_\theta (\tau) \ge \Pb_\theta (\tau)$ for all   $\tau\in\cT_H$,
    \item $\sum_{\tau\in\cT_H} | f_{\theta} (\tau) - \Pb_{\theta}(\tau)|\times \pi(\tau) \le 1/T$. 
\end{itemize}
Then the proofs of Proposition \ref{prop:mle-optimisim} and Proposition \ref{prop:mle-valid} are \emph{exactly the same} as those of Proposition 13 and 14 in \cite{liu2022partially}, by replacing their  $\P_{\bar \theta}^\pi(\tau)$ with our $f_{\theta}(\tau)\times \pi(\tau)$.
\end{proof}


\section{Proofs for  PSRs}\label{app:psr}

\newcommand{\Rw}{{\color{blue}{R_w}}}

\subsection{Expressive power of PSRs (Theorem \ref{thm:sdm-psr})}
\begin{proof}[Proof Theorem \ref{thm:sdm-psr}]
   Theorem \ref{thm:sdm-psr} is a direct corollary of Theorem \ref{thm:self-consistent}.
\end{proof}

 \subsection{Statistical hardness of learning general low-rank PSRs (Proposition \ref{prop:psr-hard})}
\label{app:psr-hard}

\begin{proof}[Proof of Proposition \ref{prop:psr-hard}]
   According to Theorem 6 in \cite{liu2022partially}, for any $\alpha\in(0,1/2)$, there exist a family of $1$-step $\alpha$-weakly revealing POMDPs with $\cO(1)$ latent states, observations and actions, so that any algorithm requires at least $\Omega(\min(\frac{1}{\alpha H},2^H))$ samples to learn a $(1/4)$-optimal policy with probability $1/6$ or higher. 
Since any $1$-step $\alpha$-weakly revealing POMDP can be represented as a PSR with $\cQ_h=\{o_{h+1}:~o_{h+1}\in\fO\}$ for any  $\alpha>0$, we know the above family of POMDPs is also a family of PSRs with $|\fO|,|\fA|,\max_h|\cQ_h|=\cO(1)$. 
However, this family of PSRs requires at least $\Omega(2^H)$ samples to learn a $1/4$-optimal policy with probability $1/6$ or higher, whenever $\alpha<\frac1{2^H H}$.\end{proof}

\subsection{Self-consistent PSR representation}\label{app:self-consistent}

Recall in Section \ref{subsec:psr-defn}, we show that for any PSR,   there exist linear weight vectors $\{\m(\omega_h)\}_{\omega_h \in \Omega_h}$ only depending on the futures (where $\m(\omega_h)$ can be the $\omega_h^{\text{th}}$ row of $\W_h$ matrix) such that for any future $\omega_h$ and history $\tau_h$, the joint probability can be written in the linear form 
\begin{equation*}
    \Pb(\tau_h, \omega_h) = \m(\omega_h)\trans \bpsi(\tau_h).
\end{equation*} 
Moreover,  there are two natural constructions for the weight vector $\{\m(\omega_h)\}_{\omega_h \in \Omega_h}$ based on the PSR operator tuple
$(\bphi, \bM, \bpsi_0)$:
For any future of full length $\omega_h = (o_{h+1:H}, a_{h+1:H})$, we can choose its weight vector according to \eqref{eq:PSR_1} as:
\begin{equation*}
    \m_1(\omega_h)\trans =  \bphi_H(o_H, a_H)\trans \bM_{H-1}(o_{H-1}, a_{H-1}) \cdots \bM_{h+1}(o_{h+1}, a_{h+1}).
\end{equation*}
For any future of form  $\omega_h = (o_{h+1}, a_{h+1}, q_i)$, where $q_i \in \cQ_{h+1} $ is the a core test at step $h+1$, we can choose its weight vector according to \eqref{eq:PSR_2} as:
\begin{equation*}
    \m_2(\omega_h)\trans =  \bm{e}_i \trans \bM_{h+1}(o_{h+1}, a_{h+1}).
\end{equation*}
In general, the above two weight vectors could be  different for the same future $\omega_h$ when the PSR is overparameterized, i.e., $|\cQ_h|>{\rm rank}(\D_h)$.
However, the following (perhaps surprising) theorem states that for any low-rank sequential decision problem, there always exists a \emph{self-consistent} PSR representation where the above two different ways of choosing the future weight vectors give the  same representation for the same future $\omega_h$.

\begin{theorem}[self-consistent PSR representation]\label{thm:self-consistent}
For any sequential decision making problem of rank $r$, there exists a self-consistent rank-$r$ PSR representation specified by operators $\{\bphi_h\}_{h=0}^{H}$,  $\{\bM_h(o,a)\}_{h,o,a}$,  $\bpsi_0$ and core test sets $\{\cQ_h\}_{h=0}^{H-1}$ so that
 \begin{enumerate}
 \item For any $\tau_h\in \cT_h$, $\P(o_{1:h} \mid a_{1:h})=\bphi_h\trans \bM_h(o_h,a_h)\cdots\bM_1(o_1,a_1)\bpsi_0$.
 \item   For any $\tau_h\in \cT_h$ , $ \bM_h(o_h,a_h)\cdots\bM_1(o_1,a_1)\bpsi_0= [\D_h]_{\cQ_h,\tau_h}$.
    \item For any $(h,a)$, $\bphi_{h}\trans\sum_o \bM_h(o,a) = \bphi_{h-1}\trans$.  
    \item For any $\omega_h=(o_h,a_h,q)\in\fO\times\fA\times\cQ_h$ $$[\bM_h(o_h,a_h)]_{q,:} = \bphi_{h+m}\trans \bM_{h+m}(o_{h+m},a_{h+m})\cdots  \bM_h(o_h,a_h).
    $$
\end{enumerate}
\end{theorem}

One important consequence of Theorem \ref{thm:self-consistent} is that for a  self-consistent PSR to be  $\gamma$-\name, we only need the weight vectors $\m_1(\cdot)$ constructed in \eqref{eq:weight_vector_1} to be properly upper bounded.
\begin{proposition}\label{prop:self-consistent-asp}
    A self-consistent PSR is $\gamma$-\name if and only if for any $h\in[H-1]$ and any policy $\pi$ independent of the history before step $h$:
\begin{equation*}
\max_{\x \in \R^{|\cQ_h|}:~\norm{\x}_1 \le 1} \sum_{\omega_{h}\in{(\fO\times\fA)^{H-h}}} \pi( \omega_h) \cdot |\m_1(\omega_h)\trans\x|  \le \frac{1}{\gamma}~.
\end{equation*}
\end{proposition}

\subsubsection{Proof of Theorem \ref{thm:self-consistent}}

\begin{proof}[Proof of Theorem \ref{thm:self-consistent}]
Recall in the proof of Theorem  \ref{thm:oom-represent}, we show that for  any sequential decision making problem with ${\rm rank}(\D_h)=r_h$ for $h\in[0,H-1]$, there exist operators $\b_0$, $\{\B_h(o,a)\}_{(h,o,a)\in[H]\times\fO\times\fA}$ and $\{\bup_h\}_{h\in[0,H]}$ satisfying that   
\begin{enumerate}
    \item
    $\b_0\in\R$, 
    $\B_h\in\R^{r_h \times r_{h-1}}$, and $\bup_h\in\R^{r_h}$,
    \item For any $(h,a_h)\in[H]\times \fA$, $\bup_{h}\trans\sum_{o_h\in\fO} \B_h(o_h,a_h) = \bup_{h-1}\trans$. 
    \item For any $\tau_h\in \cT_h$, $\P(o_{1:h} \mid a_{1:h})=\bup_h\trans \B_h(o_h,a_h)\cdots\B_1(o_1,a_1) \b_0$.
\end{enumerate}
Specifically, the above  operators are constructed as following:
$$
\b_0 = \|\D_0\|_2, \quad \B_h(o_h,a_h) =\U_h\trans [\U_{h-1}]_{(o_h,a_h,\Omega_{h}),:}, \quad \bup_h\trans = \frac{1}{A^{H-h}}\mathbf{1}\trans\U_h,
$$
where $\U_h$ is the left orthonormal matrix in the SVD of $\D_h\trans$, i.e., $\D_h\trans =\U_h\Sigma_h \V_h\trans$, and $[\U_{h-1}]_{(o_h,a_h,\Omega_{h}),:}$ denotes an $|\Omega_h|$ by $r_{h-1}$ submatrix  of $\U_{h-1}$, which is made up of the rows with  indices  of form $(o_h,a_h,\omega_h)$ with $\omega_h\in\Omega_h$. 

Note that the above operators do not necessarily correspond to a PSR representation since it is possible that they do not satisfy Equation \eqref{eq:PSR_2} for any choice of $\{\cQ_h\}_h$. 
Nonetheless, we can utilize them to  construct the following PSR operators: Let $\cQ_h$ be an arbitrary subset of $\Omega_h$,  which corresponds to $r_h$ independent columns in $\D_h$. Define 
$$
\bpsi_0: = [\U_0]_{\cQ_0,:} \b_0,  \quad \bM_h(o_h,a_h)=    [\U_h]_{\cQ_h,:}  \B_h(o_h,a_h) \left([\U_{h-1}]_{\cQ_{h-1},:}\right)^\dagger, \quad {\bphi_h} \trans= \bup_h \trans ([\U_h]_{\cQ_h,:})^\dagger.
$$

Since $\cQ_h$ indexes a rank-$r_h$ submatrix of $\D_h$, we have $\mathrm{rank}( [\D_h\trans]_{\cQ_h,:} )=r_h$, which implies $\mathrm{rank}([\U_h]_{\cQ_h,:})$ is also equal to $r_h$. As a result, $([\U_h]_{\cQ_h,:})^\dagger [\U_h]_{\cQ_h,:} = \I_{r_h\times r_h}$ for any $h$. As a result, for any $\tau_h\in \cT_h$, 
\begin{align*}
     &\bphi_h\trans \bM_h(o_h,a_h)\cdots\bM_1(o_1,a_1) \bpsi_0 \\
    = & \bphi_h\trans  ([\U_h]_{\cQ_h,:})^\dagger [\U_h]_{\cQ_h,:} \B_h(o_h,a_h)\cdots\B_1(o_1,a_1) ([\U_0]_{\cQ_0,:})^\dagger [\U_0]_{\cQ_0,:} \b_0\\
    = &  \bphi_h\trans \B_h(o_h,a_h)\cdots\B_1(o_1,a_1) \b_0= \P(o_{1:h} \mid a_{1:h}),
\end{align*}
which proves the first condition in Theorem \ref{thm:self-consistent}.

Moreover, the second condition holds because 
\begin{align*}
    &\bM_h(o_h,a_h)\cdots\bM_1(o_1,a_1) \bpsi_0 \\
   = & [\U_h]_{\cQ_h,:} \B_h(o_h,a_h)\cdots\B_1(o_1,a_1) ([\U_0]_{\cQ_0,:})^\dagger [\U_0]_{\cQ_0,:} \b_0\\
   = &  [\U_h]_{\cQ_h,:} \B_h(o_h,a_h)\cdots\B_1(o_1,a_1) \b_0\\
   = & [\U_h]_{\cQ_h,:} \U_h\trans [\D_h\trans]_{:,\tau_h} 
   = [\U_h\U_h\trans\D_h\trans]_{\cQ_h,\tau_h}
   = [\D_h\trans]_{\cQ_h,\tau_h},
\end{align*}
where the third equality uses Claim \ref{claim:oom-operator-product} from the proof of Theorem \ref{thm:oom-represent}.

It remains to verify the last two conditions. 
By the definition of $\bphi_{h}$ and $\bM_h(o,a)$, we have 
\begin{align*}
    & \bphi_{h}\trans\sum_o \bM_h(o,a) \\
    = & \bup_h \trans ([\U_h]_{\cQ_h,:})^\dagger  \sum_o [\U_h]_{\cQ_h,:}  \B_h(o,a) \left([\U_{h-1}]_{\cQ_{h-1},:}\right)^\dagger \\
    = & \bup_h \trans \sum_o  \B_h(o,a) \left([\U_{h-1}]_{\cQ_{h-1},:}\right)^\dagger \\
    = & \bup_{h-1}\trans \left([\U_{h-1}]_{\cQ_{h-1},:}\right)^\dagger  =  \bphi_{h-1}\trans,
\end{align*}
where the penultimate equality uses Theorem \ref{thm:oom-represent}.
Therefore, the third condition holds.

Given an arbitrary $(o_h,a_h,\omega_h)\in\fO\times\fA\times\cQ_h$, 
by using the second condition prove above, we know $\x=[\bM_h(o_h,a_h)]_{\omega_h,:}$ is a solution to  the following linear equations:
\begin{equation}\label{eq:Sep25-1}
    \x\trans [\D_{h-1}\trans]_{\cQ_{h-1},\tau_{h-1}} = [\D_h\trans]_{\omega_h,(\tau_{h-1},o_h,a_h)} \quad \mbox{ for all } \tau_{h-1}\in\cT_{h-1}.
\end{equation}
Moreover, by the first and second condition, we have 
\begin{align*}
    & \bphi_{h+m}\trans \bM_{h+m}(o_{h+m},a_{h+m})\cdots  \bM_h(o_h,a_h)[\D_{h-1}\trans]_{\cQ_{h-1},\tau_{h-1}} \\
    =&\bphi_{h+m}\trans \bM_{h+m}(o_{h+m},a_{h+m})\cdots  \bM_1(o_1,a_)\bpsi_0 \\
    =& [\D_h\trans]_{\omega_h,(\tau_{h-1},o_h,a_h)}.
\end{align*}
Therefore, $\x=\bphi_{h+m}\trans \bM_{h+m}(o_{h+m},a_{h+m})\cdots  \bM_h(o_h,a_h)$ is also a solution to linear equations in Equation  \eqref{eq:Sep25-1}. 
To prove the fourth condition, it suffices to prove the linear equations have a unique solution, which follows directly from  $[\D_{h-1}\trans]_{\cQ_{h-1},:}\in\R^{r_{h-1}\times|\cT_{h-1}|}$ being rank-$r_{h-1}$. 
\end{proof}

\subsubsection{Proof of Proposition \ref{prop:self-consistent-asp}}

\begin{proof}[Proof of Proposition \ref{prop:self-consistent-asp}]
We consider an arbitrary fixed $h\in[H]$ and assume all the polices are independent of the history before step $h+1$ throughout the proof. 

   The goal is to  prove that given a self-consistent PSR  representation, if \begin{equation*}
\max_\pi \max_{\x \in \R^{|\cQ_h|}:~\norm{\x}_1 \le 1} \sum_{\omega_{h}\in(\fO\times\fA)^{H-h}} \pi( \omega_h) \cdot |\m_1(\omega_h)\trans\x|  \le \frac{1}{\gamma}~,
\end{equation*}
then we also have 
\begin{equation*}
\max_\pi \max_{\x \in \R^{|\cQ_h|}:~\norm{\x}_1 \le 1} \sum_{\omega_{h}\in \fO\times\fA\times\cQ_{h+1}} \pi( \omega_h) \cdot |\m_2(\omega_h)\trans\x|  \le \frac{1}{\gamma}~.
\end{equation*}
Given an arbitrary $\omega_{h}:=(o_{h+1},a_{h+1},\omega_{h+1}):=(o_{h+1},a_{h+1},\ldots,o_{h+m},a_{h+m})\in\fO\times\fA\times\cQ_{h+1}$, by the third and fourth conditions in  Theorem \ref{thm:self-consistent}, we have 
\begin{equation}
    \begin{aligned}
    \label{eq:Sep25-2}
 \m_2(\omega_h)\trans =&[\bM_{h+1}(o_{h+1},a_{h+1})]_{\omega_{h+1},:}\\
 = &\bphi_{h+m}\trans \bM_{h+m}(o_{h+m},a_{h+m})\cdots  \bM_{h+1}(o_{h+1},a_{h+1}) \\
= & \frac{1}{A}\bphi_{h+m+1}\trans\left(\sum_{(o_{h+m+1},a_{h+m+1})\in\fO\times\fA} \bM_{h+m+1}(o_{h+m+1},a_{h+m+1})\right)\\
& \qquad \times \bM_{h+m}(o_{h+m},a_{h+m})\cdots  \bM_{h+1}(o_{h+1},a_{h+1}) \\
&\quad \vdots\\
 =  & \frac{1}{A^{H-h-m}}\sum_{(o,a)_{h+m+1:H}\in (\fO\times\fA)^{H-h-m}} \bphi_H\trans \bM_H(o_H,a_H)\cdots \bM_{h+1}(o_{h+1},a_{h+1}).
    \end{aligned}
\end{equation}
To proceed, we define an auxiliary policy $\tilde \pi$ which first executes policy $\pi$ starting from step $h+1$ and then immediately turns to ${\rm Uniform}(\fA)$ once a core action sequence $\a\in\Qa_h$ has been rolled out from following $\pi$. Since  no core action sequence is a prefix of another one, such $\tilde\pi$ is well defined. 
Denote by $|\omega_h|$ the number of actions in $\omega_h$. 
Using Equation \eqref{eq:Sep25-2}, we have for any $\x \in \R^{|\cQ_h|}$ with $\|\x\|_1\le1$:
\begin{equation*}
\begin{aligned}
& \sum_{\omega_{h}\in \fO\times\fA\times\cQ_{h+1}} \pi( \omega_h) \cdot |\m_2(\omega_h)\trans\x| \\
= & \sum_{\omega_{h}\in \fO\times\fA\times\cQ_{h+1}} \pi( \omega_h) \cdot \left|\frac{1}{A^{H-h-|\omega_h|}}\sum_{(o,a)_{h+|\omega_h|+1:H}} \bphi_H\trans \bM_H(o_H,a_H)\cdots \bM_{h+1}(o_{h+1},a_{h+1})\x\right|\\
\le  & \sum_{\omega_{h}\in \fO\times\fA\times\cQ_{h+1}}~
\sum_{(o,a)_{h+|\omega_h|+1:H}}
\frac{\pi( \omega_h)}{A^{H-h-|\omega_h|}} \cdot \left|\bphi_H\trans \bM_H(o_H,a_H)\cdots \bM_{h+1}(o_{h+1},a_{h+1})\x\right|\\
\le   & \sum_{(o,a)_{h+1:H}\in(\fO\times\fA)^{H-h}}
\tilde\pi((o,a)_{h+1:H}) \cdot \left|\bphi_H\trans \bM_H(o_H,a_H)\cdots \bM_{h+1}(o_{h+1},a_{h+1})\x\right|\\
\le &\max_\pi \max_{\x \in \R^{|\cQ_h|}:~\norm{\x}_1 \le 1} \sum_{\omega_{h}\in(\fO\times\fA)^{H-h}} \pi( \omega_h) \cdot |\m_1(\omega_h)\trans\x|  \le \frac{1}{\gamma}.
\end{aligned}
\end{equation*}
\end{proof}

\subsection{Proofs for well-conditioned PSRs (Theorem \ref{thm:psr})}
\label{app:psr-main}

For cleaner presentation of the key ideas/techniques of the proof, we will first  prove Theorem \ref{thm:psr} for the finite-observation setting.
After that, we will describe how to modify the proof so that it also works for the continuous-observation setting (Appendix \ref{app:psr-main-continuous}).

To prove Theorem \ref{thm:psr}, it suffices to prove Lemma \ref{prop:psr-eluder} which states that any well-conditioned PSRs  satisfy the generalized elude-type condition (Condition \ref{cond:eluder}). Then Theorem \ref{thm:psr} will follow directly from combining the bound in Lemma  \ref{prop:psr-eluder} with the guarantee of \omle (Theorem \ref{thm:omle}).

\paragraph{Proof outline.} The proof of Lemma \ref{prop:psr-eluder}   consists of four main steps:
\begin{enumerate}
    \item First we upper bound the TV distance by the PSR operator difference. 
    \item Conversely,  we  upper bound the operator difference by the TV distance incurred by the   exploration policies.
    \item To bridge the first two steps, we then establish a sharp   $\ell_1$-norm pigeon-hole regret bound for \textbf{S}ummation of \textbf{A}bsolute values of \textbf{I}ndependent bi\textbf{L}inear  (SAIL) functions, which improves over the previous bound  \citep{liu2022partially} by completely getting rid of the dependence on the number of feature mappings  (observations).
    \item Finally, we conclude the proof by combining the first  two steps with the  $\ell_1$-norm pigeon-hole regret bound for SAIL functions.
\end{enumerate}

\paragraph{Notations.} To simplify notations, we will use $(\bM^t,\bpsi^t,\bphi^t)$, $(\widehat\bM,\widehat\bpsi,\widehat\bphi)$ and $\bM,\bpsi,\bphi$ to denote operators corresponding to $\theta^t$, $\widehat\theta$ and $\theta^\star$ respectively. Furthermore, to condense  presentation, we will use $\bM_H(o_H,a_H)$, $\widehat\bM_H(o_H,a_H)$, and $\bM^t_H(o_H,a_H)$  to denote $\bphi_H(o_H,a_H)$, $\widehat\bphi_H(o_H,a_H)$, and $\bphi^t_H(o_H,a_H)$, respectively. 
Finally, we  denote  $\m((o,a)_{h:H}):=\bM(o_H,a_H)\times\cdots\times\bM(o_h,a_h)$, and $\widehat\m((o,a)_{h:H})$ similarly. 

\noindent \paragraph{STEP 1. TV-distance $\le$ operator-difference}

\begin{lemma}\label{lem:step1}
Suppose Condition \ref{asp:psr} holds, then for any $\widehat\theta\in\Theta$ and policy $\pi$: \begin{equation*}
    \begin{aligned}
      d_\TV( \P_{\widehat\theta}^{\pi} , \P_{\theta^\star}^{\pi})
        \le \frac{1}{2\gamma} \bigg( \sum_{h=1}^{H}\sum_{\tau_h} \left\| (\widehat\bM_h(o_{h},a_{h}) - \bM_h(o_{h},a_{h}))\bpsi(\tau_{h-1})\right\|_1\times \pi(\tau_h)   + \left\|\widehat\bpsi_0 - \bpsi_0\right\|_1 
        \bigg).
    \end{aligned}
\end{equation*}
\end{lemma}
\begin{proof}[Proof of Lemma \ref{lem:step1}]
By the definition of PSR, we can represent the probability of observing a trajectory by the product of operators, 
\begin{equation}\label{eq:proof-1}
    \begin{aligned}
       d_\TV( \P_{\widehat\theta}^{\pi} , \P_{\theta^\star}^{\pi})&= \frac{1}{2}  \sum_{\tau_H   } | \P^{\pi}_{{\widehat\theta}}(\tau_H) - \P^{\pi}_{\theta^\star}(\tau_H)| \\
        = & \frac{1}{2}  \sum_{\tau_H   }  \left|  \left(\prod_{h=1}^{H}\widehat\bM_h(o_{h},a_{h})\right) \widehat\bpsi_0 - 
        \left(\prod_{h=1}^{H}\bM_h(o_{h},a_{h})\right) \bpsi_0\right| \times \pi(\tau_H) \\
        \le & \frac{1}{2}  \bigg( \sum_{h=1}^{H}\sum_{\tau_H} \left| \widehat\m((o,a)_{h+1:H})\trans (\widehat\bM_h(o_{h},a_{h}) - \bM_h(o_{h},a_{h}))\bpsi(\tau_{h-1})\right|\times \pi(\tau_H) \\
        & \qquad + \sum_{\tau_H} \left| \widehat\m((o,a)_{1:H})\trans (\widehat\bpsi_0 - \bpsi_0)\right| \times \pi(\tau_H)
        \bigg) .
    \end{aligned}
\end{equation}
By Condition \ref{asp:psr}, we can further upper bound the RHS of Equation \eqref{eq:proof-1} by 
\begin{equation*}
    \begin{aligned}
        \frac{1}{2\gamma}  \bigg( \sum_{h=1}^{H}\sum_{\tau_h} \left\| (\widehat\bM_h(o_{h},a_{h}) - \bM_h(o_{h},a_{h}))\bpsi(\tau_{h-1})\right\|_1\times \pi(\tau_h)   + \left\| \widehat\bpsi_0 - \bpsi_0\right\|_1 
        \bigg),
    \end{aligned}
\end{equation*}
which completes the proof.
\end{proof}

\noindent \paragraph{STEP 2. operator-difference $\le$ TV-distance}~\vspace{3mm}

\noindent Denote by $\nu(\pi,h,\a)$ a composite policy that first executes policy $\pi$ for step $1$ to step $h-1$, then takes  random action at step $h$, after that executes  action sequence $\a$ starting from step $h+1$, and finally finishes the remaining steps of the current episode by taking random actions.  
\begin{lemma}\label{lem:step2}
Suppose Condition \ref{asp:psr} holds, then for any $\widehat\theta\in\Theta$ and policy $\pi$: \begin{equation*}
    \begin{aligned}
       \MoveEqLeft \sum_{h=1}^{H}\sum_{\tau_h} \left\| (\widehat\bM_h(o_{h},a_{h}) - \bM_h(o_{h},a_{h}))\bpsi(\tau_{h-1})\right\|_1\times \pi(\tau_h)   + \left\|\widehat\bpsi_0 - \bpsi_0\right\|_1 \\
    &= \cO\left(\frac{\max_h|\Qa_h|A^2}{\gamma}\right) \times \sum_{\tilde\pi\in \Pi_{\exp}(\pi)} d_\TV( \P_{\widehat\theta}^{\tilde \pi} , \P_{\theta^\star}^{\tilde\pi}),
    \end{aligned}
\end{equation*}
where $\Pi_{\rm exp}(\pi):=\bigcup_{h\in[0,H-1]}\{\nu(\pi,h,\a):~ \a\in \Qa_h \}$.
\end{lemma}
\begin{proof}[Proof of Lemma \ref{lem:step2}]
By the probabilistic meaning of $\widehat\bpsi$ and $\bpsi$, we directly have
\begin{equation*}\label{eq:proof-3}
    \left\| \widehat\bpsi_0- \bpsi_0\right\|_1 \le 2\sum_{\a\in\cQ_0^A}  d_{\rm TV} (\P^{\nu(\pi,0,\a)}_{{\widehat\theta}}, \P^{\nu(\pi,0,\a)}_{\theta^\star}).
\end{equation*}
For any $h\in[H]$, we have 
\begin{equation*}
    \begin{aligned} 
        \MoveEqLeft  \sum_{\tau_h} \left\| (\widehat\bM(o_{h},a_{h}) - \bM_h(o_{h},a_{h}))\bpsi(\tau_{h-1})\right\|_1\times \pi(\tau_{h-1})   \\
        \overset{(i)} {\le}&  \sum_{\tau_h} \left\| \widehat \bM_h(o_{h},a_{h}) \left(\widehat\bpsi(\tau_{h-1})-\bpsi(\tau_{h-1})\right) \right\|_1\times \pi(\tau_{h-1}) \\
        & +   \sum_{\tau_h} \left\| \widehat\bpsi(\tau_{h})-\bpsi(\tau_{h})\right\|_1\times \pi(\tau_{h-1})  \\
       \overset{(ii)}{\le} & \frac{|\Qa_h|A}{\gamma}  \sum_{\tau_{h-1}} \left\| \widehat\bpsi(\tau_{h-1})-\bpsi(\tau_{h-1})\right\|_1\times \pi(\tau_{h-1})  \\
        & \qquad +  \sum_{\tau_h} \left\| \widehat\bpsi(\tau_{h})-\bpsi(\tau_{h})\right\|_1\times \pi(\tau_{h-1})\\
        \overset{(iii)}{\le} & \frac{2|\Qa_h|A^2}{\gamma}  \sum_{\a\in\cQ_{h-1}^A}  d_{\rm TV} (\P^{\nu(\pi,h-1,\a)}_{{\widehat\theta}}, \P^{\nu(\pi,h-1,\a)}_{\theta^\star}) + 2A \sum_{\a\in\Qa_h}  d_{\rm TV} (\P^{\nu(\pi,h,\a)}_{{\widehat\theta}}, \P^{\nu(\pi,h,\a)}_{\theta^\star}),
    \end{aligned}
\end{equation*}
where we first use (i) triangle inequality and $\psi(\tau_h)=\bM_h(o_h,a_h)\psi(\tau_{h-1})$, then (ii) Condition \ref{asp:psr}, and finally (iii) the probabilistic meaning of $\widehat\bpsi$ and $\bpsi$.
\end{proof}

\noindent \paragraph{STEP 3.
Observation-independent $\ell_1$-norm pigeon-hole  arguments for SAIL functions}~\vspace{3mm}

\noindent We  state and prove the following  $\ell_1$-norm pigeon-hole argument for \textbf{S}ummation of \textbf{A}bsolute values of \textbf{I}ndependent bi\textbf{L}inear  (SAIL) functions.  
\begin{proposition}[ Pigeon-hole principle for SAIL functions]\label{prop:step-3}
    Suppose there exist two sets of  feature mappings $\{\{f_{h,i}\}_{i=1}^{m}\}_{h=0}^H$ and $\{\{g_{h,i}\}_{j=1}^{n}\}_{h=0}^H$ from $\Theta$ to $\R^d$ such that for any $(\theta,\theta',h)\in\Theta\times\Theta\times\{0,\ldots,H\}$
    \begin{itemize}
     \item  $\sum_{\tilde\pi\in \Pi_{\exp}(\pi_\theta)} d_\TV(  \P_{\theta^\star}^{\tilde\pi},\P_{\theta'}^{\tilde\pi} ) \ge C_{\rm low} \sum_{h=0}^H \sum_{i=1}^{m}\sum_{j=1}^{n} \left |\langle  f_{h,i}(\theta), g_{h,j}(\theta') \rangle \right|$,
        \item $d_\TV(  \P_{\theta^\star}^{\pi_\theta},\P_{\theta'}^{\pi_\theta}) \le C_{\rm up}  \sum_{h=0}^H \sum_{i=1}^{m}\sum_{j=1}^{n} \left |\langle  f_{h,i}(\theta), g_{h,j}(\theta') \rangle \right|$,
        \item $\left(\sum_{i=1}^{m} \|f_{h,i}(\theta)\|_1\right) \times\left(\sum_{j=1}^{n} \|g_{h,j}(\theta')\|_\infty\right) \le C_{\rm norm}$.
    \end{itemize}
Then we have $\Theta$ satisfies Condition
\ref{cond:eluder} with 
\begin{equation*}
\begin{cases}
 d_\Theta =  \frac{  C_{\rm up}^2 d^2  \poly(H)}{C_{\rm low}^2}, \\
 \xi(d_\Theta,\Delta,K,|\Pie|) = \log^2(K)\left( \sqrt{d_\Theta \Delta |\Pie| K} +d C_{\rm up}C_{\rm norm}    \poly(H)\right).
\end{cases}
\end{equation*}
\end{proposition}
\begin{proof}[Proof of Lemma \ref{prop:step-3}]
By the precondition of Condition \ref{cond:eluder}, the precondition of Proposition \ref{prop:step-3} and Cauchy-Schwarz inequality, we have that for any $h\in\{0,\ldots,H\}$  
$$
  \sum_{t=1}^{k-1}\sum_{i=1}^{m}\sum_{j=1}^{n} \left |\langle  f_{h,i}(\theta^k), g_{h,j}(\pi_{\theta^t}) \rangle \right|
\le \frac{1}{C_{\rm low}} \sqrt{k|\Pie|\Delta}.
$$
Therefore, by Proposition \ref{prop:pigeon-hole}, we immediately have:     for all  $k\in[K]$
    \begin{equation*}
    \begin{aligned}
        \sum_{t=1}^{k} d_\TV( \P_{\theta^t}^{\pi^t} , \P_{\theta^\star}^{\pi^t})  
        &\le C_{\rm up}  \sum_{t=1}^{k} \sum_{h=0}^H \sum_{i=1}^{m}\sum_{j=1}^{n} \left |\langle  f_{h,i}(\theta^t), g_{h,j}(\theta^t) \rangle \right| \\
        &=\cO \left(  C_{\rm up} d\log^2(k)\left(C_{\rm norm}   + \frac{1}{C_{\rm low}} \sqrt{k|\Pie|\Delta} \right)H\right).
    \end{aligned}
    \end{equation*}
\end{proof}

Now we are ready to prove that well-conditioned PSRs satisfy the generalized eluder-type condition. And our approach is to utilize Lemma \ref{lem:step1} and Lemma \ref{lem:step2} proved in the first two steps to explicitly construct the $f,g$ feature mappings in Proposition \ref{prop:step-3} and calculate the corresponding parameters $C_{\rm up},~C_{\rm low}$ and $C_{\rm norm}$. 
\noindent \paragraph{STEP 4.
Construct SAIL feature mappings for PSRs}~\vspace{3mm}

The following lemma states that we can construct two sets of feature mappings for well-conditioned PSRs, which satisfy a \emph{stronger} version of the precondition in Proposition \ref{prop:step-3}.
Once this lemma is established, we can conclude the entire proof by plugging the value of $d,C_{\rm up},C_{\rm low},C_{\rm norm}$  back into Proposition \ref{prop:step-3}.
Denote by $\Pi$ the universal policy space.

\begin{lemma}[well-conditioned PSRs $\subset$ SAIL]\label{lem:psr-sail}
Suppose Condition \ref{asp:psr} holds.  
Then there exist two sets of  feature mappings $\{\{f_{h,i}\}_{i=1}^{m}\}_{h=0}^H$ from $\Pi$ to $\R^d$ and $\{\{g_{h,i}\}_{j=1}^{n}\}_{h=0}^H$ from $\Theta$ to $\R^d$ such that for any $(\pi,\theta,h)\in\Pi\times\Theta\times\{0,\ldots,H\}$
    \begin{itemize}
     \item  $\sum_{\tilde\pi\in \Pi_{\exp}(\pi)} d_\TV(  \P_{\theta^\star}^{\tilde\pi},\P_{\theta}^{\tilde\pi} ) \ge C_{\rm low} \sum_{h=0}^H \sum_{i=1}^{m}\sum_{j=1}^{n} \left |\langle  f_{h,i}(\pi), g_{h,j}(\theta) \rangle \right|$,
        \item $d_\TV(  \P_{\theta^\star}^{\pi},\P_{\theta}^{\pi}) \le C_{\rm up}  \sum_{h=0}^H \sum_{i=1}^{m}\sum_{j=1}^{n} \left |\langle  f_{h,i}(\pi), g_{h,j}(\theta) \rangle \right|$,
        \item $\left(\sum_{i=1}^{m} \|f_{h,i}(\pi)\|_1\right) \times\left(\sum_{j=1}^{n} \|g_{h,j}(\theta)\|_\infty\right) \le C_{\rm norm}$,
    \end{itemize}
    where 
    $$
C_{\rm up} = \cO\left(\frac{1}{\gamma}\right), \quad C_{\rm low}= \Omega\left(\frac{\gamma}{\max_h|\Qa_h|A^2}\right), \quad \text{ and } C_{\rm norm}= \cO\left(\frac{Ar^2\max_h|\Qa_h|^2}{\gamma}\right).
$$
\end{lemma}
\begin{proof}[Proof of Lemma \ref{lem:psr-sail}]
We first construct the feature mappings. 
For $h=0$, let 
\begin{equation*}
    \begin{cases}
     f_{0,i}(\pi) := 1 \quad i\in[1], \\
     g_{0,q}(\widehat\theta): = \e_{q}\trans (\widehat\bpsi_0-\bpsi_0) \quad q\in\cQ_0. 
    \end{cases}
\end{equation*}
For any fixed $h>0$, by Lemma \ref{lem:barycentric}, there exists $\A_h\in\R^{|\cQ_{h-1}|\times r_{h-1}}$ such that $\forall \tau_{h-1}\in\cT_{h-1}$
\begin{equation*}
    \begin{cases}
     \A_h \A_h^\dagger\bpsi(\tau_{h-1})   = \bpsi(\tau_{h-1}), \\
     \| \A_h^\dagger\bpsi(\tau_{h-1}) \|_1 \le r_h \|\bpsi(\tau_{h-1})\|_1,\\
     \|\A_h\|_1 \le 1.
    \end{cases}
\end{equation*}
We construct the following feature mappings of dimension no larger than $ r=\max_h r_h$ (the rank of PSRs):
\begin{equation*}
    \begin{cases}
     g_{h,(o_h,a_h,q_h)}(\widehat\theta) \trans:= \e_{q_h}\trans (\widehat\bM_h(o_h,a_h) - \bM_h(o_h,a_h))\A_h , \quad (q_h,o_h,a_h) \in\cQ_h\times\fO\times\fA,\\
     f_{h,\tau_{h-1}}(\pi):=\A_h^\dagger\bpsi(\tau_{h-1}) \pi(\tau_{h-1}), \quad \tau_{h-1}\in(\fO\times\fA)^{h-1}.
    \end{cases}
\end{equation*}
We remark that for different $h$, the dimension of the features and the range of $i,j$ are different. Nonetheless, we can still apply Proposition \ref{prop:step-3} by simply (a) lifting the low dimensional features into high dimension by padding the added dimension with zeros, and (b) adding zero feature mappings to make the range of $i,j$ consistent across different steps. 

Now, we calculate the parameters $C_{\rm up}, C_{\rm low}, C_{\rm norm} $ for the above  feature mappings.
By  invoking Lemma \ref{lem:step1} and \ref{lem:step2} proved in the first two steps, we immediately obtain
$$
C_{\rm up} = \cO\left(\frac{1}{\gamma}\right) \qquad \text{and} \quad C_{\rm low}= \Omega\left(\frac{\gamma}{\max_h|\Qa_h|A^2}\right).
$$
Furthermore, by using the fact that $\|\A_h\|_1\le1$ and Condition \ref{asp:psr}, we have 
\begin{equation*}
\begin{aligned}
    \sum_{v,o_h,a_h} \|g_{h,(v,o_h,a_h)}(\widehat\theta)\|_\infty & \le \sum_{v,o_h,a_h} \|g_{h,(v,o_h,a_h)}(\widehat\theta)\|_1\\
    	= &  \sum_{o_h,a_h} \sum_{j=1}^{r} \left\|\left(\widehat\bM_h(o,a) - \bM_h(o,a)\right)\A_h \e_j  \right\|_1 
	 \le\frac{2|\Qa_h|Ar}{\gamma},\label{eqn:fbound}
\end{aligned}
\end{equation*}
and by using the definition of $\A_h$ as well as the probabilistic interpretation of $\bpsi(\tau_{h-1})$, we have
\begin{equation*}
    \sum_{\tau_{h-1}} \|f_{h,\tau_{h-1}}(\pi)\|_1 
            \le r \sum_{\tau_{h-1}}  \| \bpsi(\tau_{h-1})\|_1 \times \pi(\tau_{h-1})
              \le r |\Qa_{h-1}|.
\end{equation*}
As for $h=0$, we trivially have 
$\sum_{q\in\cQ_0}|g_{0,q}(\widehat\theta)|\le 2|\cQ_0^A|$ and $\sum_{i=1}^1 |f_{0,i}(\pi)|=1$. Therefore, we obtain
$$
C_{\rm norm}  \le R_f R_g \quad \mbox{where}\quad 
R_f = \cO\left(\frac{Ar\max_h|\Qa_h|}{\gamma}\right) \qquad \text{and} \quad R_g= \cO\left(r \max_h |\Qa_{h}|\right).
$$
\end{proof}

\subsection{Bracketing number of  tabular PSRs (Theorem \ref{thm:tabularPSR-bracket})}

Since any rank-$r$ PSR has system-dynamic matrices of rank no larger than $r$ \citep[e.g.,][]{singh2012predictive}, to prove Theorem \ref{thm:tabularPSR-bracket}, it suffices to upper bound the bracketing number of the collections of all rank-$r$ sequential decision making models with $O$ observations, $A$ actions and episode length $H$.  

The following theorem states that any rank-$r$ sequential decision making model can be represented by a finite number of operators, the size of which is at most $r$ by $r$ and the norm of which is properly upper bounded. 

\begin{theorem} [self-consistent OOM representation]\label{thm:oom-represent}
For any sequential decision making problem with ${\rm rank}(\D_h)=r_h$ for $h\in[0,H-1]$, there exist operators $\b_0$, $\{\B_h(o,a)\}_{(h,o,a)\in[H]\times\fO\times\fA}$ and $\{\bup_h\}_{h\in[0,H]}$ satisfying that   
\begin{enumerate}
    \item $\B_h\in\R^{r_h \times r_{h-1}}$ and $\|\B_h\|_2 \le 1$.
    \item $|\b_0| \le \sqrt{A^H}$.
    \item $\|\bup_h\|_2 \le \sqrt{(O/A)^{H-h}}$.
    \item For any $(h,a_h)\in[H]\times \fA$, $\bup_{h}\trans\sum_{o_h\in\fO} \B_h(o_h,a_h) = \bup_{h-1}\trans$. 
    \item For any $\tau_h\in \cT_h$, $\P(o_{1:h} \mid a_{1:h})=\bup_h\trans \B_h(o_h,a_h)\cdots\B_1(o_1,a_1) \b_0$.
\end{enumerate}
\end{theorem}
\begin{proof}[Proof of Theorem \ref{thm:oom-represent}]
WLOG, define $\D_H\in\R^{|\cT_H|\times 1}$ such that $[\D_H]_{\tau_H}=\Pb(\tau_H)$. 
We consider the SVD of $\D_h\trans$: $\D_h\trans= \U_h \Sigma_h \V_h\trans$ and define  
$$
\b_0 = \|\D_0\|_2 \quad \B_h(o_h,a_h) = \U_h\trans [\U_{h-1}]_{(o_h,a_h,\Omega_{h}),:} \quad \bup_h\trans = \frac{1}{A^{H-h}}\mathbf{1}\trans \U_h,
$$
where $[\U_{h-1}]_{(o_h,a_h,\Omega_{h}),:}$ denotes an $|\Omega_h|$ by $r_{h-1}$ submatrix  of $\U_{h-1}$, which is made up of the rows with  indices  of form $(o_h,a_h,\omega_h)$ with $\omega_h\in\Omega_h=(\fO\times\fA)^{H-h}$. 

It is straightforward to verify that  these operators satisfy the norm condition  1-3. It remains to prove they also satisfy condition 4-5. 
To verify condition $5$, we first  prove the following claim. 
\begin{claim}\label{claim:oom-operator-product}
For any $h\in\{0,\ldots,H\}$ and $\tau_h\in\cT_h$, $\B_h(o_h,a_h)\cdots\B_1(o_1,a_1) \b_0 = 
\U_h\trans [\D_h\trans]_{:,\tau_h}$.
\end{claim}
\begin{proof}[Proof of Claim \ref{claim:oom-operator-product}]
For the base case $h=0$, the history is empty and $\D_0\trans$ is simply a vector in  $\R^{(OA)^H}$.
Therefore, $\U_0 = \D_0\trans/\|\D_0\trans\|_2$ and $\U_0\trans \D_0\trans  = \D_0 \D_0\trans /\|\D_0\|_2 = \|\D_0\|_2 = \b_0$.
Next we prove by induction that the claim holds for all $h\in\{0,\ldots,H\}$. 
Assume the claim holds for step $h-1$. At step $h$, by using the induction hypothesis and the definition of $\B_h(o_h,a_h)$
\begin{align*}
    \B_h(o_h,a_h)\cdots\B_1(o_1,a_1) \b_0  
    & = \U_h\trans [\U_{h-1}]_{(o_h,a_h,\Omega_{h}),:} \U_{h-1}\trans [\D_{h-1}\trans]_{:,\tau_{h-1}} \\
    & =\U_h\trans \left[\U_{h-1}\U_{h-1}\trans \D_{h-1}\trans \right]_{(o_h,a_h,\Omega_{h}),\tau_{h-1}} \\
    & =\U_h\trans   \left[\D_{h-1}\trans\right]_{(o_h,a_h,\Omega_{h}),\tau_{h-1}} =\U_h\trans   [\D_h\trans]_{:,\tau_h},
\end{align*}
where the second equality follows from $[\mat{A}]_{I,:} [\mat{C}]_{:,J}  =  [\mat{A}\mat{C}]_{I,J}$, the third one uses the fact that $\U_{h-1}\U_{h-1}\trans$ is the projection matrix for the column space of $\D_{h-1}\trans$, and the final one uses the  definition of $\D_{h-1}$ and $\D_h$.
\end{proof}
Given Claim \ref{claim:oom-operator-product}, condition 5 follows naturally:
\begin{align*}
\bup_h\trans \B_h(o_h,a_h)\cdots\B_1(o_1,a_1) \b_0 
 & = \frac{1}{A^{H-h}}\mathbf{1}\trans\U_h\U_h\trans  [\D_h\trans]_{:,\tau_h}\\
 & = \frac{1}{A^{H-h}}\mathbf{1}\trans [\D_h\trans]_{:,\tau_h}\\
 & = \frac{1}{A^{H-h}} \sum_{o_{h+1:H},a_{h+1:H}} \P(o_{1:H} \mid a_{1:H}) = \P(o_{1:h} \mid a_{1:h}).
\end{align*}

It remains to prove condition 4. 
By condition 5 and Claim \ref{claim:oom-operator-product}, $\bup_h$ satisfies the following linear equation:
$$
\bup_h\trans\U_h\trans\D_h\trans = \frac{1}{A^{H-h}}\mathbf{1}\trans\D_h\trans.
$$
Besides, by using the basic fact that  $\sum_{o_{h+1}}\P(o_{1:h+1} \mid a_{1:h+1}) = \P(o_{1:h} \mid a_{1:h})$ for any $a_{h+1}\in\fA$, we  have that for any $a_{h+1}\in\fA$: 
$$
\left(\bup_{h+1}\trans\sum_{o_{h+1}} \B_{h+1}(o_{h+1},a_{h+1}) \right)\U_h\trans\D_h\trans = \frac{1}{A^{H-h}}\mathbf{1}\trans\D_h\trans.
$$
As a result, $\bup_h\trans$ and $\bup_{h+1}\trans\sum_{o_{h+1}} \B_{h+1}(o_{h+1},a_{h+1})$ are both solutions to $\x\trans\U_h\trans\D_h\trans = \frac{1}{A^{H-h}}\mathbf{1}\trans\D_h\trans$. Notice that $\U_h\trans\D_h\trans \in \R^{r_h\times (OA)^h}$ has rank $r_h$, which implies the linear equation has a unique solution. Therefore, we conclude $\bup_h\trans=\bup_{h+1}\trans\sum_{o_{h+1}} \B_{h+1}(o_{h+1},a_{h+1})$.
\end{proof}

The following corollary is a direct consequence of Theorem \ref{thm:oom-represent}. 
\begin{corollary}
\label{cor:oom-represent}
For any sequential decision making problem with $\max_h{\rm rank}(\D_h)\le r$, there exist operators $\b_0\in\R^r$,  $\{\B_h(o,a)\}_{(h,o,a)\in[H]\times\fO\times\fA}$
$\subseteq \R^{r\times r}$ and $\bup_H\in\R^r$ satisfying that   \begin{enumerate}
    \item  $\|\B_h(o,a)\|_2 \le 1$, 
     $\|\b_0\|_2 \le \sqrt{A^H}$, and $\|\bup_H\|_2 \le \sqrt{(O/A)^{H-h}}$.
    \item For any $\tau_H\in \cT_H$, $\P(o_{1:H} \mid a_{1:H})=\bup_H\trans \B_H(o_H,a_H)\cdots\B_1(o_1,a_1) \b_0$.
\end{enumerate}
\end{corollary}
\begin{proof}[Proof of Corollary \ref{cor:oom-represent}]
 In Theorem \ref{thm:oom-represent}, we construct a set of operators that satisfy the above two conditions, the size of which is no larger than $r$ or  $r\times r$. To make their size exactly $r$ or  $r\times r$, we can simply add dummy columns and rows filled by zero.  
\end{proof}

\begin{theorem}[bracketing number of low-rank sequential decision making problems]\label{thm:sdm-cover}
Let $\fM$ be the collections of all rank-$r$ sequential decision making problems with $O$ observations, $A$ actions and episode length $H$, then we have 
$$
\log \cN_\epsilon(\fM) \le \cO\left(r^2OAH^2 \log(rOAH/\epsilon)\right).
$$
\end{theorem}
\begin{proof}[Proof of Theorem \ref{thm:sdm-cover}]
Let $c\in\R^+$ be an absolute constant large enough, and  $\cC_\delta$ be a $\delta$-cover of 
$$
\left\{ \b_0\in\R^r,~  \{\B_h(o,a)\}_{(h,o,a)\in[H]\times\fO\times\fA}\subseteq \R^{r\times r},~ \bup_H\in\R^r:~\text{the two conditions in Corollary \ref{cor:oom-represent}}\right\},
$$
with respect to $\ell_\infty$-norm and 
with $\delta:=\epsilon\cdot(OA)^{-cH}$.

Given a set of operators $\cM:=(\b_0,\{\B_h(o,a)\},\bup_H)\in\cC_\delta$, we define $$
\widehat{\P}_\cM(o_{1:H}\mid a_{1:H}) := \bup_H\trans \B_H(o_H,a_H)\cdots\B_1(o_1,a_1) \b_0 + \epsilon/2.
$$
By simple calculation, one can verify that the collections of all such $\widehat{\P}_\cM$ with $\cM\in\cC_\delta$ constitute  an $\epsilon$-bracketing cover for $\fM$. Finally, we conclude the proof by noting that $$
\log |\cC_\delta| = \cO\left(r^2OAH \log(rOAH/\delta)\right)
= \cO\left(r^2OAH^2 \log(rOAH/\epsilon)\right).
$$
\end{proof}

Since any  rank-$r$ PSR family is a subset of the collections of all rank-$r$ sequential decision making problems, the upper bound in Theorem \ref{thm:sdm-cover} is also a valid upper bound for the bracketing number of any rank-$r$ PSR family, which proves Theorem \ref{thm:tabularPSR-bracket}.


\section{Proofs for Important PSR subclasses}
\label{app:psr-examples}

In this section, we prove the results for  various families of  POMDPs in Section \ref{sec:psr-example}.

\paragraph{POMDP notations.} For the convenience of readers, here we briefly review several important notations for POMDPs in the finite-observation setting:
\begin{itemize}
    \item $\bmu_1\in\R^S$ denotes the  distribution over the initial latent state $s_1$ at step $1$.
    \item $\T_{h,a}\in\R^{S\times S}$ denotes the transition matrix of action $a$ at step $h$, the $(s',s)\th$ entry of which is equal to the probability of transitioning to $s'$ provided that  action $a$ is taken at state $s$ and step $h$.
    \item $\O_h\in\R^{O\times S}$ denotes the observation matrix at step $h$, the $(o,s)\th$ entry of which is equal to the probability of observing $o$ at step $h$ conditioning on the current latent state being $s$. We will also use $\O_h(o\mid \cdot)$ to denote the $o\th$ row of matrix $\O_h$.
    \item $
\{\M_h\in\R^{(A^{m-1}O^m)\times S}\}_{h\in[H-m+1]}
$ define the $m$-step emission-action matrices such that 
for an observation sequence $\o$ of length $m$, initial state $s$ and action sequence $\a$ of length $m-1$,
we have  $[\M_h]_{(\a,\o),s}$ equal to  the probability of receiving  $\o$ provided that the action sequence $\a$ is used from state $s$ and step $h$: 
 \begin{equation*}
     [\M_h]_{(\a,\o),s}= \P(o_{h:h+m-1}=\o \mid 
     s_h=s,a_{h:h+m-2} =\a )\quad \text{for all } (\a,\o)\in \fA^{m-1}\times\fO^m  \text{ and } s\in\fS.
 \end{equation*}
 We remark that $\M$ is determined by $\O$ and $\T$, and 
  we have $\M=\O$ when $m=1$. 
\end{itemize}

\subsection{Proofs for finite-observation observable POMDPs}\label{app:observable-finite}

\begin{proof}[Proof of Theorem \ref{thm:pomdp} with finite observations]

The goal is to construct  PSR representations that satisfy Condition \ref{asp:psr} with $\gamma=\cO(\alpha/(S+A^{m-1}))$.
 The definitions of most notations used below can be found at  the beginning of Appendix \ref{app:psr-examples}.

\paragraph{PSR representation.}
Given a POMDP parameterization  $\theta=(\O,\T,\bmu_1)$ with $m$-step observation-action matrices  $\{\M_h\}_{h\in[H-m+1]}$, we construct the following PSR representations:
\begin{align*} 
\begin{cases}
\bpsi_0 = \M_1 \bmu_1,\\
    \bM_h(o,a) = \M_{h+1} \T_{h,a} \diag(\O_h(o\mid\cdot)) (\M_h^\dagger+\Y_h) \quad \mbox{for } h\in[H-m],\\
    \bM_h(o,a) = [ \mathbf{1}([o,a,\omega_{h}] =\omega_{h-1})  ]_{\omega_{h}\in\Omega_{h},\omega_{h-1}\in\Omega_{h-1}}\in \R^{|\Omega_{h}|\times |\Omega_{h-1}|} \quad \mbox{for}\quad H-m< h<H,\\
     \bphi_H(o,a) =\e_{(o,a)} \in \R^{OA},\\
    \cQ_h= (\fO\times\fA)^{\min\{m-1,H-h\}}\times\fO ~\text{ and }~ \cQ_h^A = \fA^{\min\{m-1,H-h\}},
\end{cases}
\end{align*}
 where  $\Y_h \in \arg\min_{\tilde \Y\M_h = 0} \|\M_h^\dagger + \tilde \Y\|_1$ and $\Omega_h=(\fO\times\fA)^{H-h}$.
 
By repeatedly  applying Bayesian's rule, one can easily verify the above PSR representation  satisfy the basic PSR probabilistic definition under Condition \ref{cond:pomdp}: 
$$
\begin{cases}
\Pb(o_{1:H}|a_{1:H}) = \bphi_H(o_H, a_H)\trans \bM_{H-1}(o_{H-1}, a_{H-1}) \cdots \bM_1(o_1, a_1)\bpsi_0,\\
\Pb(\tau_h,\cQ_h) =  \bM_{h}(o_{h}, a_{h}) \cdots \bM_1(o_1, a_1)\bpsi_0.
\end{cases}
$$

\paragraph{Verify Condition \ref{asp:psr}.} Next, we upper bound $\gamma$ for the above operators.     By direct calculation using the definition, we have 
 \begin{equation*}
     \m_1(\omega_h)\trans = \begin{cases}
         \Pb(\omega_h\mid s_{h+1}=\cdot)\trans   (\M_{h+1}^\dagger+\Y_{h+1} ), &h\in[H-m-1],\\
         \e_{\omega_h}\trans, & \text{otherwise}.
     \end{cases}
 \end{equation*}
 In the first case,  given any $\x\in\R^{O^m A^{m-1}}$
  \begin{equation*}
 \begin{aligned}
 & \sum_{\omega_h\in\Omega_h^{(1)}} | \m_1(\omega_h)\trans \x | \times \pi(\omega_h)\\ 
 = &  \sum_{\omega_h\in\Omega_h^{(1)}} |  \Pb(\omega_h\mid s_{h+1}=\cdot)\trans   (\M_{h+1}^\dagger+\Y_{h+1} )\x  | \times \pi(\omega_h ) \\
  \le & \sum_{\omega_h\in\Omega_h^{(1)}} \sum_{i\in [S]}   \Pb(\omega_h\mid s_{h+1}=i) \times    |\e_i\trans(\M_{h+1}^\dagger+\Y_{h+1} )\x| \times \pi(\omega_h ) \\
  = &   \|(\M_{h+1}^\dagger+\Y_{h+1} )\x \|_1  \le \cO\left(\frac{S}{\alpha}\right)\times \|\x\|_1 ,
 \end{aligned}
 \end{equation*}
 where the final inequality uses Lemma \ref{lem:Sep26} and  \ref{lem:inverse-O}. In the second case,  given any $\x\in\R^{O^{H-h}A^{H-h-1}}$
  \begin{equation*}
 \begin{aligned}
  \sum_{\omega_h\in\Omega_h^{(1)}} | \m_1(\omega_h)\trans \x | \times \pi(\omega_h)=
   \sum_{\omega_h\in\Omega_h^{(1)}} |  \e_{\omega_h}\trans \x  | \times \pi(\omega_h ) \le \|\x\|_1.
  \end{aligned}
 \end{equation*}
 
As for the second set of linear weight vectors, we have  for $h\in[H-m]$, 
  \begin{equation*}
 \begin{aligned}
 & \sum_{\omega_{h-1}\in\Omega_{h-1}^{(2)}} | \m_2(\omega_{h-1})\trans \x | \times \pi(\omega_{h-1})\\ 
 =&\sum_{o_h,a_h} \sum_{\omega_h\in\cQ_h} | \e_{\omega_h}\trans\bM_h(o_h,a_h) \x | \times \pi(o_h,a_h,\omega_h) \\
 = & \sum_{o_h,a_h} \sum_{\omega_h\in\cQ_h} | \e_{\omega_h}\trans \M_{h+1} \T_{h,a_h} \diag(\O_h(o_h\mid\cdot)) (\M_h^\dagger+\Y_h) \x | \times \pi(o_h,a_h,\omega_h) \\
 \le & \sum_{s_h,o_h,a_h}\sum_{\omega_h\in\cQ_h} | \e_{\omega_h}\trans\M_{h+1} \T_{h,a_h} \diag(\O_h(o_h\mid\cdot)) \e_{s_h}| \times | \e_{s_h}\trans  (\M_h^\dagger+\Y_h) \x | \times \pi(o_h,a_h,\omega_h) \\
 =& \sum_{s_h} | \e_{s_h}\trans  (\M_h^\dagger+\Y_h) \x | \le\cO\left( \frac{S}{\alpha}\right)\times  \|\x\|_1,
 \end{aligned}
 \end{equation*}
 where the final inequality uses Lemma \ref{lem:Sep26} and  \ref{lem:inverse-O}. And for $h>H-m$,
   \begin{equation*}
 \begin{aligned}
\sum_{o_h,a_h}\sum_{\omega_h\in\cQ_h} | \e_{\omega_h}\trans\bM_h(o_h,a_h) \x | \times \pi(o_h,a_h,\omega_h)  =    \sum_{o_{h:H},a_{h:H}} |  \e_{(o_{h:H},a_{h:H})}\trans  \x |\times \pi(o_h,a_h,\omega_h)  \le \|\x\|_1.
 \end{aligned}
 \end{equation*}

 \paragraph{Rank of the PSRs.} By Lemma \ref{lem:decodable-linearMDPs}, we have $r\le S$.
\end{proof}

\subsection{Proofs for continuous-observation observable POMDPs}\label{app:observable-continuous}

We first set up some basic notations for continuous POMDPs. 
Let $\O_h(\cdot \mid i)$ be the density function of $o_h$ conditioning on $s_h=i$. 
We will use $f(o_{1:h}\mid a_{1:h})$ to denote the density at $o_{1:h}$ conditioning on taking a fixed action sequence $a_{1:h}$.

To begin with, we prove the following lemma which at a high level states that for any $\alpha$-observable POMDP family  with finite covering number, we can group its observations in a way so that the induced meta-POMDPs not only have a  finite (but potentially arbitrarily large, e.g., double exponentially many)  meta-observation space  but  also are $\alpha/2$-observable. 
\begin{lemma}[Observation grouping]\label{lem:discrete}
Given any $\alpha$-observable POMDP class $\Theta$ with finite covering number at precision level  $\alpha/8$, there exists  $n\in\N^+$ and $\phi:\fO\rightarrow\{0,1\}^n$ so that  $\Theta_\phi:=\{\theta_\phi:~\theta\in\Theta\}$ is $\alpha/2$-observable, where $\theta_\phi$ denotes a meta-POMDP generated by applying $\phi$ to model $\theta$ such that 
\begin{itemize}
    \item the observation space of $\theta_\phi$ is  $\fO_\phi=\{\phi(o):~o\in\fO\}$,
    \item the observation probability in $\theta_\phi$ is defined as  $\O_{\theta_\phi,h}(y\mid s) = \int_{o\in\fO} \O_{\theta,h}(o\mid s) \mathbf{1}(\phi(o)=y) do$ for $y\in\fO_\phi$~,
    \item the underlying MDP part of $\theta_\phi$ is the same as $\theta$.
\end{itemize} 
\end{lemma}
\begin{proof}[Proof of Lemma \ref{lem:discrete}]
We prove the lemma for $m=1$. The case of $m>1$ follows almost the same with trivial modification. 
First, let us consider a fixed POMDP model $\theta\in\Theta$ and $h\in[H]$. 
Let $\cC_\delta$ be a $\delta$-cover of $\Delta_S$ which is generated in the following way: for every $k\in[S]$, let  $\cC_{\delta,k}\subset(\Delta_S\cap\{\bnu:~\|\bnu\|_0= k\})$ be a $\delta$-cover of $\Delta_S\cap\{\bnu:~\|\bnu\|_0= k\}$, and then define $\cC_{\delta}=\cup_{k\in[S]}\cC_{\delta,k}$.

For any $\bnu_1,\bnu_2\in\cC_\delta$, by using the $\alpha$-observable condition and the definition of TV distance, there exists $\phi_{\bnu_1,\bnu_2}:\fO\rightarrow\{0,1\}$ so that 
$$
d_{\rm TV}(\phi_{\bnu_1,\bnu_2}(\O_h\bnu_1),\phi_{\bnu_1,\bnu_2}(\O_h\bnu_2)) \ge \alpha \times 
d_{\rm TV}(\bnu_1,\bnu_2)
$$
where $\phi_{\bnu_1,\bnu_2}(\O_h\bnu_i)$ denotes the distribution over $\phi(o)$ with   $o$ sampled from $\O_h\bnu_i$. 
Now we define $\phi_{\theta}:\fO\rightarrow\{0,1\}^{H|\cC_\delta|^2}$ to be a vector of binary functions where each entry corresponds to one binary function  $\phi_{\bnu_1,\bnu_2}$ with $\bnu_1,\bnu_2\in\cC_\delta$ at some step $h\in[H]$. By the definition of TV distance, we have: for any $\bnu_1,\bnu_2\in\cC_\delta$ and $h$:
$$
d_{\rm TV}(\phi_{\theta}(\O_h\bnu_1),\phi_{\theta}(\O_h\bnu_2)) \ge \alpha \times 
d_{\rm TV}(\bnu_1,\bnu_2).
$$
Since $\cC_\delta$ is a $\delta$-cover of $\Delta_S$, we further have  
for any $\bnu_1,\bnu_2\in\Delta_S$ and $h$:
$$
d_{\rm TV}(\phi_{\theta}(\O_h\bnu_1),\phi_{\theta}(\O_h\bnu_2)) \ge \alpha \times 
d_{\rm TV}(\bnu_1,\bnu_2)-2\delta.
$$
Now, let us go one-step further by defining  $\phi:\fO\rightarrow\{0,1\}^{H|\cC_\delta|^2|\Theta_\delta|}$  to be a stack of binary vector functions where each entry corresponds to one  $\phi_{\theta}$ with $\theta\in\Theta_\delta$ that is a $\delta$-cover of model class $\Theta$. Therefore, we have that for any observation matrix $\O_h$ from model class $\Theta_\delta$, any $\bnu_1,\bnu_2\in\Delta_S$ and $h$:
$$
d_{\rm TV}(\phi(\O_h\bnu_1),\phi(\O_h\bnu_2)) \ge \alpha \times 
d_{\rm TV}(\bnu_1,\bnu_2)-2\delta.
$$
Since $\Theta_\delta$ is a $\delta$-cover of $\Theta$, we can pay another $2\delta$ to replace $\Theta_\delta$ with $\Theta$ in the above guarantee, which gives that for any observation matrix $\O_h$ from model class $\Theta$, any $\bnu_1,\bnu_2\in\Delta_S$ and $h$:
$$
d_{\rm TV}(\phi(\O_h\bnu_1),\phi(\O_h\bnu_2)) \ge \alpha \times 
d_{\rm TV}(\bnu_1,\bnu_2)-4\delta.
$$

In particular, when  $\bnu_1$ and $\bnu_2$ are disjoint, we further have
$$
d_{\rm TV}(\phi(\O_h\bnu_1),\phi(\O_h\bnu_2)) \ge \alpha -4\delta.
$$
By invoking Lemma \ref{lem:observable-disjoint} and choosing $\delta<\alpha/8$ , we conclude the entire proof.
\end{proof}

Now are ready to prove Theorem \ref{thm:pomdp} by designing well-conditioned PSR representations with the above grouping technique and the $\ell_1$-norm matrix inverse technique (Lemma \ref{lem:inverse-O}).

\subsubsection{Proof of Theorem \ref{thm:pomdp} with continuous observation.}

Recall in Appendix \ref{sec:psr_cts} we introduced two approaches to handling PSRs with continuous observations. Below we will use a mixture of the two approaches to design well-conditioned PSR representations for observable POMDPs:
\begin{itemize}
    \item For step $h<H-m$, we will use Lemma \ref{lem:discrete} to design general core tests which are sets of trajectories. 
    Specifically, we will consider general core tests of form 
    $$
    \mathbf{1}\bigg((\phi(o_{h+1}),a_{h+1},\ldots,\phi(o_{h+m}))=q\bigg) \quad \text{ for }\quad  q\in\cQ_h:=(\fO_\phi \times \fA)^{m-1}\times\fO_\phi,
    $$
    where $\phi$ is the  mapping constructed in Lemma \ref{lem:discrete} and $\fO_\phi=\{\phi(o):~o\in\fO\}$ is the induced finite  meta-observation space.
    
    \item For step $h\ge H-m$, we choose $\cQ_h =(\fO\times\fA)^{H-m}$ and construct PSR operators that are  linear  operators in the function space.
\end{itemize}

\paragraph{PSR representation.}
Given a POMDP model  $\theta=(\O,\T,\bmu_1)$, we denote by $\M_h:~\R^S\rightarrow \R^{(\fO\times \fA)^{m-1}\times\fO}$ its $m$-step  observation-action probability function such that $\M_h \x = \sum_{s=1}^S \x_s \times \M_{h,s}$, where $\R^{(\fO\times \fA)^{m-1}\times\fO}$ denotes the collections of all real-valued functions  over $(\fO\times \fA)^{m-1}\times\fO$. We further define the $m$-step  observation-action probability matrices of its  $\phi$-induced meta-POMDP $\theta_\phi$ by $\{\D_h\}_{h\in[H-m+1]}$, which are matrices because the number of observations in $\theta_\phi$ is finite.
Now we construct the following PSR representation:
\begin{align*} 
\begin{cases}
\bpsi_0 = \D_1 \bnu_1,\\
    \bM_h(o,a) = \D_{h+1} \T_{h,a} \diag(\O_h(o\mid\cdot)) (\D_h^\dagger+\Y_h) \quad \mbox{for } h<H-m,\\
      \bM_h(o,a) = \M_{h+1} \T_{h,a} \diag(\O_h(o\mid\cdot)) (\D_h^\dagger+\Y_h) \quad \mbox{for } h=H-m,\\
    \bM_h(o,a) = [ \mathbf{1}([o,a,\omega_{h}] =\omega_{h-1})  ]_{\omega_{h}\in\Omega_{h},\omega_{h-1}\in\Omega_{h-1}}\quad \mbox{for}\quad H-m< h<H,\\
     \bphi_H(o,a) =\e_{(o,a)},\\
     \cQ_h = (\fO_\phi\times \fA)^{m-1}\times\fO, \quad  \mbox{for } h< H-m  \\
      \cQ_h = (\fO\times \fA)^{H-h}, \quad  \mbox{for } h\ge H-m
\end{cases}
\end{align*}
where 
 \begin{itemize}
 \item $\bpsi_0$ is a finite-dimensional vector,
 \item  for $h\le H-m$, $\bM_h(o,a)$ is a finite-dimensional matrix,  and  $\Y_h \in \arg\min_{\tilde \Y\D_h = 0} \|\D_h^\dagger + \tilde \Y\|_1$,
 \item  for $h=H-m$,  $\bM_h(o,a)$ is a linear operator that maps a function defined on  $\cQ_{h-1}$ to a function defined on $\Omega_h=(\fO\times\fA)^{m}$, 
 \item for $h>H-m$,  $\bM_h(o,a)$ is a linear operator that maps a function defined on  $\Omega_{h-1}$ to a function defined on $\Omega_{h}$ so that for any $\omega_h\in\Omega_h$ and function $g$, $[\bM_h(o,a) g](\omega_h) = g([o,a,\omega_h])$, 
    \item $\bphi_H(o,a)$ is a linear operator that maps a function defined on $\fO\times\fA$ to its value at $(o,a)$.
 \end{itemize}  
 
By direct calculation using the above definition  and Lemma \ref{lem:discrete}, one can easily verify the above operators satisfy the probabilistic definition of continuous PSR in Appendix \ref{sec:psr_cts}, under Condition \ref{cond:pomdp}:
\begin{itemize}
    \item for any $\tau_H$
    $$
    f(o_{1:H}\mid  a_{1:H}) =  \bphi_H(o_H, a_H)\trans \bM_{H-1}(o_{H-1}, a_{H-1}) \cdots \bM_1(o_1, a_1)\bpsi_0,
    $$
    \item for any $h\le H-m$,  $\tau_h=(o,a)_{1:h}$ 
and $q=(y_{h+1},a_{h+1},\ldots,y_{h+m})\in\cQ_h$
$$
f(o_{1:h}\mid  a_{1:h})\times  \P(\phi(o_{h+1:h+m})=y_{h+1:h+m} \mid \tau_h,a_{h+m-1}) =  \e_q\trans \bM_{h}(o_{h}, a_{h}) \cdots \bM_1(o_1, a_1)\bpsi_0,
$$
\item for any $h> H-m$,   $\tau_h=(o,a)_{1:h}$ and $q=(o_{h+1},a_{h+1},\ldots,o_{h+m})\in\cQ_h$
$$
f(o_{1:h}\mid  a_{1:h})\times f(o_{h+1:H}\mid o_{1:h},a_{1:H-1})  =  [\bM_{h}(o_{h}, a_{h}) \cdots \bM_1(o_1, a_1)\bpsi_0](q).
$$
\end{itemize}

\paragraph{Verify Condition \ref{asp:psr-continuous}.}
Now we calculate $\gamma$ for the above representation. 
For any $h<H-m$ and  $\z\in\R^{|\fO_\phi|^m A^{m-1}}$, let $\omega_h=(o_{h+1},a_{h+1},\ldots,o_H,a_H) \in (\fO\times\fA)^{H-h}$
  \begin{equation*}
 \begin{aligned}
 & \sum_{a_{h+1:H}} \int_{o_{h+1:H}} | \m_1(\omega_h)\trans \z | \times \pi(\omega_h ) ~d o_{h+1:H}\\ 
 = &  \sum_{a_{h+1:H}} \int_{o_{h+1:H}}  |  f(o_{h+1:H}\mid s_{h+1}=\cdot,a_{h+1:H})\trans   (\D_{h+1}^\dagger+\Y_{h+1} )\z  | \times \pi(\omega_h ) ~d o_{h+1:H}\\
  \le & \sum_{a_{h+1:H}} \int_{o_{h+1:H}}  \sum_{i\in [S]}   f(o_{h+1:H}\mid s_{h+1}=i,a_{h+1:H}) \times    |\e_i\trans(\D_{h+1}^\dagger+\Y_{h+1} )\z| \times \pi(\omega_h) ~d o_{h+1:H}\\
  = &   \|(\D_{h+1}^\dagger+\Y_{h+1} )\z \|_1  \le \cO\left(\frac{S}{\alpha}\right)\times \|\z\|_1 ,
 \end{aligned}
 \end{equation*}
 where the final inequality uses Lemma \ref{lem:Sep26} and  \ref{lem:inverse-O}.
 And for any $h\ge H-m$ and function $g:\Omega_h\rightarrow\R$ 
 \begin{equation*}
 \begin{aligned}
  \sum_{a_{h+1:H}} \int_{o_{h+1:H}} | \m_1(\omega_h)\trans g | \times \pi(\omega_h )~d o_{h+1:H}
 = \sum_{a_{h+1:H}} \int_{o_{h+1:H}} | g(\omega_h)  | \times \pi(\omega_h )  ~d o_{h+1:H}\le \|g\|_1.
 \end{aligned}
 \end{equation*}

As for the second set of linear weight vectors: for $h<H-m$, 
  \begin{equation*}
 \begin{aligned}
 &\int_{\omega_{h-1}\in{\fO\times\fA\times \cQ_{h}}}   \pi(o_h,a_h) \cdot |\m_2(\omega_{h-1})\trans\z|  ~d\omega_{h-1} \\
  = &\int_{o_h}\sum_{a_h} \sum_{q\in\cQ_{h}} | \e_{q}\trans\bM_h(o_h,a_h) \z | \times \pi(o_h,a_h)~ d o_h \\
 = &\int_{o_h}\sum_{a_h} \sum_{q\in\cQ_{h}}  | \e_{q}\trans \D_{h+1} \T_{h,a_h} \diag(\O_h(o_h\mid\cdot)) (\D_h^\dagger+\Y_h) \z | \times \pi(o_h,a_h) d o_h \\
 \le & \int_{o_h}\sum_{s_h}\sum_{a_h} \sum_{q\in\cQ_{h}}   | \e_{q}\trans\D_{h+1} \T_{h,a_h} \diag(\O_h(o_h\mid\cdot)) \e_{s_h}| \times | \e_{s_h}\trans  (\D_h^\dagger+\Y_h) \z | \times \pi(o_h,a_h) ~do_h \\
 =& A^{m-1}\sum_{s_h} | \e_{s_h}\trans  (\D_h^\dagger+\Y_h) \z | \le\cO\left( \frac{A^{m-1}S}{\alpha}\right) \|\z\|_1,
 \end{aligned}
 \end{equation*}
 where the final inequality uses Lemma \ref{lem:Sep26} and  \ref{lem:inverse-O}. For $h=H-m$, the argument follows almost the same as above except we replace the summation over $q\in\cQ_{h}$ with summation over $a_{h+1:H}\in\fA^{H-h}$ and integral over $o_{h+1:H}\in\fO^{H-h}$. 
Finally, for $h>H-m$ and function $g:\Omega_{h-1}\rightarrow\R$, let $\omega_{h-1}=(o_{h},a_{h},\ldots,o_H,a_H) \in (\fO\times\fA)^{H-h+1}$
   \begin{equation*}
 \begin{aligned}
   &\sum_{a_{h:H}} \int_{o_{h:H}}  \left| \left[\bM_h(o_h,a_h) g\right](o_{h+1:H},a_{h+1:H}) \right| \times \pi(o_h,a_h) ~do_{h:H} \\
  = &   \sum_{a_{h:H}} \int_{o_{h:H}}  |g(\omega_{h-1})| \times \pi(o_h,a_h)  ~do_{h:H} \le \|g\|_1.
 \end{aligned}
 \end{equation*}
 
 \paragraph{Rank of the PSRs.} By Lemma \ref{lem:decodable-linearMDPs}, we have $r\le S$.


\subsection{Proofs for GM-POMDPs}
\label{app:gm-pomdps}

We first prove two lemmas, which state that $\eta$-separable GM-POMDPs are $\eta/2$-observable and have well controlled log-bracketing number. 

\begin{lemma}\label{lem:gmpomdps-alpha}
Any $\eta$-separable GM-POMDPs satisfy the observable condition (Condition \ref{cond:pomdp}) with $m=1$ and $\alpha=\Omega(\eta)$.
\end{lemma}
\begin{proof}[Proof of Lemma \ref{lem:gmpomdps-alpha}]
WLOG, we only need to verify the observable condition for a fixed $h$, so we will omit the subscript $h$ in the following proof for simplicity of notations. For the same reason, we will abbreviate $\E_{x\sim {\rm Gauss}(\x_{i},\sigma^2\cdot \I_{d\times d})}=[f(x)]$ as $\E_i f$. Denote by $p_i$ the density function of ${\rm Gauss}(\x_{i},\sigma^2\cdot \I_{d\times d})$.

By simple calculation using  the tail bound for standard  Gaussian and the first inequality in  Condition \ref{cond:GM-pomdp}, we have 
$$
d_{\rm TV}\left({\rm Gauss}(\x_{i},\sigma^2\cdot\I_{d\times d}),
{\rm Gauss}(\x_{j},\sigma^2\cdot\I_{d\times d})
\right)\ge 1-\frac{1}{4d}.
$$
For each $i\in[n]$, let $f_{i}(\z):=\mathbf{1}(p_i(\z) = \max_{j\in[n]} p_j(\z))$.
By the definition of TV distance and the above TV lower bound,  we have 
$\E_{i}[f_i]\ge 3/4$ and 
$\E_{j}[f_i]\le 1/(4d)$ for $j\neq i$. 
Given two Gaussian mixtures with weights $\a=\Wbb_h \bnu_1$ and $\b =\Wbb_h \bnu_2 \in\Delta_n$ respectively, we have 
\begin{equation*}
    \begin{aligned}
    & d_{\rm TV}\left(\sum_i a_i {\rm Gauss}(\x_{i},\sigma_h^2\cdot\I_{d\times d}),
\sum_i b_i{\rm Gauss}(\x_{i},\sigma_h^2\cdot\I_{d\times d})
\right) \\
\ge & \sum_i a_i  \E_i \left[ \sum_{j:~a_j >b_j} f_j\right] - \sum_i b_i  \E_i \left[ \sum_{j:~a_j >b_j} f_j\right]\\
=  & \sum_{i:~a_i >b_i}  |a_i-b_i| \times  \E_i \left[ \sum_{j:~a_j >b_j} f_j\right] - \sum_{i:~a_i < b_i}  |a_i-b_i|\times   \E_i \left[ \sum_{j:~a_j >b_j} f_j\right]  \\
\ge & \frac34 \sum_{i:~a_i >b_i}  |a_i-b_i| - 
\frac14 \sum_{i:~a_i <b_i}  |a_i-b_i|  \\ 
= & \frac14  \sum_{i}  |a_i-b_i| 
\ge  \frac{\eta}{2} d_{\rm TV}\left(\bnu_1,\bnu_2\right),
    \end{aligned}
\end{equation*}
where the first inequality uses the definition of TV distance and the fact that $\sum_{j:~a_j >b_j} f_j$ is  binary almost surely when Condition \ref{cond:GM-pomdp} holds, the second inequality uses the property of $f_i$, and the final inequality uses the second inequality in  Condition \ref{cond:GM-pomdp}.
\end{proof}

\begin{lemma}\label{lem:gmpomdps-bracket}
Given model class
 $$
\Theta:=\left\{ \left(\T,\Wbb,\{(\x_{h,i},\sigma_{h}^2\cdot \I_{d\times d})\}_{h,i},\bmu_1\right):~~\eta\text{-separable, }\|\x_{h,i}\|_2 \le C_x \text{ and }  \Csl \le \sigma_h \le \Csu\right\},
$$
we have 
$$
\log\cN_\Theta(\omega) \le \cO \bigg( 
H(S^2 A+ Sn)\log(SAHn/\omega) + H\log(dH(\Csu/\Csl)/\omega)+
Hnd\log (dH (C_x/\Csu)/\omega)
\bigg). 
$$
\end{lemma}
\begin{proof}[Proof of Lemma \ref{lem:gmpomdps-bracket}]
Let $\eta$, $\epsilon$, $t$ and $\delta$ be discretizatin parameters to be determined later. 

Given a GM-POMDP parameterized by $\left(\T,\Wbb,\{(\x_{h,i},\sigma_{h}^2\cdot \I_{d\times d})\}_{h,i},\bmu_1\right)$. 
Denote by $\widehat\T$ the ``discretized'' transition such that $[\widehat{\T}_{h,a}]_{i,j} = \lceil [{\T}_{h,a}]_{i,j}\cdot \eta^{-1}\rceil \cdot \eta$, where $[{\T}_{h,a}]_{i,j}$ denotes the entry at position $(i,j)$ of transition matrix ${\T}_{h,a}$. Similarly, we define $\{\widehat{\Wbb}_h\}_h$ and $\widehat{\bmu}_1$. Importantly, we are performing upper bounding here so these ``discretized'' parameters ($\widehat{\T}$, $\widehat{\Wbb}$, and $\widehat{\bmu}_1$) are entry-wise upper bounds for the original parameters. 

It remains to show how to discretize the density function of each base Gaussian component. Given a $d$-dimensional Gaussian distribution with mean $\x$ and covariance matrix $\sigma^2\cdot \I_{d\times d}$ from $\Theta$, we consider the following ``discretized density function'':
\begin{equation*}
    \widehat{p}(\z): =\frac{1}{\sqrt{(2\pi)^d \underline{\sigma}^{2d}}} \exp\left(-\frac{(1-t)\|\z- \widehat{\x}\|^2_2 -(1+1/t)\epsilon^2   }{2\overline{\sigma}^2} \right),
\end{equation*}
where $\underline{\sigma} = \lfloor\sigma\cdot\delta^{-1}\rfloor\cdot\delta$, $\overline{\sigma} = \lceil\sigma\cdot\delta^{-1}\rceil\cdot\delta$, and $\widehat{\x} =\argmin_{\y \in \mathcal{C}_\epsilon} \| \y -\x\|_2$ where $\mathcal{C}_\epsilon$ is an $\epsilon$-cover of $B_{C_x}(\mathbf{0})$ w.r.t. $\ell_2$-norm. 
First, observe that $\widehat{p}$ is an entry-wise upper bound for the original density function because $\underline{\sigma} \le \sigma \le \overline{\sigma}$ and 
\begin{equation*}
\begin{aligned}
     \|\z- \widehat{\x}\|^2_2 - \|\z- {\x}\|^2_2 
= & \langle  \widehat{\x} -\x , 2(\z- \widehat{\x}) + \widehat{\x}- {\x}\rangle\\ 
\le & t\| \z- \widehat{\x}\|_2^2 + \frac{\epsilon^2}{t}+ \epsilon^2.
\end{aligned}
\end{equation*}

Beside, we also have 
\begin{equation*}
    \begin{aligned}
     &\quad \int_{\z} | \widehat{p}(\z) - p(\z) | d\z \\
     & =  \int_{\z}  \widehat{p}(\z) d\z  - 1 \\
     & = \int_\z \frac{1}{\sqrt{(2\pi)^d \underline{\sigma}^{2d}}} \exp\left(-\frac{(1-t)\|\z- \widehat{\x}\|^2_2 -(1+1/t)\epsilon^2   }{2\overline{\sigma}^2} \right) -1 \\
     & = \exp\left(\frac{(1+1/t)\epsilon^2}{2\overline{\sigma}^2}\right)\cdot \sqrt{\frac{\overline{\sigma}^{2d}}{(1-t)^d\underline{\sigma}^{2d}}}\int_\z \sqrt{\frac{(1-t)^d}{(2\pi)^d \overline{\sigma}^{2d}}} \exp\left(-\frac{(1-t)\|\z- \widehat{\x}\|^2_2   }{2\overline{\sigma}^2} \right) -1 \\
     &= \exp\left(\frac{(1+1/t)\epsilon^2}{2\overline{\sigma}^2}\right)\cdot \sqrt{\frac{\overline{\sigma}^{2d}}{(1-t)^d\underline{\sigma}^{2d}}}-1.
    \end{aligned}
\end{equation*}
Now let us pick 
\begin{equation*}
   \begin{cases}
\eta =  \omega /\poly(S,A,H,n), \\
 t = \omega /\poly(d,H), \\
\epsilon = \omega\times\Csu/\poly(H,t^{-1}),\\
\delta = \omega \times  \Csl / \poly(d,H).
\end{cases} 
\end{equation*}
With slight abuse of notations, 
let $\|\X-\Y\|_{1}$ denote the entry-wise $\ell_1$ distance between two tensors $\X$ and $\Y$ of the same size. 
Under the above choice of parameters, it is straightforward to verify 
\begin{equation}\label{eq:sep12-1}
    \begin{cases}
    \| \widehat\Wbb - \Wbb\|_1,~\|\widehat\T- \T\|_1,~\|\widehat\bmu_1 - \bmu_1\|_1 \le \omega / \poly(H),\\
    d_{\rm TV} \left( {\rm Gauss}(\x_{h,i},\sigma_{h}^2\cdot \I_{d\times d}), \widehat{p}_{h,i}\right) \le \omega  / \poly(H) \text{ for all }h,i,
    \end{cases}
\end{equation}
where $\widehat{p}_{h,i}$ is the discretized version of ${\rm Gauss}(\x_{h,i},\sigma_{h}^2\cdot \I_{d\times d})$. 

Now that we have  $\widehat\theta= \left(\widehat\T,\widehat\Wbb,\{\widehat p_{h,i}\}_{h,i},\widehat\bmu_1\right)$ which is a discretized entry-wise upper bound for $\theta$, we can naturally construct an upper bound for  ${\P}_\theta$, defined as 
\begin{equation*}
    \begin{aligned}
     \widehat{\P}_\theta(o_{1:H}\mid a_{1:H}):= 
&\sum_{(s,i)_{1:H}} \widehat\bmu_1(s_1)\times \widehat\Wbb_1(i_1 \mid s_1) \times \widehat p_{1,i_1}(o_1) \times \widehat\T_{1,a_1}(s_2\mid s_1) \\
&\times \cdots \times \widehat\Wbb_{H}(i_H \mid s_H) \times \widehat p_{H,i_H}(o_H).
    \end{aligned}
\end{equation*}
Moreover, by basic algebra, one can  verify that 
\eqref{eq:sep12-1} implies that for any policy  $\pi$, 
$$
\sum_{\tau_H} \left| \widehat{\P}_\theta(o_{1:H}\mid a_{1:H})  - {\P}_\theta(o_{1:H}\mid a_{1:H}) \right| \times \pi(\tau_H)\le \omega. 
$$
As a result, by counting the number of all possible  $\widehat\theta$, we conclude that the log-bracketing number of $\Theta$ is at most 
$$
\cO \bigg( 
H(S^2 A+ Sn)\log(SAHn/\omega) + H\log(dH(\Csu/\Csl)/\omega)+
Hnd\log (dH (C_x/\Csu)/\omega)
\bigg). 
$$
\end{proof}

\begin{proof}[Proof of Proposition \ref{prop:gauss-pomdp}]
Proposition \ref{prop:gauss-pomdp} follows immediately from combining Lemma \ref{lem:gmpomdps-alpha} and \ref{lem:gmpomdps-bracket} with  Theorem \ref{thm:pomdp}.
\end{proof}

\subsection{Proof for POMDPs with few core action sequences}

To avoid repetitive arguments,  we only prove Theorem \ref{thm:pomdp-small-core-action-set} for the finite-observation case. The generalization to the continuous-observation case follows basically the same arguments as in the proof for continuous observable POMDPs (in the same way as  Appendix \ref{app:observable-continuous} generalizes Appendix \ref{app:observable-finite}).

\begin{proof}[Proof of Theorem \ref{thm:pomdp-small-core-action-set}]
Given $h\in[H]$ and $\a\in\cA_h$ of length $l\le H-h$, denote by $\K_{h}(\cdot,\a)$ an $O^{l+1}\times S$ matrix so that its entry at position $(\o,i)\in\fO^{l+1}\times\fS$ is equal to the probability of observing $o_{h:h+l}$ provided that we execute action sequence $\a$  from state $s$ and step $h$, i.e., 
$$
[\K_{h}(\cdot,\a)]_{\o,i} := \P_\theta (o_{h:h+l}=\o \mid s_h =i, a_{h:h+l-1}=\a).
$$
We further define $\K_h$ to be an $(\sum_{\a\in\cA_h} O^{|\a|+1})\times S$ matrix by stacking all the $\K_h(\cdot,\a)$'s with $\a\in\cA_h$. Formally,  if  $\cA_h=\{\a^{(1)}\}_{i=1}^{|\cA_h|}$, then 
$$
\K_h:=[\K_h(\cdot,\a^{(1)});\ldots;\K_h(\cdot,\a^{(|\cA_h|)})].
$$
Notice that by Condition \ref{cond:pomdp-small-core-action-set}, we have that for any $\bnu_1,\bnu_2\in\Delta_S$, 
$$
\| \K_h(\bnu_1-\bnu_2)\|_1 \ge \alpha \|\bnu_1-\bnu_2\|_1.
$$

Given $\theta=(\T,\O,\bmu_1)$ and its corresponding $\{\K_h\}_h$,   we construct the following PSR representations:
\begin{align*} 
\begin{cases}
\bpsi_0 = \K_1 \bmu_1,\\
    \bM_h(o,a) = \K_{h+1} \T_{h,a} \diag(\O_h(o\mid\cdot)) (\K_h^\dagger+\Y_h) \quad \mbox{for } h\in[H-1],\\
     \bphi_H(o,a) =\e_{(o,a)} \in \R^{OA},\\
     \cQ_h^A = \cA_h,
\end{cases}
\end{align*}
 where  $\Y_h \in \arg\min_{\tilde \Y\K_h = 0} \|\K_h^\dagger + \tilde \Y\|_1$. Observe that the above operators  are almost the same as those constructed in the proof of Theorem \ref{thm:pomdp} in Appendix \ref{app:observable-finite} except that we replace the $m$-step observation-action matrix $\M_h$ by $\K_h$.  
One can easily verify that by following exactly the same arguments as in Appendix \ref{app:observable-finite}, the above PSR representations are $\cO(S/\alpha)$-well conditioned.
\end{proof}

\subsection{Proofs for multi-step decodable POMDPs}
\label{appsub:m-decodable}

\begin{lemma}\label{lem:decodable-no-revealing}
There exist $1$-step observable POMDPs that are not multi-step decodable for any $m$ and there exist $1$-step decodable POMDPs that are not observable for any $m$ and $\alpha$. 
\end{lemma}
\begin{proof}[Proof of Lemma \ref{lem:decodable-no-revealing}]
Consider a  POMDP with $O=2$, $S=2$, $A=1$,  $\O_h=[0.99,0.01;0.01,0.99]$, $\T_h=\I_{2\times 2}$ and $\mu_1 = [0.5,0.5]$. This POMDP is clearly $1$-step $0.5$-observable and not decodable for any $m$. 
Consider another POMDP with $O=2$, $S=2$, $A=2$, $\O_h=[1,1;0,0]$, $\T_{h,1}=[1,1;0,0]$,  $\T_{h,2}=[0,0;1,1]$ and $\mu_1=[1,0]$. Note that in this POMDP, the transition is not only deterministic but also always satisfies   $s_1=1$ and $s_h=a_{h-1}$ for $h>1$. Therefore, it is $1$-step decodable. However, since the observation distribution is independent of the latent states, this POMDP is not observable for any $\alpha$ and $m$. 
\end{proof}

To avoid repetitive arguments,  we only prove Theorem \ref{thm:decodable-pomdp} for the finite-observation case. The generalization to the continuous-observation case is quite straightforward: 
we simply follow the second approach proposed in Appendix \ref{sec:psr_cts} to  replace all the discrete probability with density,  summations with integrals, vectors with (conditional) probabilistic functions and  matrices with linear operators in the following proofs.

\begin{proof}[Proof of Theorem \ref{thm:decodable-pomdp}]
It suffices to show  $m$-step decodable POMDPs are $1$-well-conditioned PSRs. 
Then Theorem \ref{thm:decodable-pomdp} follows immediately from the guarantee for well-conditioned PSRs (Theorem \ref{thm:psr}). WLOG, assume $m\le H$.

\paragraph{PSR representation.}
Given a POMDP $\theta=(\O,\T,\mu_1)$ that is  $m$-step decodable, we choose  the core action set at step $h$ to consist of all action sequences of length $\min\{H-h,m\}$. Therefore the core tests at step $h$, that is $\cQ_h$, is equal to $
(\fO\times\fA)^{\min\{H-h,m\}}$.
We construct the following PSR operators:
\begin{itemize}
    \item $\bpsi_0$ is a vector of dimension $|\cQ_0|$, the $(o,a)_{1:m}$'th entry of which is equal to $\Pb_\theta((o,a)_{1:m})$.
    \item $\bM_h(o_h,a_H)$ is an $|\cQ_h|$ by $|\cQ_{h-1}|$ matrix and its entry at position $(\omega_{h},\omega_{h-1})\in \cQ_h \times \cQ_{h-1}$ is equal to 
    $$
    \begin{cases}
    \Pb\big( (\omega_h)_{m} \mid \omega_{h-1} \big) \times \mathbf{1}\big([o_h,a_h,(\omega_h)_{1:m-1}] = \omega_{h-1} \big), \quad  &h \le H-m, \\
    \mathbf{1}\big([o_h,a_h,\omega_h] = \omega_{h-1} \big), \quad  &h > H-m,
    \end{cases}
    $$
    where $(\omega_h)_{i:j}$  denotes  the $i^{\rm th}$ to the $j^{\rm th}$ observation-action pairs in $\omega_h$, and $(\omega_h)_{t}$ denotes the  $t^{\rm th}$ observation-action pair in $\omega_h$. 
    \item $\bphi_H(o,a) =\e_{(o,a)} \in \R^{OA}$.
\end{itemize}
By direct calculation, one can easily verify the above operators constitute a legal PSR representation, in the sense that they jointly satisfy the PSR probabilistic definition under Condition \ref{cond:decodable-pomdp}: 
$$
\begin{cases}
\Pb(o_{1:H}|a_{1:H}) = \bphi_H(o_H, a_H)\trans \bM_{H-1}(o_{H-1}, a_{H-1}) \cdots \bM_1(o_1, a_1)\bpsi_0,\\
\Pb(\tau_h,\cQ_h) =  \bM_{h}(o_{h}, a_{h}) \cdots \bM_1(o_1, a_1)\bpsi_0.
\end{cases}
$$
\paragraph{Verify Condition \ref{asp:psr}.} Next, we show  the above operators constitute a well-conditioned PSR  representation. By direct calculation, we have 
$$
\m_1((o,a)_{h+1:H}) = 
\begin{cases}
\e_{(o,a)_{h+1:h+m}} \times \Pb\big( (o,a)_{h+m+1:H} \mid(o,a)_{h+1:h+m}\big),  \quad  &h \le H-m,\\
\e_{(o,a)_{h+1:H}}, \quad  &h > H-m,
\end{cases}
$$
where 
$\Pb\big( (o,a)_{h+m+1:H} \mid(o,a)_{h+1:h+m}\big) = \prod_{i=h+m+1}^H \P(o_i\mid (o,a)_{h+1:i-1})$.

Therefore, for any $\omega_{h}\in \cQ_h$ and policy $\pi$ independent of the history before step $h+1$, we have that for $h\le H-m$
\begin{equation*}
\begin{aligned}
 & \sum_{(o,a)_{h+1:H}} | \m_1((o,a)_{h+1:H} )\trans \e_{\omega_h} | \times \pi((o,a)_{h+1:H})\\ 
 = &\sum_{(o,a)_{h+1:H}} | \e_{(o,a)_{h+1:h+m} }\trans \e_{\omega_h} |  \times \Pb\big( (o,a)_{h+m+1:H} \mid (o,a)_{h+1:h+m}\big)\times \pi((o,a)_{h+1:H})\\ 
 = &\sum_{(o,a)_{h+m+1:H}}   \Pb\big( (o,a)_{h+m+1:H} \mid (o,a)_{h+1:h+m}=\omega_h \big)\times \pi((o,a)_{h+1:H}) \\
 = & \pi(\omega_h) \le 1
\end{aligned}
\end{equation*}
and for $h>H-m$ we trivially have 
\begin{equation*}
\begin{aligned}
\sum_{(o,a)_{h+1:H}} | \m_1((o,a)_{h+1:H} )\trans \e_{\omega_h} | \times \pi((o,a)_{h+1:H})=
\pi(\omega_h) \le 1.
\end{aligned}
\end{equation*}

It remains to verify Condition \ref{asp:psr} for the second set of linear weight vectors $\{\m_2(\omega_h)\}_{\omega_h\in\Omega_h^{(2)}}$. By definition, we directly have 
$$
\m_2((o,a)_{h+1:h+m+1}) = 
\begin{cases}
\e_{(o,a)_{h+1:h+m}} \times \Pb\big( (o,a)_{h+m+1} \mid(o,a)_{h+1:h+m}\big),  \quad  &h < H-m,\\
\e_{(o,a)_{h+1:h+H}}, \quad  &h \ge H-m.
\end{cases}
$$
So we only need to  deal with the case of $h< H-m$. 
And for $h< H-m$,   by following similar arguments for handling $\m_1$,  we have 
\begin{equation*}
\begin{aligned}
& \sum_{q_i\in\cQ_{h+1}}\sum_{(o,a)_{h+1}} \left| \m_2((o,a)_{h+1},q_i) \trans \e_{\omega_h}\right|\times \pi((o,a)_{h+1},q_i) \\
 =&\sum_{(o,a)_{h+1:h+m+1}} \left| \m_2((o,a)_{h+1:h+m+1}) \trans \e_{\omega_h}\right|\times \pi((o,a)_{h+1:h+m+1}) \\
= &\sum_{(o,a)_{h+1:h+m+1}} | \e_{(o,a)_{h+1:h+m} }\trans \e_{\omega_h} |  \times \Pb\big( (o,a)_{h+m+1} \mid (o,a)_{h+1:h+m}\big)\times \pi((o,a)_{h+1:h+m+1})\\
 =& \sum_{(o,a)_{h+m+1}}   \Pb\big( (o,a)_{h+m+1} \mid (o,a)_{h+1:h+m}=\omega_h \big)\times \pi((o,a)_{h+1:h+m+1}) \\
 = & \pi(\omega_h) \le 1.
\end{aligned}
\end{equation*}
\paragraph{Rank of the PSRs.} By Lemma \ref{lem:decodable-linearMDPs}, we have $r\le \min\{S,\dlin\}$.
\end{proof}

\begin{lemma}\label{lem:decodable-linearMDPs}
We have $r\le S$ in any POMDP. 
Furthermore, if the  underlying MDP of a  POMDP can be  represented by a $\dlin$-dimensional linear kernel MDP, then we have $\max_h \rank(\D_h)\le \dlin$.
\end{lemma}
\begin{proof}[Proof of Lemma \ref{lem:decodable-linearMDPs}]
Recall   in kernel linear MDPs \citep{yang2020reinforcement}, the transition functions can be represented as a linear function  of the tensor product of two feature mappings. Formally, there exist  $\phi:\fS\times\fA\rightarrow \R^\dlin$ and $\psi:\fS\rightarrow \R^\dlin$ so that for any $h\in[H]$, there exists $\W_h\in \R^{\dlin \times \dlin}$ satisfying $\T_{h,a_h}(s_{h+1}\mid s_h)= \phi(s_h,a_h)\trans \W_h \psi(s_{h+1})$ for all $(s_h,a_h,s_{h+1})\in\fS\times\fA\times\fS$. 
Therefore, we have 
$$
{\rm dim}\left(\bigcup_{a\in\cA}{\rm column\_span}(\T_{h,a})\right)
\le 
\min\left\{{\rm dim}\left({\rm span}\{\psi(s):~s\in\fS\}\right), S\right\} \le \min\{\dlin,S\}.
$$
To relate the LHS to the rank of the system-dynamic matrices, we notice that 
$$
\Pb((o,a)_{1:H})
= \mathbf{1}\trans\left(\prod_{h'=h+1}^H\T_{h',a_{h'}}\diag(\O_{h'}(o_{h'}\mid \cdot))\right)
\left(\prod_{h'=1}^h\T_{h',a_{h'}}\diag(\O_{h'}(o_{h'}\mid \cdot))\right)\mu_1,
$$
which implies 
\begin{align*}
\rank(\D_h)& \le\dim\left({\rm span}\left\{\left(\prod_{h'=1}^h\T_{h',a_{h'}}\diag(\O_{h'}(o_{h'}\mid \cdot))\right)\mu_1:~(o,a)_{1:h}\in\cT_h\right\} \right) \\
& \le {\rm dim}\left(\bigcup_{a\in\cA}{\rm column\_span}(\T_{h,a})\right)\le \min\{\dlin,S\}.
\end{align*}
\end{proof}


\section{Proofs for Beyond Low-rank Sequential Decision Making}
\label{app:beyond-low-rank}

\subsection{Proofs for \SAIL}

\begin{proof}[Proof of Proposition \ref{prop:salp-relation} (generality of \SAIL)]
The proof of part (a) follows directly from Lemma \ref{lem:psr-sail}. 

For part (b), let $\Theta$ be the family of factored MDPs which (i) consist of $n$ factors and respect the factorization structure $\pa_i=\{i\}$ for all $i\in[n]$, (b) has $2$ actions, i.e., $\fA=[2]$, (c) has state space $\fS=[2]^n$, and (d) has episode length $H=n$. By Proposition \ref{prop:factored_MDPs} and  Proposition \ref{prop:witness-salp}, $\Theta$ satisfies the \SAIL\  condition holds with $d,\kappa,B=\cO(n)$ and $\Pie(\pi)=\pi$. However, for a special $\theta\in\Theta$ wherein the transition is deterministic and satisfies $s_{h+1} =(s_h[1:h-1], a_h,s_h[h+1:H])$ (i.e., $s_h=(a_{1:h-1},s_{0}[h:H+1])$ ,  the rank of the system dynamics matrix at step $H-1$ is equal to $2^{H-2}$.
\end{proof}

\begin{proof}[Proof of Theorem \ref{thm:salp} (main theorem)]
By Lemma \ref{lem:sail-eluder}, the \SAIL\ condition implies that the generalized eluder-type condition holds  with 
\begin{equation}
\begin{cases}
    d_\Theta =  {  \kappa^2 d^2 |\Pie| \poly(H)},\\  \xi(d_\Theta,\Delta,|\Pie|,K) = \log^2(K)\left( \sqrt{d_\Theta \Delta |\Pie| K} +d B    \poly(H)\right).
    \end{cases}
\end{equation}
Plugging the above quantities back into Theorem \ref{thm:omle} completes the proof.
\end{proof}

\begin{proof}[Proof of Lemma \ref{lem:sail-eluder}(\SAIL$\rightarrow$eluder)]
The proof follows directly from  invoking Proposition \ref{prop:step-3}.
\end{proof}

\begin{proof}[Proof of Theorem \ref{thm:salp-sharper} (sharper guarantee for $m=n=1$)]
First we verify the generalized eluder-type condition. 
By the precondition of Condition \ref{cond:eluder} and Cauchy-Schwarz inequality, we have that for any $h\in[0,H]$  
$$
  \sum_{t=1}^{k-1}\left |\langle  g_{h}(\theta^k), f_{h}(\theta^t) \rangle \right|^2 
\le \kappa^2  |\Pie|\Delta.
$$
Therefore, by the standard elliptical potential arguments, we have 
    \begin{equation*}
    \begin{aligned}
        \sum_{t=1}^{k} d_\TV( \P_{\theta^t}^{\pi^t} , \P_{\theta^\star}^{\pi^t})  
        \le  \sum_{h=0}^H \sum_{t=1}^{k}\left |\langle  f_{h}(\theta^t), g_{h}(\theta^t) \rangle \right| =\tilde{\cO }\left(  \left(d B  + \kappa \sqrt{ d \Delta |\Pie| k} \right)H\right),
    \end{aligned}
    \end{equation*}
    for all  $k\in[K]$. Plugging the above quantities back into Theorem \ref{thm:omle} completes the proof.
\end{proof}

\subsection{Proofs for examples of \SAIL}

\subsubsection{Witness rank}
\begin{proof}[Proof of Proposition \ref{prop:witness-salp} (witness rank)]
By triangle inequality, we have that in both Q-type and V-type witness condition
\begin{equation*}
    d_{\rm TV}(\P^{\pi_\theta}_{\theta^\star},\P^{\pi_\theta}_{\theta} )  \le \frac{1}{2} \sum_{h=1}^H \E_{(s_h,a_h)\sim \P_{\theta^\star}^{\pi_\theta}} 
    \left[\|\D_{\theta}(s_h,a_h) - \D_{\theta^\star}( s_h,a_h)     \|_1 \right]
     \le  \frac{1}{2} \sum_{h=1}^H\langle f_h(\theta),g_h(\theta)\rangle.
\end{equation*}
On the other hand, for Q-type witness condition
\begin{equation*}
\begin{aligned}
    \frac{1}{2} \sum_{h=1}^H\langle f_h(\theta),g_h(\theta')\rangle 
    &\le \frac{\kappa}{2}\sum_{h=1}^H\E_{(s_h,a_h)\sim \P^{\pi_\theta}_{\theta^\star} }\left[\|\D_{\theta'}(s_h,a_h) - \D_{\theta^\star}( s_h,a_h)     \|_1 \right]  \\
    &\le \frac{\kappa}{2}\sum_{h=1}^H \sum_{s_h,a_h,s_{h+1}} \bigg(|(\P^{\pi_\theta}_{\theta^\star}-\P^{\pi_\theta}_{\theta'})(s_h,a_h,s_{h+1})| \\
    &\quad + |(\P^{\pi_\theta}_{\theta^\star}-\P^{\pi_\theta}_{\theta'})(s_h,a_h)\times\P_{\theta'}(s_{h+1}\mid s_h,a_h)|\bigg) \\
    &\le 2\kappa \sum_{h=1}^H d_{\rm TV}(\P^{\pi_\theta}_{\theta^\star},\P^{\pi_\theta}_{\theta'} ),
\end{aligned}
\end{equation*}
and similarly for V-type witness condition: 
\begin{equation*}
\begin{aligned}
    \frac{1}{2} \sum_{h=1}^H\langle f_h(\theta),g_h(\theta')\rangle 
    &\le \frac{\kappa}{2}\sum_{h=1}^H\E_{s_h\sim \P^{\pi_\theta}_{\theta^\star},~a_h\sim \pi_{\theta'}(s_h) }\left[\|\D_{\theta'}(s_h,a_h) - \D_{\theta^\star}( s_h,a_h)     \|_1 \right]  \\
    &\le \frac{A\kappa}{2}\sum_{h=1}^H\E_{s_h\sim \P^{\pi_\theta}_{\theta^\star},~a_h\sim {\rm Uniform}(\fA) }\left[\|\D_{\theta'}(s_h,a_h) - \D_{\theta^\star}( s_h,a_h)     \|_1 \right]  \\
    & \le {2A\kappa}  \sum_{\tilde{\pi}\in\Pie (\pi_\theta)} d_{\rm TV}(\P^{\tilde\pi}_{\theta^\star},\P^{\tilde\pi}_{\theta'} ),
\end{aligned}
\end{equation*}
where $\Pie(\pi)=\{\pi_{1:h}\circ{\rm Uniform}(\fA):~h\in[H]\}$.
\end{proof}

\subsubsection{Factored MDPs}
\begin{proof}[Proof of Proposition \ref{prop:factored_MDPs} (factored MDPs)]
We will show $d=A \sum_{i=1}^m |\cX|^{|\pa_i|}$.  The two feature mappings are constructed in the following way:
\begin{itemize}
    \item The coordinate of $f_h(\theta)$ at position $(\hat a,i)\in\fA\times[m]$ and $z\in\cX^{|\pa_i|}$ is equal to $$\P_{\theta^\star}^{\pi_\theta}(s_{h}[\pa_i]=z,a_{h}=\hat a).$$
    \item The coordinate of $g_h({\theta'})$ at position $(\hat a,i)\in\fA\times[m]$ and $z\in\cX^{|\pa_i|}$ is equal to $$\left\| \P_{\theta'}^i (s_{h+1}[i]=\cdot \mid s_h[\pa_i]=z,a_h=\hat a) -  \P_{\theta^\star}^i (s_{h+1}[i]=\cdot   \mid s_h[\pa_i]=z,a_h=\hat a)    \right\|_1.$$
\end{itemize}
By simple calculation, we have  $\|f_h(\theta)\|_1 \le 1$ and $\|g_h(\theta')\|_1 \le \sum_{i=1}^m |\cX|^{|\pa_i|}$.

Now we verify that the above mappings satisfy the definition of witness rank: by triangle inequality 
\begin{equation*}
\begin{aligned}
     &\E_{(s_h,a_h)\sim \P^{\pi_\theta}_{\theta^\star} }\left[   \| \P_{\theta'}(s_{h+1}=\cdot\mid s_h,a_h) -  \P_{\theta^\star}(s_{h+1}=\cdot\mid s_h,a_h)\|_1   \right]\\
       = & \E_{(s_h,a_h)\sim \P^{\pi_\theta}_{\theta^\star} } \sum_{s_{h+1}}  \left| \prod_{i=1}^m \P_{\theta'}^i (s_{h+1}[i] \mid s_h[\pa_i],a_h) -  \prod_{i=1}^m \P_{\theta^\star}^i (s_{h+1}[i]  \mid s_h[\pa_i],a_h)    \right|\\
       \le &  \E_{(s_h,a_h)\sim \P^{\pi_\theta}_{\theta^\star} } \sum_{i=1}^m   \left\| \P_{\theta'}^i (s_{h+1}[i]=\cdot \mid s_h[\pa_i],a_h) -  \P_{\theta^\star}^i (s_{h+1}[i]=\cdot   \mid s_h[\pa_i],a_h)    \right\|_1\\
       = &   \langle f_h(\theta), g_h(\theta')\rangle.
      \end{aligned}
\end{equation*}
On the other hand,  
\begin{equation*}
\begin{aligned}
& \langle f_h(\theta), g_h(\theta')\rangle \\
= & \E_{(s_h,a_h)\sim \P^{\pi_\theta}_{\theta^\star} } \sum_{i=1}^m   \left\| \P_{\theta'}^i (s_{h+1}[i]=\cdot \mid s_h[\pa_i],a_h) -  \P_{\theta^\star}^i (s_{h+1}[i]=\cdot   \mid s_h[\pa_i],a_h)    \right\|_1\\
\le & m \E_{(s_h,a_h)\sim \P^{\pi_\theta}_{\theta^\star} }\left[   \| \P_{\theta'}(s_{h+1}=\cdot\mid s_h,a_h) -  \P_{\theta^\star}(s_{h+1}=\cdot\mid s_h,a_h)\|_1   \right].
\end{aligned}
\end{equation*}
Therefore, $\kappa=m$.
\end{proof}

\subsubsection{Linear kernel MDPs}

\begin{proof}[Proof of Proposition \ref{prop:linearMDPs} (linear kernel MDPs)] By the definition of linear kernel  MDPs, for $h>1$
\begin{equation*}
\begin{aligned}
     &\E_{s_h\sim \P^{\pi_\theta}_{\theta^\star},~a_h\sim\pi_{\theta'}(s_h) }\left[   \| \P_{\theta'}(s_{h+1}=\cdot\mid s_h,a_h) -  \P_{\theta^\star}(s_{h+1}=\cdot\mid s_h,a_h)\|_1   \right]\\
     = & \bigg\langle \E_{(s_{h-1},a_{h-1})\sim \P^{\pi_\theta}_{\theta^\star} }\left[ \phi(s_{h-1},a_{h-1})\right],\\
     &\sum_{s_h} \W_{h-1,\theta^\star}\psi(s_h) \E_{a_h\sim\pi_{\theta'}(s_h)} \| \P_{\theta'}(s_{h+1}=\cdot\mid s_h,a_h) -  \P_{\theta^\star}(s_{h+1}=\cdot\mid s_h,a_h)\|_1   \bigg\rangle\\
       := &   \langle f_h(\theta), g_h(\theta')\rangle.
      \end{aligned}
\end{equation*}
Using the normalization condition of linear MDPs, we have 
$\|f_h \|_1 \le C_\phi$, $\|g_h(\theta')\|_1 \le 2\sqrt{\dlin}C_W C_\psi$.
And for $h=1$, we can simply choose 
$$
\begin{cases}
f_h(\theta):=1 \\
g_h(\theta'):=\E_{s_1\sim \mu_1,~a_1\sim\pi_{\theta'}(s_1) }\left[   \| \P_{\theta'}(s_2=\cdot\mid s_1,a_1) -  \P_{\theta^\star}(s_2=\cdot\mid s_1,a_1)\|_1   \right],
\end{cases}
$$
which satisfies $f_1(\theta)=1$ and $\ g_h(\theta') \le 2$.
\end{proof}

\subsubsection{Sparse linear bandits}
\begin{proof}[Proof of Proposition \ref{prop:linear-bandits-witness} (sparse linear bandits)]
We construct the two feature mappings in the following way:
\begin{equation*}
f(\theta):=a_\theta, \quad g(\theta'):=2(\theta'-\theta^\star).
\end{equation*} 
Now we verify that the above mappings satisfy the definition of witness rank. Since the reward follows Bernoulli distribution, 
$$
\E_{s\sim\mu, a\sim\pi_\theta(s)} \left\| \P_{\theta'}( r=\cdot \mid a) - \P_{\theta^\star}( r=\cdot \mid a)\right\|_1 = |\langle f(\theta), g(\theta')\rangle|. 
$$
Therefore, we have $d=\dlin$, $\kappa = 1$ and $B= 4\sqrt{\dlin} C_\Theta C_\fA$.
\end{proof}

\section{Proofs for Reward-free  OMLE}\label{app:reward-free}

\subsection{Proof of Theorem \ref{thm:salp-rf}}
Denote 
$$
(\pi^k,\theta^k,\tilde\theta^k) \leftarrow \argmax_{\pi,~{\theta},\tilde{\theta}\in\cB^k} d_{\rm TV}(\P_{\theta}^\pi,\P_{\tilde\theta}^\pi).
$$
Since (a) $\theta^{\rm out}\in\cB^K$ and $\cB^k$ is monotonically shrinking with respect to $k$,   and (b) with probability at least $1-\delta$,  $\theta^\star\in\Theta^k$ for all $k\in[K]$ by Proposition \ref{prop:mle-optimisim}, we have 
\begin{align*}
    \max_{\pi} d_{\rm TV}(\P_{\theta^{\rm out}}^\pi,\P_{\theta^\star}^\pi)
    \le \frac{1}{K} \sum_{k=1}^K 
    d_{\rm TV}(\P_{\theta^{k}}^{\pi^k},\P_{\tilde\theta^k}^{\pi^k})
    \le \frac{1}{K} \sum_{k=1}^K 
    d_{\rm TV}(\P_{\theta^{k}}^{\pi^k},\P_{\theta^\star}^{\pi^k})+ \frac{1}{K} \sum_{k=1}^K 
    d_{\rm TV}(\P_{\theta^{\star}}^{\pi^k},\P_{\tilde\theta^k}^{\pi^k})
    .
\end{align*}
Below we show how to control the first term and the second term can be controlled in the same way. 
By using the strong SAIL condition, we have 
$$
\sum_{k=1}^K 
    d_{\rm TV}(\P_{\theta^{\star}}^{\pi^k},\P_{\theta^k}^{\pi^k})
    \le 
    \sum_{k=1}^K
    \sum_{h=1}^H \sum_{i=1}^m \sum_{j=1}^n |\langle f_{h,i}(\pi^k),g_{h,j}(\theta^k)\rangle|.
$$
On the other hand, by using the strong SAIL condition and Proposition \ref{prop:mle-valid} with $\theta^k\in\cB^k$, we have that with probability at least $1-\delta$,  for all $k\in[K]$: 
\begin{align*}
  \sum_{t=1}^{k-1}
     \sum_{h=1}^H \sum_{i=1}^m \sum_{j=1}^n |\langle f_{h,i}(\pi^t),g_{h,j}(\theta^k)\rangle| 
    \le \kappa\sum_{t=1}^{k-1} \sum_{\tilde{\pi}\in\Pie (\pi^t)} d_{\rm TV}(\P^{\tilde\pi}_{\theta^\star},\P^{\tilde\pi}_{\theta^k} ) \le \cO\left(\kappa\sqrt{k\beta |\Pie|}\right).
\end{align*}
As a result, we can invoke Lemma \ref{prop:pigeon-hole} with the above inequality and the normalization condition, which gives:
$$
\sum_{k=1}^K
    \sum_{h=1}^H \sum_{i=1}^m \sum_{j=1}^n |\langle f_{h,i}(\pi^k),g_{h,j}(\theta^k)\rangle|
\le     \poly(H) d \left( {B}  + \kappa \sqrt{{\beta|\Pie|}{ K}} \right) \log^2(K),
$$
which concludes the proof.

\subsection{Proofs for examples of strong SAIL}

Below we verify the strong SAIL conditions for all examples discussed in Section \ref{sec:beyond-low-rank}.

\paragraph{Well-conditioned PSRs.} The proof follows directly from the three-step proof for well-conditioned PSRs in Appendix \ref{app:psr-main}.

\paragraph{Factored MDPs.} The proof follows basically the same as the
proof of Proposition \ref{prop:factored_MDPs}. The only modification needed is to replace greedy policy $\pi_\theta$ with general policy $\pi$.

\paragraph{Kernel linear MDPs.} 
By triangle inequality, we have 
\begin{align*}
    \sum_{\tau_H}|\P_\theta^\pi(\tau_H) - \P_{\theta^\star}^\pi(\tau_H)| 
    &\le \sum_{h=1}^H 
    \sum_{s_h,a_h,s_{h+1}} \P^\pi_{\theta^\star}(s_h,a_h) \times | \P_\theta(s_{h+1}\mid s_h,a_h)-\P_{\theta^\star}(s_{h+1}\mid s_h,a_h)| \\
    & =\sum_{h=1}^H 
    \sum_{s_h,a_h,s_{h+1}} \left| \langle \P^\pi_{\theta^\star}(s_h,a_h) \phi(s_h,a_h)\psi(s_{h+1})\trans, \W_{h,\theta} - \W_{h,\theta^\star}\rangle\right|.
\end{align*}
On the other hand, we have 
\begin{align*}
   & \sum_{h=1}^H 
    \sum_{s_h,a_h,s_{h+1}} \left| \langle \P^\pi_{\theta^\star}(s_h,a_h) \phi(s_h,a_h)\psi(s_{h+1})\trans, \W_{h,\theta} - \W_{h,\theta^\star}\rangle\right|\\
    = & \sum_{h=1}^H 
    \sum_{s_h,a_h,s_{h+1}} \P^\pi_{\theta^\star}(s_h,a_h) \times | \P_\theta(s_{h+1}\mid s_h,a_h)-\P_{\theta^\star}(s_{h+1}\mid s_h,a_h)| \\
    \le & 2H^2   \sum_{\tau_H}|\P_\theta^\pi(\tau_H) - \P_{\theta^\star}^\pi(\tau_H)|.
\end{align*}
Therefore, we can construct feature mappings 
$$
\begin{cases}
f_{h,(s,a,s')}(\pi) = 
\text{vect}(\P^\pi_{\theta^\star}(s_h,a_h) \phi(s_h,a_h)\psi(s_{h+1})\trans)\in\R^{d^2} \text{ where } (s,a,s')\in\fS\times\fA\times\fS,\\
g_{h}(\theta) =  \text{vect}(\W_{h,\theta} - \W_{h,\theta^\star}).
\end{cases}
$$

\paragraph{Sparse linear bandits.} 
Since sparse linear bandit can be viewed as a single-state single-step MDP with no  transition dynamics, it trivially satisfies the strong SAIL condition.

\section{Novel $\ell_1$-norm Observation-free Techniques}

\subsection{$\ell_1$-norm elliptical potential arguments}
\label{app:l1-pigeon-hole}

\begin{lemma}\label{prop:pigeon-hole}
    Suppose $\{p^k\}_{k\in[K]},~\{q^k\}_{k\in[K]}$ are two sets of distributions  $\R^N$ that satisfy 
    \begin{align*}
    \begin{cases}
    \sum_{t=1}^{k-1} \E_{x\sim p^k,~y\sim q^t} \left[ \left| \left\langle x , y \right\rangle \right| \right] \le   \zeta_k  	\\
     \E_{x\sim p^k} \left[ \|x\|_\infty \right] \le R_x \\
     \E_{y\sim q^k} \left[ \| y \|_1 \right] \le R_y \\
     \end{cases}
    \quad  \mbox{ for all } k\in[K].
    \end{align*}
    and all $\{q^k\}_{k\in[K]}$ are supported on a common $d$-dimensional linear subspace of $\R^N$. 
    Then we have 
    \begin{equation*}
        \sum_{t=1}^{k} \E_{x\sim p^t, y\sim q^t} \left[ \left| \left\langle x , y \right\rangle \right| \right]  =\cO \left( d\log^2(k)\left(R_x R_y  + \max_{t\le k} \zeta_t\right)\right)
    \quad  \mbox{ for all } k\in[K].
    \end{equation*}
    \end{lemma}

\begin{proof}[Proof of Lemma \ref{prop:pigeon-hole}]
Our objective is to control the value of the following optimization problem:
\begin{equation*}
\begin{aligned}
    & \max_{\{p^t\}_{t=1}^k,\{q^t\}_{t=1}^k } \sum_{t=1}^{k} \E_{x\sim p^t, y\sim q^t} \left[ \left| \left\langle x , y \right\rangle \right| \right] \\
    \text{s.t. }& 
   \begin{cases}
    \sum_{i=1}^{t-1} \E_{x\sim p^t,~y\sim q^i} \left[ \left| \left\langle x , y \right\rangle \right| \right] \le   \max_{t\le k}\zeta_t  	\\
    \E_{x\sim p^k} \left[ \|x\|_\infty \right]  \le R_x \\
     \E_{y\sim q^t} \left[ \| y \|_1 \right] \le R_y \\
     \end{cases} 
     \text{ for all } t\in[k].
\end{aligned}
\end{equation*}
By merging the first two constraints, it suffices to control the value of the following optimization problem:
\begin{equation}\label{opt:1}
\begin{aligned}
    & \max_{\{p^t\}_{t=1}^k,\{q^t\}_{t=1}^k } \sum_{t=1}^{k} \E_{x\sim p^t, y\sim q^t} \left[ \left| \left\langle x , y \right\rangle \right| \right] \\
    \text{s.t. }& 
   \begin{cases}
    \sum_{i=1}^{t-1} \E_{x\sim p^t,~y\sim q^i} \left[ \left| \left\langle x , y \right\rangle \right| \right] + \lambda \E_{x\sim p^t} \left[ \|x\|_\infty \right]  \le \lambda R_x  +\max_{t\le k}\zeta_t  	\\
     \E_{y\sim q^t} \left[ \| y \|_1 \right] \le R_y \\
     \end{cases} 
     \text{ for all } t\in[k],
\end{aligned}
\end{equation}
where $\lambda>0$ can be arbitrarily chosen.

To proceed, we introduce the following auxiliary function $f(\{q^t\}_{t=1}^k,\gamma)$ defined as 
\begin{equation}\label{eq:aux-f}
\begin{aligned}
 f(\{q^i\}_{i=1}^{t},\gamma):= \qquad  & \max_{p }  \E_{x\sim p, y\sim q^t} \left[ \left| \left\langle x , y \right\rangle \right| \right] \\
    \text{s.t. }& 
    \sum_{i=1}^{t-1} \E_{x\sim p,~y\sim q^i} \left[ \left| \left\langle x , y \right\rangle \right| \right] + \lambda  \E_{x\sim p} \left[ \left\| x \right\|_\infty \right] \le \gamma.
\end{aligned}
\end{equation}
We make the following key observation about $f$:
\begin{claim}\label{lem:reduce-1}
For any distributions $\{q^t\}_{t=1}^k$ over $\R^d$ s.t. $ \E_{y\sim q^t} \left[ \| y \|_1 \right] \le R_y $ for all $t\in[k]$, and $\gamma>0$
\begin{equation}\label{eq:aux-f-equiv}
\begin{aligned} 
f(\{q^t\}_{i=1}^t,\gamma) \text{ is equal to } \quad & \max_{x\in \R^d}  \E_{ y\sim q^t} \left[ \left| \left\langle x , y \right\rangle \right| \right] \\
    \text{s.t. }& 
    \sum_{i=1}^{t-1} \E_{y\sim q^i} \left[ \left| \left\langle x , y \right\rangle \right| \right] + \lambda \|x\|_\infty \le \gamma.
\end{aligned}
 \end{equation}
\end{claim}
\begin{proof}[Proof of Claim \ref{lem:reduce-1}]
It is straightforward to see when we replace a distribution by a point, the value of the optimization problem will not increase, because a single point is equivalent to a Dirac distribution. So we only need to show this change will also not decrease the value of the optimization problem. 

Since $ \E_{y\sim q^t} \left[ \| y \|_1 \right] \le R_y $ for all $t\in[k]$,  the maximum of problem \eqref{eq:aux-f-equiv} is always achieved at some $x^\star \neq 0$. 
Besides, it is easy to see $x^\star$ always ``fully'' utilize the constraint budget $\gamma$, i.e.,
\begin{equation}\label{eq:aug9-1}
    {\sum_{i=1}^{t-1} \E_{y\sim q^i} \left[ \left| \left\langle x^\star , y \right\rangle \right| \right] + \lambda \|x^\star\|_\infty} = \gamma,
\end{equation}
since otherwise we can scale up $x^\star$ to improve the objective value.
Moreover, by the optimality of $x^\star$, we also have that for any $x\neq 0$:
\begin{equation}\label{eq:aug9-2}
\frac{\E_{ y\sim q^t} \left[ \left| \left\langle x^\star , y \right\rangle \right| \right] }{\sum_{i=1}^{t-1} \E_{y\sim q^i} \left[ \left| \left\langle x^\star , y \right\rangle \right| \right] + \lambda \|x^\star\|_\infty} \ge \frac{\E_{ y\sim q^t} \left[ \left| \left\langle x , y \right\rangle \right| \right] }{\sum_{i=1}^{t-1} \E_{y\sim q^i} \left[ \left| \left\langle x , y \right\rangle \right| \right] + \lambda \|x\|_\infty},
\end{equation}
because otherwise the following feasible solution $\hat{x}$ has strictly larger objective value than $x^\star$
\begin{equation*}
    \hat{x} := x \times \frac{\sum_{i=1}^{t-1} \E_{y\sim q^i} \left[ \left| \left\langle x^\star , y \right\rangle \right| \right] + \lambda \|x^\star\|_\infty}{\sum_{i=1}^{t-1} \E_{y\sim q^i} \left[ \left| \left\langle x , y \right\rangle \right| \right] + \lambda \|x\|_\infty}.
\end{equation*}
Suppose the maximizer of Problem \eqref{eq:aux-f} is achieved at $p^\star$. By Equation \eqref{eq:aug9-1} and  \eqref{eq:aug9-2},   we have 
\begin{equation*}
    \begin{aligned}
     f(\{q^i\}_{i=1}^{t},\gamma) &= \E_{x\sim p^\star, y\sim q^t} \left[ \left| \left\langle x , y \right\rangle \right| \right] \\
        & \le \E_{x\sim p^\star, y\sim q^t} \left[ \left| \left\langle x^\star  , y \right\rangle \right|  \times  \frac{\sum_{i=1}^{t-1} \E_{y'\sim q^i} \left[ \left| \left\langle x , y' \right\rangle \right| \right] + \lambda \|x\|_\infty}{\sum_{i=1}^{t-1} \E_{y'\sim q^i} \left[ \left| \left\langle x^\star , y' \right\rangle \right| \right] + \lambda \|x^\star\|_\infty} \right] \\
        & = \E_{ y\sim q^t}\left[ \left| \left\langle x^\star  , y \right\rangle \right| \right] \times \frac{\sum_{i=1}^{t-1} \E_{x\sim p^\star,~y'\sim q^i} \left[ \left| \left\langle x , y' \right\rangle \right|  + \lambda \|x\|_\infty\right]}{\sum_{i=1}^{t-1} \E_{y\sim q^i} \left[ \left| \left\langle x^\star , y \right\rangle \right| \right] + \lambda \|x^\star\|_\infty} \\
        & = \E_{ y\sim q^t}\left[ \left| \left\langle x^\star  , y \right\rangle \right| \right] \times \frac{\sum_{i=1}^{t-1} \E_{x\sim p^\star,~y'\sim q^i} \left[ \left| \left\langle x , y' \right\rangle \right|  + \lambda \|x\|_\infty\right]}{\gamma} \\
        &\le \E_{ y\sim q^t}\left[ \left| \left\langle x^\star  , y \right\rangle \right| \right],
    \end{aligned}
\end{equation*}
which concludes the proof.
\end{proof}

Using the notation of $f$, Problem \eqref{opt:1} can be equivalently written as
\begin{equation*}
    \max_{\{q^t\}_{t=1}^k } \sum_{t=1}^{k}  f(\{q^i\}_{i=1}^t,\gamma)  
    \quad \text{s.t. }
   \E_{y\sim q^t} \left[ \| y \|_1 \right] \le R_y 
     \text{ for all } t\in[k],
\end{equation*}
where $\gamma:= \lambda R_x  +\max_{t\le k}\zeta_t$.
By Claim \ref{lem:reduce-1}, it is further equivalent to  
\begin{equation}\label{opt:2}
\begin{aligned}
    & \max_{\{x^t\}_{t=1}^k,\{q^t\}_{t=1}^k } \sum_{t=1}^{k} \E_{ y\sim q^t} \left[ \left| \left\langle x^t , y \right\rangle \right| \right] \\
    \text{s.t. }& 
   \begin{cases}
    \sum_{i=1}^{t-1} \E_{y\sim q^i} \left[ \left| \left\langle x^t , y \right\rangle \right| \right] + \lambda 
    \|x^t\|_\infty  \le \gamma 	\\
     \E_{y\sim q^t} \left[ \| y \|_1 \right] \le R_y \\
     \end{cases} 
     \text{ for all } t\in[k].
\end{aligned}
\end{equation}
Now, we claim Problem \eqref{opt:2} is further equivalent to the following simpler optimization problem where we again replace distributions with points:
\begin{equation}\label{opt:3}
\begin{aligned}
    & \max_{\{x^t\}_{t=1}^k,\{y^t\}_{t=1}^k } \sum_{t=1}^{k}   \left| \left\langle x^t , y^t \right\rangle \right|  \\
    \text{s.t. }& 
   \begin{cases}
    \sum_{i=1}^{t-1}  \left| \left\langle x^t , y^i \right\rangle \right|  + \lambda \|x^t\|_\infty  \le \gamma 	\\
      \| y^t \|_2  \le R_y \\
     \end{cases} 
     \text{ for all } t\in[k].
\end{aligned}
\end{equation}
To see why, suppose $\{{x^\star}^t\}_{t=1}^k,\{{q^\star}^t\}_{t=1}^k$ is a maximizer of  Problem \eqref{opt:2}, then it is direct to verify the same choice of $\{{x^\star}^t\}_{t=1}^k$ and the following choice of $\{{y^\star}^t\}_{t=1}^k$ is a feasible solution to Problem \eqref{opt:3}:
\begin{equation*}
    {y^\star}^t = \E_{y\sim {q^\star}^t }\left[ y \times \text{sign}(y\trans  {x^\star}^t)\right].
\end{equation*}
Moreover, we have 
\begin{equation*}
    \sum_{t=1}^{k}   \left| \left\langle {x^\star}^t , {y^\star}^t \right\rangle \right| =  \sum_{t=1}^{k}   \left| \left\langle {x^\star}^t , \E_{y\sim {q^\star}^t }\left[ y \times \text{sign}(y\trans  {x^\star}^t)\right] \right\rangle \right| = 
    \sum_{t=1}^{k}   \E_{y\sim {q^\star}^t }\left[ \left| \left\langle {x^\star}^t ,  y  \right\rangle \right|\right].
\end{equation*}
Therefore, we conclude that the value of Problem \eqref{opt:3} is no smaller than Problem \eqref{opt:2}. On the other hand, the value of Problem \eqref{opt:3} is trivially upper bounded by that of  Problem \eqref{opt:2} because a single point can be represented as a Dirac distribution. As a result, we reduce the proof of Proposition \ref{prop:pigeon-hole} to upper bound the value of Problem \eqref{opt:3}.

Now, we can invoke Proposition 21 from \cite{liu2022partially} for $d$-dimensional linear function class, and  upper bound the optimal value of Problem \eqref{opt:3} by
$$
\min_{\omega \le C} \left\{ d_E(\omega) C + d_E(\omega) \gamma \log(C/\omega) + k\omega \right\}
$$
where 
$$
\begin{cases}
d_E(\omega) = d\log(\gamma R_y/(\lambda\omega)) \quad \text{ by Lemma \ref{prop:l2-eluder-linear} }\\
C =  \gamma R_y/\lambda \\
\gamma= \lambda R_x  +\max_{t\le k}\zeta_t.
\end{cases}
$$
By choosing $\omega = R_w R_x/k$ and $\lambda=R_y$, we obtain the following final upper bound for Proposition \ref{prop:pigeon-hole}:
$$
\cO \left( d\log^2(k)\left(R_x R_y  + \max_{t\le k} \zeta_k\right)\right).
$$
where in the calculation we assume without loss of generality (a) $\zeta_k \le k R_x R_y$ and (b) $k \ge d$.
\end{proof}

\begin{corollary}\label{cor:pigeon-hole}
  Suppose $\{f^k\}_{k\in[K]}$ and $\{g^k\}_{k\in[K]}$ are two sets of functions from $\cX$ and $\cY$ to $\R^d$ respectively,  which  satisfy 
    \begin{align*}
    \begin{cases}
    \sum_{t=1}^{k-1} \int_\cX\int_\cY   \left| \left\langle f^k(x) , g^t(y) \right\rangle \right|  dxdy \le   \zeta_k  	\\
    \int_\cX \| f^k(x)\|_\infty dx \le R_x \\
     \int_\cY \| g^k(y)\|_1 dy \le R_y \\
     \end{cases}
    \quad  \mbox{ for all } k\in[K].
    \end{align*}
    Then we have 
    \begin{equation*}
    \sum_{t=1}^{k-1} \int_\cX \int_\cY   \left| \left\langle f^t(x) , g^t(y) \right\rangle \right|  dxdy    =\cO \left( d\log^2(k)\left(R_x R_y  + \max_{t\le k} \zeta_t\right)\right)
    \quad  \mbox{ for all } k\in[K].
    \end{equation*}
\end{corollary}
\begin{proof}[Proof of Corollary \ref{cor:pigeon-hole} ]
To simplify notations, denote $I_f :=\int_\cX \| f(z)\|_2 dz$ and $I_g :=\int_\cY \| g(z)\|_2 dz$.
Given $f$ and $g$ that satisfy the norm preconditions in Corollary \ref{cor:pigeon-hole}, we can construct two probability measures $p_f$ and $q_g$ over $\R^d$ so that for any measurable set $V\in\R^d$
\begin{equation*}
\begin{aligned}
 \begin{cases}
 p_f(V) := \int_{\cX} \frac{\|f(x)\|_\infty}{I_f} \mathbf{1} \left( \frac{I_f \times f(x) }{\|f(x)\|_\infty}  \in V\right) dx ,\\
 q_g(V) := \int_{\cY} \frac{\|g(y)\|_1}{I_g} \mathbf{1} \left( \frac{I_g \times g(x) }{\|g(y)\|_1}  \in V\right)
 dy.
 \end{cases}
\end{aligned}
\end{equation*}
Moreover, by direct calculation, one can verify
\begin{equation*}
\begin{aligned}
 \begin{cases}
\int_\cX \| f(x)\|_\infty dx = \E_{v\sim p_f}[\|v\|_\infty ] ,\\
\int_\cY \| g(y)\|_1 dx = \E_{w\sim q_g}[\|w\|_1 ] ,\\
\int_\cX\int_\cY   \left| \left\langle f(x) , g(y) \right\rangle \right|  dxdy = \E_{v\sim p_f,~w\sim q_g}[|v\trans w|].
 \end{cases}
\end{aligned}
\end{equation*}
As a result, Corollary \ref{cor:pigeon-hole} follows immediately from Lemma  \ref{prop:pigeon-hole}.
\end{proof}

\subsection{$\ell_1$-norm projection lemma}

\begin{lemma}[$\ell_1$-norm projection lemma]\label{lem:barycentric}
For each $h\in[H]$, there exists $\X_h\in\R^{|\cQ_h|\times r_{h}}$ such that 
\begin{equation*}
    \begin{cases}
     \X_h \X_h^\dagger\bpsi(\tau_{h})   = \bpsi(\tau_{h}) \qquad \forall \tau_{h}\\
     \| \X_h^\dagger\bpsi(\tau_{h}) \|_1 \le r_{h} \|\bpsi(\tau_{h})\|_1\qquad \forall \tau_{h}\\
     \|\X_h\|_1 \le 1.
    \end{cases}
\end{equation*}
\end{lemma}
\begin{proof}[Proof of Lemma \ref{lem:barycentric}]
To begin with, note that $\dim(\text{span}\{\bpsi(\tau_h)/\|\bpsi(\tau_h)\|_1:~\tau_h\in(\fO\times\fA)^h\})=r_h$. Therefore, by the definition of Barycentric spanner, there exist  $\{\x_1,\ldots,\x_{r_h}\}\subseteq \{\bpsi(\tau_h)/\|\bpsi(\tau_h)\|_1:~\tau_h\in(\fO\times\fA)^h\}$ so that for any $\tau_h$: (a) $\bpsi(\tau_h) = \X_h \X_h^\dagger \bpsi(\tau_h)$, (b) $\|\X_h^\dagger \bpsi(\tau_h)/\|\bpsi(\tau_h)\|_1\|_\infty \le 1$, and (c)$\|\X_h\|_1\le 1$, where $\X_h:=[\x_1,\ldots,\x_{r_h}]\in\R^{|\cQ_h|\times r_h}$.
\end{proof}

\subsection{$\ell_1$-norm matrix inverse}
\label{app:l1-inverse}

\newcommand{\act}{{\rm{ACT}}}
\newcommand{\sign}{{\rm{sign}}}
\newcommand{\supp}{{\rm{support}}}
\newcommand{\trace}{ {\rm{trace}} }

\begin{lemma}\label{lem:inverse-O}
Given $\O\in\R^{O \times S}$ such that $\min_{\z:~\|\z\|_1 =1}\|\O \z\|_1 = \alpha$, we have 
$$
\min_{\Y\O = 0} \|\O^\dagger + \Y\|_1 \le \frac{S}{\alpha}.
$$
\end{lemma}
\begin{proof}[Proof of Lemma \ref{lem:inverse-O}]
 To simplify notations, we  define $\act(\Y) := \arg\max_{i\in[O]} \| (\O^\dagger + \Y) \e_i\|_1$. 
 Note that $\act(\Y)$ is a subset of $[O]$ containing the indices where the maximum is achieved. 
 Given a scalar $z\in\R$, we define 
 \begin{equation*}
     \sign(z) := \begin{cases}
     1, & z>0\\
     -1, & z<0 \\
     [-1,1], & \text{otherwise}.
     \end{cases}
 \end{equation*}
 By generalizing the definition from scalar  to matrix, 
 we define 
 $$
\forall \Z\in\R^{O\times S},\quad \sign(\Z) := \{ \X\in\R^{O\times S}:~ \X_{ij} \in \sign(\Z_{ij})\}.  
 $$

Let  $\Y^\star$ be a  minimizer of $\min_{Y\O = 0} \|\O^\dagger + \Y\|_1$. By KKT condition, we have
$$
0 \in \left\{  \Z \diag(\w) + \W\O\trans:~ \Z\in\sign(\O^\dagger +\Y^\star),~\w\in\Delta_O,~\supp(\w)\subseteq\act(\Y^\star),~\W\in\R^{S\times S} \right\}.
$$
As a result, there exist $\Z^\star\in\sign(\O^\dagger +\Y^\star),~\w^\star\in\Delta_O$ with $\supp(\w^\star)\subseteq\act(\Y^\star)$ and $\W^\star\in\R^{S\times S}$ such that 
$ \Z^\star\diag(\w^\star) = \W^\star\O\trans$. Therefore, we have 
\begin{equation*}
\begin{aligned}
    \min_{\Y\O = 0} \|\O^\dagger + \Y\|_1  &= 
    \|  \O^\dagger + \Y^\star \|_1 \\
    & = \sum_{i\in \act(\Y^\star)}
    \| (\O^\dagger + \Y^\star )\e_i\|_1 \w^\star_i  \\
    &=
    \langle \O^\dagger + \Y^\star, \Z^\star\diag(\w^\star) \rangle
    \\
    &=
    \langle \O^\dagger + \Y^\star, \W^\star\O\trans \rangle = 
    \trace\left((\O^\dagger + \Y^\star)\O {\W^\star}\trans \right) = \trace\left( {\W^\star}\trans \right).
\end{aligned}
\end{equation*}
Furthermore, notice that 
$\|\O {\W^\star}\trans\|_1 = \|\diag(\w^\star) {\Z^\star}\trans\|_1 \le 1 $ since  $\Z^\star\in\sign(\O^\dagger +\Y^\star)$ and $\w^\star\in\Delta_O$. So we have
\begin{equation*}
    \begin{aligned}
       \trace\left( {\W^\star}\trans \right) 
        \le \max_{\G= \O \X  ,~\|\G\|_1 \le 1 }  \trace\left(\X \right) 
        \le  S \times \max_{\g = \O \x  ,~\|\g\|_1 \le 1 }  \|\x \|_1,
    \end{aligned}
\end{equation*}
where the second inequality follows from that  the trace of a matrix is upper bounded by its entry-wise $\ell_1$-norm. 
Finally, We conclude the proof by noticing that 
$$
\max_{\g = \O \x  ,~\|\g\|_1 \le 1 }  \|\x \|_1
= \max_{\g = \O \x  ,~\|\g\|_1 =  1 }  \|\x \|_1
= \max_{\|\O \x\|_1 =  1 }  \frac{\|\x \|_1}{\|\O \x \|_1}=\max_{\|\O \x\|_1 \neq 0 }  \frac{\|\x \|_1}{\|\O \x \|_1} = \frac{1}{\alpha}.
$$
\end{proof}

\section{Auxiliary lemmas}

\begin{lemma}\label{lem:Sep26}
Suppose $\|\M_h(\bnu_1-\bnu_2)\|_1 \ge \alpha \|\alpha_1-\alpha_2\|_1$ for any $\bnu_1,\bnu_2\in\Delta_S$, then we have 
$\|\M_h\x\|_1 \ge \frac{\alpha}{4} \|\x\|_1$ for any $\x\in\R^S$.
\end{lemma}
\begin{proof}[Proof of Lemma \ref{lem:Sep26}]
Let $\bar\x= (\mathbf{1}\trans\x /S) \times \mathbf{1}\in\R^S$,  $\x^+=\max\{\x-\bar\x,0\}$ and $\x^-=\max\{\bar\x-\x,0\}$. We have 
\begin{align*}
\|\M_h \x\|_1 &= 
\|\M_h(\x^+ - \x^-) + \M_h \bar\x \|_1\\
&\ge \frac{1}{4}\left(\|\M_h(\x^+ - \x^-)\|_1 + \|\M_h \bar\x \|_1\right)\\
& = \frac{1}{4}\left(\|\M_h(\x^+ - \x^-)\|_1 + \|\bar\x \|_1\right)\\
& \ge \frac{1}{4}\left(\alpha \|\x^+ - \x^-\|_1 + \|\bar\x \|_1\right)\\
& \ge \frac{1}{4}\left(\alpha \|\x^+ - \x^-\|_1 + \alpha \|\bar\x \|_1\right) \\
& \ge \frac{\alpha}{4} \|\x^+ - \x^- + \bar\x \|_1  = \frac{\alpha}{4}\|\x\|_1,
\end{align*}
where the first inequality uses Lemma 34 in \cite{liu2022partially}, the first equality uses that all entries of $\bar\x$ have the same sign, and the final equality uses $\x=\x^+ - \x^- + \bar\x$.
\end{proof}

\begin{lemma}\label{lem:observable-disjoint}
Suppose for any disjoint  $\bnu_1,\bnu_2\in\Delta_S$, $\|\O_h(\bnu_1-\bnu_2)\|_1\ge \alpha$. Then we have for arbitrary  $\bnu_1,\bnu_2\in\Delta_S$ not necessarily disjoint, 
$\|\O_h(\bnu_1-\bnu_2)\|_1\ge \alpha\|\bnu_1-\bnu_2\|_1$.
\end{lemma}
\begin{proof}[Proof of Lemma \ref{lem:observable-disjoint}]
Let $\bnu^+ = \max\{\bnu_1-\bnu_2,0\}$ and $\bnu^- = \max\{\bnu_2-\bnu_1,0\}$. We have $\bnu_2-\bnu_1=\bnu^+-\bnu^-$ and $\|\bnu^+\|_1=\|\bnu^-\|_1$. Therefore,  for arbitrary  $\bnu_1,\bnu_2\in\Delta_S$ not necessarily disjoint, 
\begin{align*}
    \|\O_h(\bnu_1-\bnu_2)\|_1 
    = &  \|\O_h(\bnu^+-\bnu^-)\|_1\\
    \ge  & \alpha \|\bnu^+-\bnu^-\|_1
    = \alpha \| \bnu_1-\bnu_2\|_1.
\end{align*}
\end{proof}

\begin{lemma}\label{prop:l2-eluder-linear}
Let $\cX,\Theta$ be two bounded subsets of $\R^N$ such that $\max_{x\in\cX}\max_{\theta\in\Theta} |\langle x,\theta \rangle| \le C$ and $\dim(\mathrm{span}(\cX))=d$. Then 
    	the $\ell_2$-norm $\epsilon$-eluder dimension of 
	\[
	\Fcal=\{f_\theta :\ 
	f_\theta(x)=\langle x,\theta \rangle, x\in \cX\,,
	\theta\in \Theta\}
	\]
	is at most ${\mathcal{O}}(d\log(dC/\epsilon))$.
\end{lemma}
\begin{proof}[Proof of Lemma \ref{prop:l2-eluder-linear}]
 Let $\{b_1,\ldots,b_d\}\subset\cX$ be a Barycentric spanner of $\cX$. 
 Denote $B:=[b_1,\ldots,b_d]$. 
 By definition, this implies that for any $x\in\cX$, $x = BB^\dagger x$ and $\|B^\dagger x\|_\infty \le 1$.
 
 Suppose $x^1,\ldots,x^K$ is an $\epsilon$-independent sequence. Then there exist $\theta^1,\ldots,\theta^K$ such that  for any $k\in[K]$, 
 $$
{\theta^k}\trans\left( \sum_{t=1}^{k-1}x^t {x^t}\trans\right)\theta^k \le \epsilon^2 \quad \mbox{ but } \quad  |\langle \theta^k, x^k \rangle | \ge \epsilon.
 $$
Define $V_k :=\sum_{t=1}^{k-1}x^t {x^t}\trans + \lambda \sum_{i=1}^d b_i b_i\trans$ where $\lambda= \epsilon^2/ (C^2 d) $. We have 
$$
{\theta^k}\trans V_k {\theta^k} \le \epsilon^2 + \lambda d R_x R_\theta \le 2\epsilon^2,
$$
which together with $|\langle \theta^k, x^k \rangle | \ge \epsilon$ implies ${x^k}\trans V_k^\dagger x^k\ge 1/2$. 
Let $P$ be a column-orthonormal matrix  with  column space equal to $\text{span}(\cX)$. Then, we have 
\begin{equation}\label{eq:Aug14-1}
    \begin{aligned}
         \left(\frac{3}{2}\right)^K  \le &  \prod_{k=1}^K \left(1+  {x^k}\trans V_k^\dagger x^k\right) \\
       = & \frac{\det(P\trans V_{K+1} P )}{\det(P V_1 P)} \\
      =  &\det\left(P\trans   \left(V_1^{\dagger}\right)^{1/2} \left(\sum_{t=1}^{K}x^t {x^t}\trans + \lambda \sum_{i=1}^d b_i b_i\trans \right)\left(V_1^{\dagger}\right)^{1/2} P  \right) \\
      =  & \det\left(    \left(\sum_{t=1}^{K}P\trans\left( V_1^{\dagger}\right)^{1/2}x^t {x^t}\trans\left(V_1^{\dagger}\right)^{1/2}   P \right)+ I_{d\times d}\right). 
    \end{aligned}
\end{equation}
Define $z_t := P\trans\left( V_1^{\dagger}\right)^{1/2}x^t$. We have 
$$
\| P\trans\left( V_1^{\dagger}\right)^{1/2}x^t\|_2^2 =
{x^t}\trans  V_1^{\dagger} x^t \le \frac{1}{\lambda} {x^t}\trans (BB\trans)^\dagger x^t = \frac{1}{\lambda} \| B^\dagger x^t \|_2^2 \le \frac{1}{\lambda} \| B^\dagger x^t \|_1^2 \le \frac{d^3 C^2}{\epsilon^2},
$$
where the last inequality uses the fact that $B$ is a Barycentric spanner. 
Combining the norm upper bound with  Equation \eqref{eq:Aug14-1}, we obtain 
$$
 \left(\frac{3}{2}\right)^K \le  \left( \frac{Kd^2 C^2}{\epsilon^2} +1\right)^d.
$$
As a result, by simple algebra,  we conclude 
$$
K \le \cO(d \log(dC/\epsilon)).
$$
\end{proof}

\end{document}